%% file: main.tex
\documentclass{article} 
\usepackage{algorithm}
\usepackage{algorithmicx}
\usepackage[noend]{algcompatible}
\usepackage[compatible,noend]{algpseudocode}
\usepackage{amsmath}
\usepackage{amssymb}
\usepackage{amsthm}
\usepackage{booktabs}
\usepackage{braket}
\usepackage{diagbox}
\usepackage{environ}
\usepackage{fullpage}
\usepackage{listings}
\usepackage{multicol}
\usepackage{multirow}
\usepackage{orcidlink}
\usepackage{subcaption}
\usepackage{svg}
\usepackage{tabularx}
\usepackage{tikz}
\usepackage{xcolor}

\begin{document}
\include{macros}
\newcommand{\changes}[1]{{\color{blue}#1}}
\renewcommand{\changes}[1]{#1}

\title{Homomorphisms and Embeddings of STRIPS Planning Models}

\author
{Arnaud Lequen\,\orcidlink{0000-0003-0339-0967} \and Martin\,C. Cooper \orcidlink{0000-0003-4853-053X} \and Frédéric Maris\,\orcidlink{0000-0002-1084-1669}}

\date{IRIT, University of Toulouse, France \\ \vspace{1ex} \textit{Arnaud.Lequen@irit.fr} \\ \vspace{0.5ex} \textit{Martin.Cooper@irit.fr} \\ \vspace{0.5ex} \textit{Frederic.Maris@irit.fr}}

\maketitle
%

\begin{center}
    \textbf{Abstract}%
\end{center}%
\paragraph{}
Determining whether two STRIPS planning instances are isomorphic is the simplest form of comparison between planning instances.
It is also a particular case of the problem concerned with finding an isomorphism between a planning instance $P$ and a sub-instance of another instance $P'$.
One application of such a mapping is to efficiently produce a compiled form containing all solutions to $P$ from a compiled form containing all solutions to $P'$.
We also introduce the notion of \emph{embedding} from an instance $P$ to another instance $P'$, which allows us to deduce that $P'$ has no solution-plan if $P$ is unsolvable.
In this paper, we study the complexity of these problems. We show that the first is GI-complete, and can thus be solved, in theory, in quasi-polynomial time. While we prove the remaining problems to be NP-complete, we propose an algorithm to build an isomorphism, when possible.
We report extensive experimental trials on benchmark problems
which demonstrate conclusively that applying constraint 
propagation in preprocessing can greatly improve the efficiency
of a SAT solver.

\section{Introduction}



Automated planning is concerned with finding a course of action in order to achieve a goal.
In this paper, we are concerned with deterministic, fully-observable planning tasks.
A case in point of such a task is the 3x3x3 Rubik's cube, presented in Figure~\ref{fig:cube_intro:3x3}~\footnote{All images of Rubik's cube in this paper have been generated using \url{https://github.com/Cride5/visualcube}}: even though the puzzle can be described in a few words, there are more than $4.3 \cdot 10^{19}$ different configurations that can be reached through a sequence of legal moves~\cite{rokicki2014}.
As a consequence, without further knowledge, finding a solution for this puzzle is far from trivial, as some configurations might require up to 26 moves to be solved~\cite{rokicki2014}.
However, nowadays, efficient methods for solving a 3x3x3 Rubik's cube are known. But many variants of the puzzle have emerged: consider for instance the 2x2x2 cube, as depicted in Figure~\ref{fig:cube_intro:2x2}. Despite its smaller size, there are approximately $3 \cdot 10^6$ 
different configurations. To tackle this puzzle, rather than trying to find a solution \emph{ex nihilo}, 
one could make the simple observation that a 2x2x2 cube can be seen as a 3x3x3 cube on which edges are ignored. This allows us to transfer the knowledge that we have on the bigger cube to the smaller. In particular, one can readily transform a solution for the 3x3x3 cube into a solution for the 2x2x2 cube, given that the corner pieces are in a similar configuration.


\begin{figure}
    \centering
    \begin{subfigure}[b]{0.25\textwidth}
         \centering
         \includegraphics[width=\textwidth]{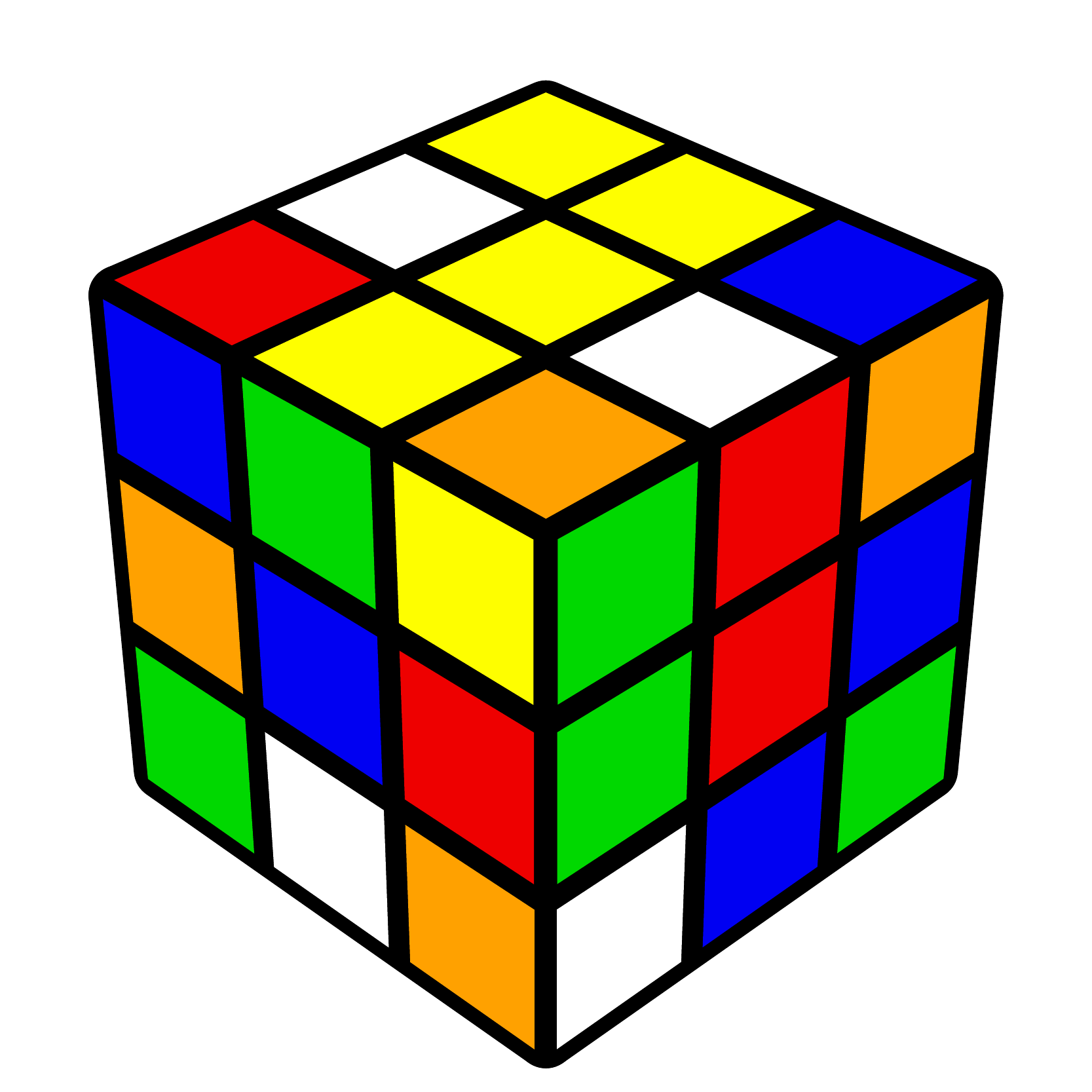}
         \caption{A 3x3x3 Rubik's cube}
         \label{fig:cube_intro:3x3}
     \end{subfigure}
     \begin{subfigure}[b]{0.25\textwidth}
         \centering
         \includegraphics[width=\textwidth]{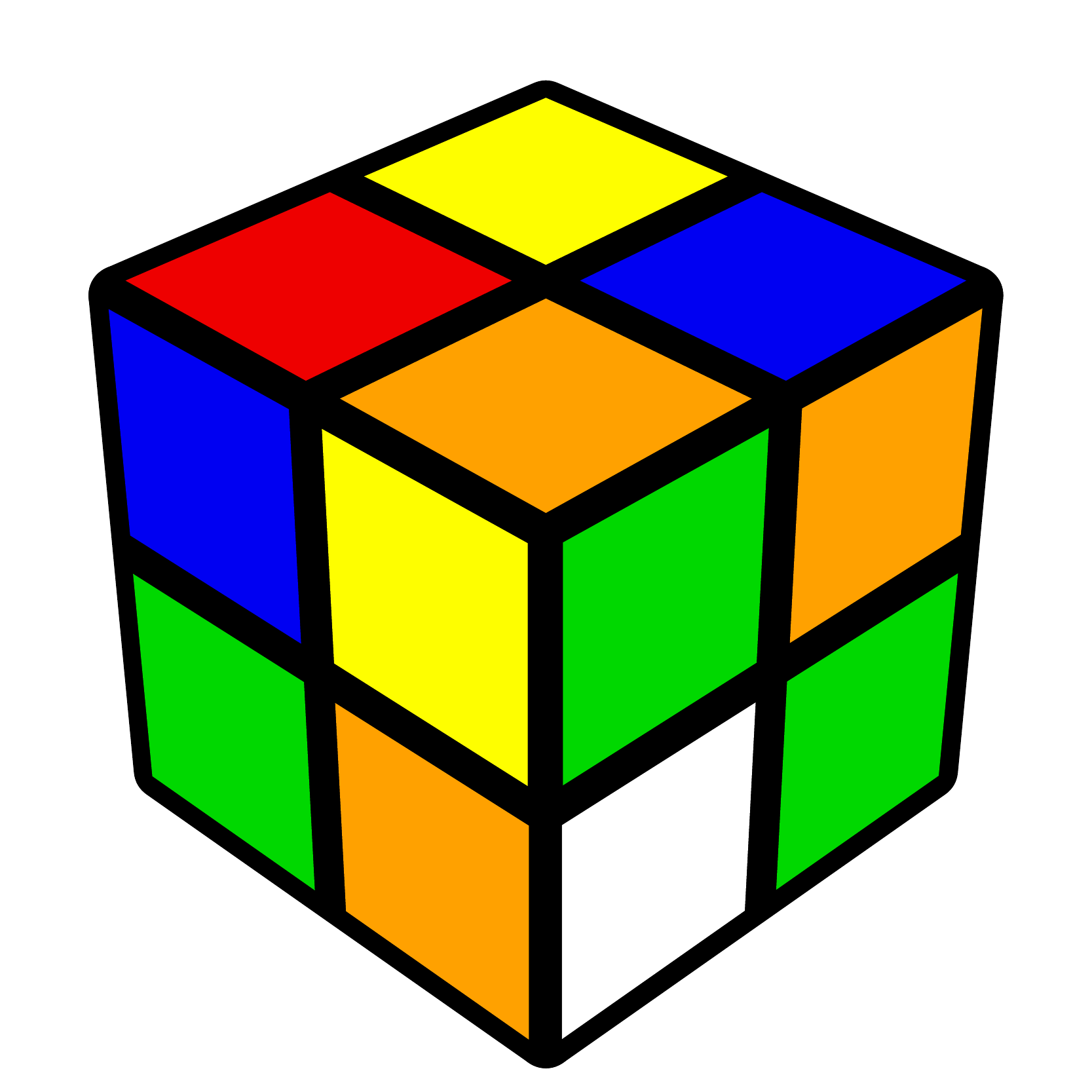}
         \caption{A 2x2x2 Rubik's cube}
         \label{fig:cube_intro:2x2}
     \end{subfigure}
    \caption{A pair of Rubik's cubes. For both of these puzzles, each face can be rotated 90 degrees, clockwise or counter-clockwise, any number of times. The goal is to reach a configuration where for all six colors, all facets of that color are on the same face, so that faces have a homogeneous color. Note that the respective configurations of the corner pieces of these cubes are similar: any sequence of moves that solves the 3x3x3 cube also solves the 2x2x2 cube.}
    \label{fig:cube_intro}
\end{figure}

In the more general case of automated planning, planning tasks expressed in STRIPS~\cite{FikesN71} encode sizeable state-spaces that can rarely be represented explicitly, but that have a clear and somewhat regular structure. Parts of this structure can be, however, common to multiple planning instances, although this similarity is often far from immediate to identify by looking at the STRIPS representation.
Indeed, finding whether or not an instance $P$ is equivalent to a sub-instance of another instance $P'$ requires 
finding a mapping between fluents and operators of the two instances,
while respecting a morphism property. This requires the exploration of the exponential search space of mappings from $P$ to $P'$. 
Finding such a mapping, however, allows us to carry over significant pieces of information (such as solvability) from one problem to the other. 

A classical technique in constraint programming is to store all
solutions to a CSP or SAT instance in a compact compiled form~\cite{DBLP:journals/ai/AmilhastreFM02}.
This is performed off-line. A compilation map indicates which
operations and transformations can be performed in polynomial
time during the on-line stage~\cite{DBLP:journals/jair/DarwicheM02}. STRIPS fixed-horizon planning
can be coded as a SAT instance using the classical SATPLAN
encoding~\cite{DBLP:conf/ecai/KautzS92}. So, for a given instance, all plans \changes{up to a given length} can be stored in
a compiled form, at least in theory. In practice, the compiled
form will often be too large to be stored. Types of planning
problems which are nevertheless amenable to compilation are
those where the number of plans is small or, at the other
extreme, there are few constraints on the order of operators.
If we have a compiled form $C'$ representing all \changes{shortest} solution-plans
to an instance $P'$ and we encounter a similar problem $P$,
it is natural to ask whether we can synthesize a plan for $P$
from $C'$. If $P$ is isomorphic to a subinstance of $P'$, then
it suffices to apply a sequence of conditioning operations to
$C'$ to obtain a compiled form $C$ representing all solutions
to $P$. This was our initial motivation for studying isomorphisms
between subproblems. 
A trivial but important special case occurs
when $C'$ is empty, i.e. $P'$ has no solution. In this case,
an isomorphism from $P$ to a subproblem of $P'$ is a proof that
$P$ has no solution.



In this paper, we will introduce two notions of homomorphisms between planning models. In general, any mapping from a model (a planning instance in our case) to another model that carries over (part of) the structure of the original model is called a homomorphism. If, in addition, the mapping is bijective, then it can be referred to as an isomorphism. We
first focus on problem \stripsiso{}, which is concerned with finding an isomorphism between two STRIPS instances of identical size. As we show that the problem is \GI{}-complete, we prove that a quasi-polynomial time algorithm theoretically exists~\cite{babai2018group}.

We then consider problem \stripssubiso{}, which is concerned with finding an isomorphism between a STRIPS instance \changes{$P$} and a subinstance of another STRIPS instance \changes{$P'$}.
\changes{Given a STRIPS instance $P'$, we call subinstance of $P'$ any STRIPS instance $P'_s$ which shares the same initial state and goal, and whose set of fluents $F'_s$ (resp. operators $O'_s$) forms a subset of the fluents $F'$ (resp. operators $O'$) of $P'$, so that the operators of $O'_s$ do not contain fluents of $F' \setminus F'_s$.}
\changes{An isomorphism between $P$ and some subinstance $P'_s$ of $P'$ is called} a \emph{subinstance isomorphism}. After showing that 
\changes{the problem of finding such a mapping} is \NP{}-complete, we propose an algorithm that finds a subinstance isomorphism if one exists, or that detects that none exists. This algorithm is based on constraint propagation techniques, that allow us to prune impossible associations between elements of $P$ and $P'$, as well as on a reduction to SAT.

So far we have assumed that the two planning instances $P$ and $P'$ 
have the same initial states and goals (modulo the isomorphism).
Even when this is not the case, an isomorphism from $P$ to a subinstance
of $P'$ can still be of use. For example, if $\pi$ is a solution-plan 
for $P$, then its image in $P'$ can be converted to a single new operator \changes{(commonly called a \emph{macro-operator})} which could be
added to $P'$ to facilitate its resolution. We therefore
also consider this weaker notion of subinstance isomorphism, that we call \emph{homogeneous subinstance isomorphism}, and the corresponding 
computational problem \stripssubisogeneral{}. 

In addition, we also introduce another form of comparison between two STRIPS planning instances, that we call \emph{embedding}. An embedding of $P'$ into a bigger instance $P$ is a form of homomorphism that only takes into account operators of $P$ that are not trivialized by the induced transformation. It is useful in proving the unsolvability of an instance, as an embedding of an unsolvable instance $P'$ into $P$ is a proof that $P$ is also unsolvable. To quickly see an essential difference
between the notions of subinstance-isomorphism and embedding, consider an
instance $P$, an instance $Q$ obtained by deleting some operators from $P$
and another instance $R$ obtained by deleting some goals from $P$: $Q$ is
isomorphic to a subinstance of $P$ whereas $R$ embeds in $P$. This simple
example illustrates the important difference that an instance which is
isomorphic to a subinstance (such as $Q$) is harder to solve than the original
instance $P$, whereas embedded instances (such as $R$) are easier.

Previous work investigated the complexity of various problems related to finding solution-plans for STRIPS planning instances~\cite{bylander1994computational}, or focused on the complexity of solving instances from specific domains~\cite{helmert2003complexity}. 
More scarcely, problems focused on altering planning models have been studied from a complexity theory point of view, such as the problem concerned with adapting a planning model so that some user-specified plans become feasible~\cite{lin2021change}.
The present paper follows an orthogonal track in that we detect relationships
between planning instances rather than directly solving them. Independently of
any computational problem, the notions of subinstance isomorphism and embedding
may be of interest in explainable AI: a minimal solvable isomorphic subinstance
of a solvable instance can be viewed as an explanation of solvability,
whereas a minimal unsolvable embedded subinstance can be seen as an explanation
of unsolvability.

The paper is organized as follows. In Section~\ref{sec:preliminaries}, we introduce general notation, concepts and constructions that we use throughout this paper. In Section~\ref{sec:strips_iso_complexity} and Section~\ref{sec:strips_subiso_complexity}, we present our complexity results, for \stripsiso{} and \stripssubiso{} respectively. In Section~\ref{sec:ssi_algorithmic}, we present the outline of our algorithm for \stripssubiso{}. We follow the same pattern for \stripsembedding{}, as we present our complexity result for the problem in Section~\ref{sec:strips_embedding_complexity}, and the adapted algorithm in Section~\ref{sec:se_algorithmic}. Section~\ref{sec:experimental_evaluation} is then dedicated to the experimental evaluation and discussion. The present paper is a considerably extended version
of a conference paper which studied the notion of subinstance isomorphism
but not of embedding~\cite{CP2022}.



\section{Preliminaries}
\label{sec:preliminaries}

In this section, we present generalities about automated planning, and we give some background on the complexity class \GI{}. In addition, Section~\ref{sec:graph_constructions} presents some constructions that we require for the rest of this paper.

\subsection{Automated Planning}

In this paper, we encode planning tasks in STRIPS. A STRIPS planning instance is a tuple $P = \langle F, I, O, G \rangle$ such that $F$ is a set of \emph{fluents} 
(propositional variables whose values can change over time), $I$ and $G$ are 
sets of fluents of $F$, called the \emph{initial state} and \emph{goal},
and $O$ is a set of \emph{operators}. Operators are of the form $o = \langle \pre{o}, \eff{o} \rangle$. $\pre{o} \subseteq F$ is the \emph{precondition} of $o$ and $\eff{o}$ is the \emph{effect} of $o$, which is a set of literals of $F$. We will denote $\effp{o} = \{f \in F \mid f \in \eff{o}\}$, and $\effm{o} = \{f \in F \mid \neg f \in \eff{o}\}$.
\changes{For any set $S$ of literals of $F$, we denote $\neg S = \{\neg l \mid l \in S\}$.}
By a slight abuse of notation, we will denote $\presymb: O \rightarrow 2^F$ the function $o \mapsto \presymb(o)$, and use similar notation with $\effpsymb$ and $\effmsymb{}$, as well as with $\effsymb: O \rightarrow 2^F \cup 2^{\neg F}$.  In the rest of this paper, we will note $\mathcal{C} = \{\presymb{}, \effpsymb{}, \effmsymb{}\}$. In addition, we apply functions to sets in the natural way: for a function $u : A \rightarrow B$, if $C \subseteq A$,
then $u(C)$ represents the set $\Set{u(x) | x \in C}$. 

Note that even if the formalism we present does not allow negative literals in the preconditions, we do not lose in generality: any STRIPS instance with negative preconditions can be translated into an equivalent instance in the formalism we work with, in linear time~\cite{geffner2013concise}. 

A state $s$ is an assignment of truth values to all fluents in $F$.
For notational convenience, we associate $s$ with the set of literals
of $F$ which are true in $s$\changes{, so that $2^F$ is the set of all states}. Given an instance $P = \langle F, I, O, G \rangle$, 
\changes{the state that results from the application of $o \in O$ to some state $s \in 2^F$ is the state $\apply{s}{o} = \left( s \setminus \effm{o} \right) \cup \effp{o}$, which we only define when $o$ is \emph{applicable} in $s$, which means that $\pre{o} \subseteq s$. A}
\emph{solution-plan} is a sequence of operators $o_1,\ldots,o_k$
from $O$ such that the   
sequence of states $s_0,\ldots,s_k$ defined by  $s_0 = I$ and 
\changes{$s_i = \apply{s_{i-1}}{o}$} 
(for all $i \in \{1,\ldots,k\}$)
satisfies $\pre{o_i} \subseteq s_{i-1}$ (for all $i \in \{1,\ldots,k\}$) and
$G \subseteq s_k$. A \emph{plan} is defined similarly but without the conditions concerning $I$ and $G$.

For example, in a STRIPS formalisation of a 2x2x2 cube, fluents would be used to represent the different states of the cube, while the operators would model the action of rotating a face. A possible way to model the puzzle is sketched in Figure~\ref{fig:cube_strips}.
Coding a Rubik’s cube in STRIPS would be unwieldy, due to the large number of actions required, but it will allow us to easily illustrate certain notions since it is a well-known puzzle that requires little explanation. \changes{Various encodings of the problem into a planning language have been proposed, and both the 2x2x2 and the 3x3x3 Rubik's cube have been modeled in SAS$^+$~\cite{buchner2022comparison} (which is a more succinct planning language where variables are multivalued, whereas in STRIPS they are binary since a fluent is either true or false) or in PDDL with conditional effects~\cite{muppasani2023solving}(which is also significantly richer than STRIPS).}

Another generic problem that can be compactly coded in STRIPS is the path-finding problem on a graph, described below in Section~\ref{sec:graph_constructions} which could easily be generalized to a multi-agent version. 

\begin{figure}
    \centering
    \begin{center}
        \includegraphics[width=0.2\textwidth]{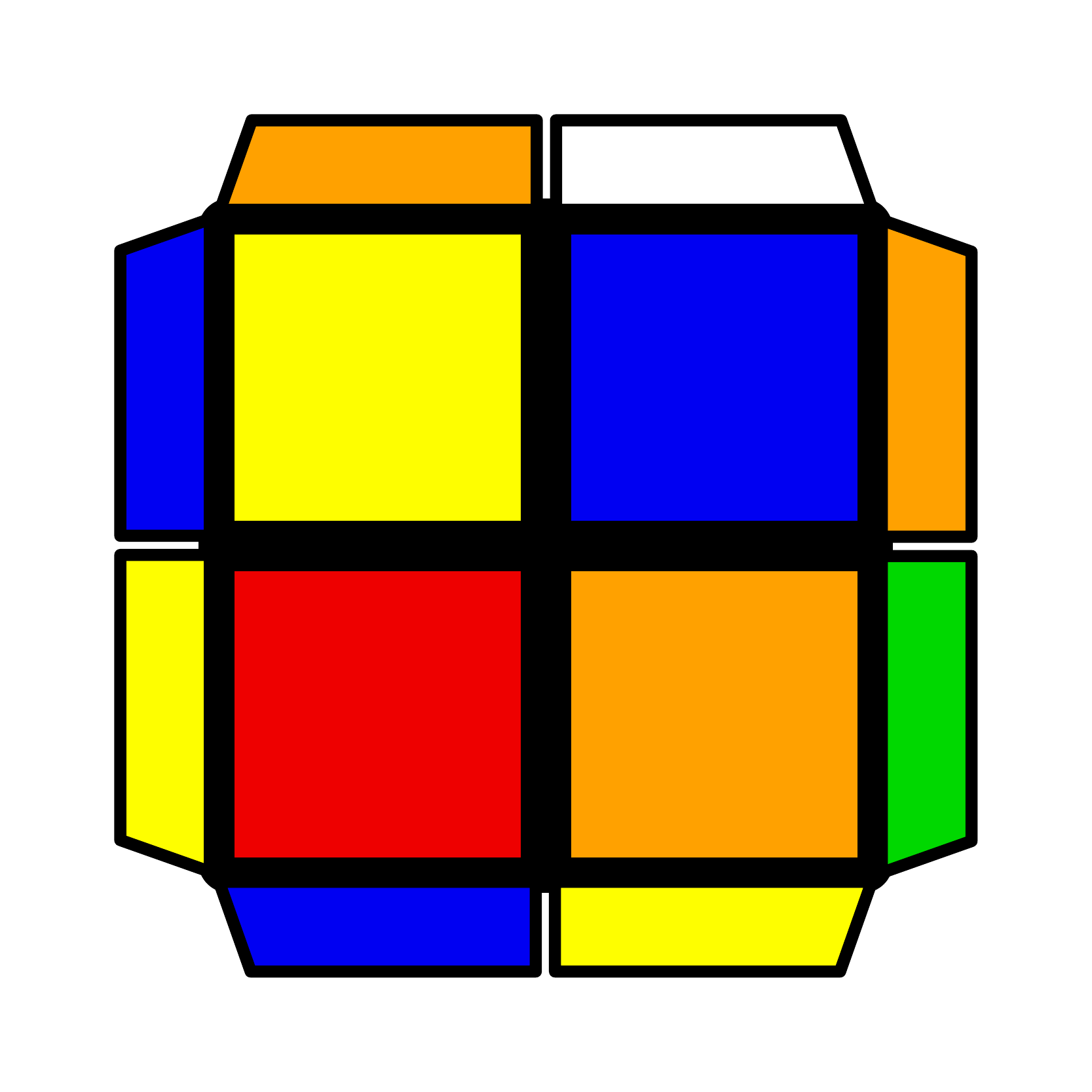}
        \includegraphics[width=0.2\textwidth]{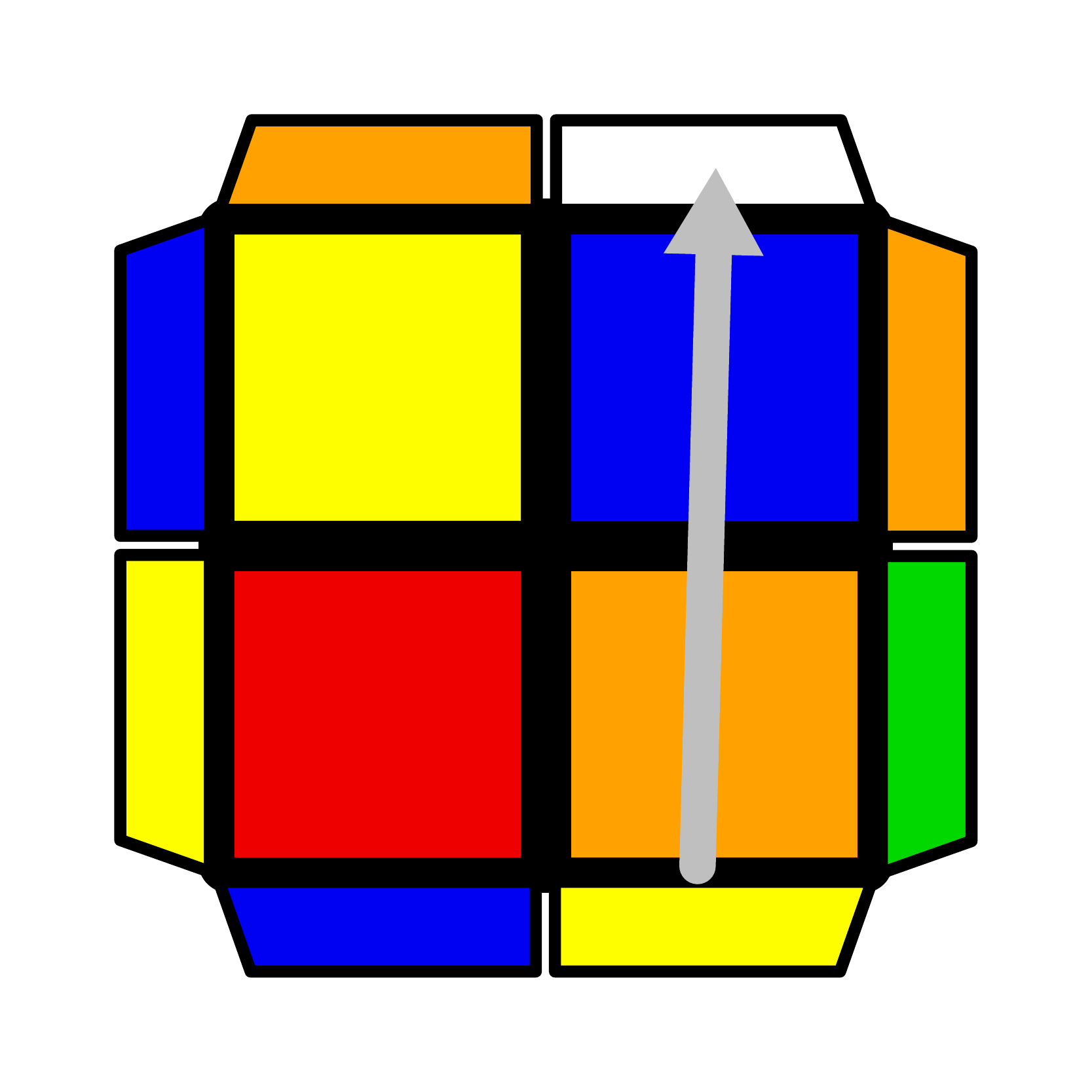}
        \includegraphics[width=0.2\textwidth]{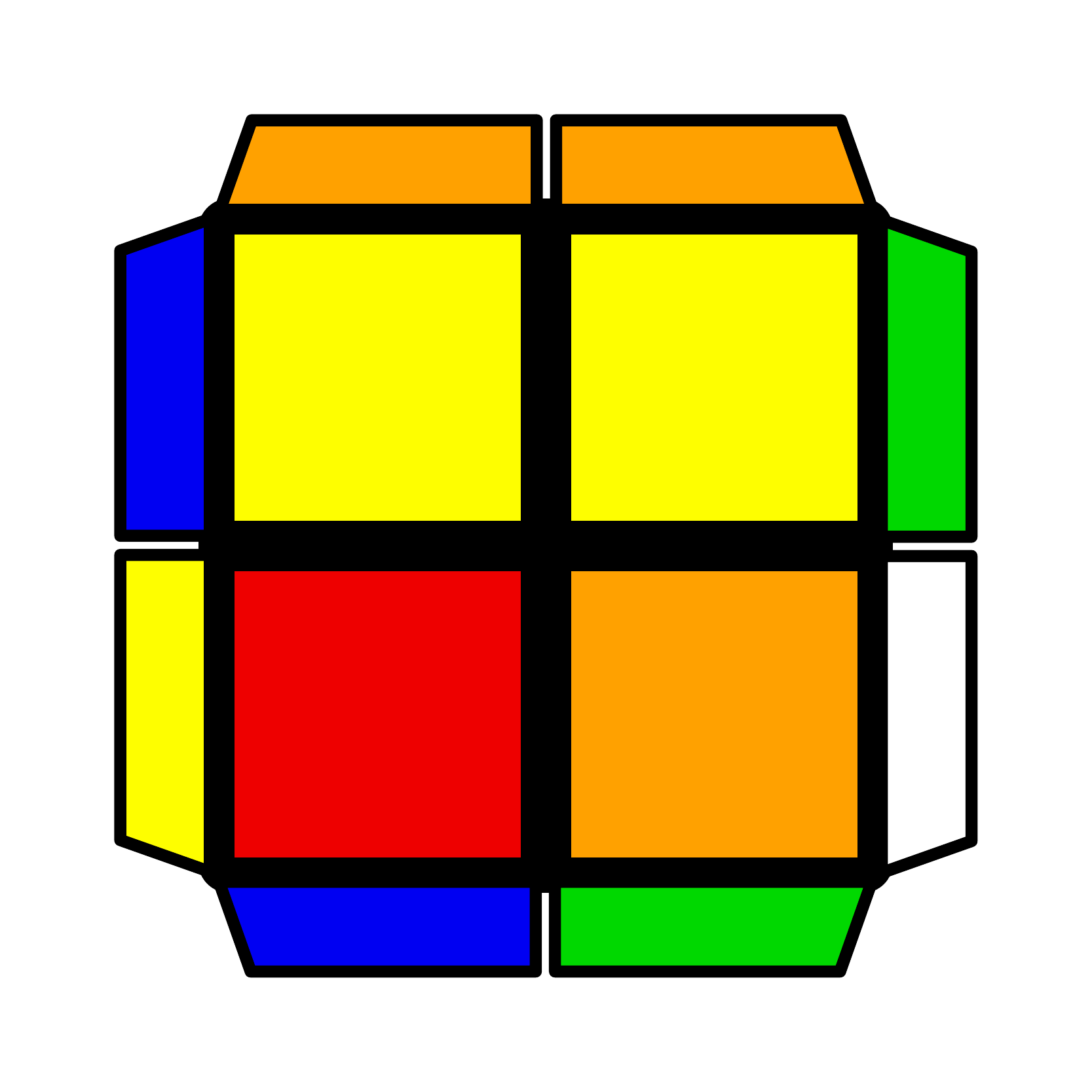}
    \end{center}
    
    $\Set{\begin{array}{c}\text{col}(\text{up-top-right, \textcolor{blue}{blue}}),\\
    \text{col}(\text{up-low-left, \textcolor{red}{red}}), \\
    \ldots
    \end{array}}
    \hspace{7ex}\xrightarrow{\hspace{1ex}\changes{\lstinline{Operator o}}\hspace{1ex}}\hspace{6ex}
    \Set{\begin{array}{c}\text{col}(\text{up-top-right, \textcolor{orange}{yellow}}),\\
    \text{col}(\text{up-low-left, \textcolor{red}{red}}), \\
    \ldots
    \end{array}}
    $
    \caption{Top view of the 2x2x2 cube presented earlier, as well as the resulting configuration after the right face is rotated of a quarter of turn clockwise. Below both configurations, we partially represent the associated state resulting from a formalisation into STRIPS of the puzzle. For instance, the fluent \lstinline|col(up-top-right, blue)|, which is true in the first state, represents the fact that, on the upper face of the cube, the top right facet is blue. After \changes{some operator \lstinline{o} (which, in this case, rotates the right face clockwise)} is applied to the state, the facet is no longer blue, but has been replaced by a yellow one, as reflected in the resulting state.}
    \label{fig:cube_strips}
\end{figure}

\changes{A STRIPS planning instance represents succinctly a more expansive state-space, which we formalise as a \emph{labeled transition system} (LTS). A labeled transition system is a tuple $\lts{} = \langle S, L, T, s_I, S^G\rangle$, where $S$ is a finite set of \emph{states}, $L$ is a finite set of \emph{labels}, $T \subseteq S \times L \times S$ is the set of \emph{transitions}, $s_I \in S$ is the \emph{initial state}, and $S^G \subseteq S$ is the set of \emph{goal states}. When given a STRIPS planning instance $P = \langle F, I, O, G \rangle$, the underlying state-space is naturally represented by the LTS $\lts^P = \langle 2^F, O, T^P, I, S^G \rangle$, where the states of $\lts^P$ are exactly the states of $P$, the initial state of $\lts^P$ is the same as for $P$, the labels are the operators, and the transitions $T^P$ of $\lts^{P}$ between two states $s_1, s_2 \in 2^F$ are the ones that are possible by the application of an operator of $P$. More formally, we have $T^S = \left\{\langle s_1, o, s_2\rangle \in 2^F \times O \times 2^F \mid \pre{o} \subseteq s_1 \text{ and } s_2 = \apply{s_1}{o} \right\}$. The goal states of $\lts^P$ are the ones that satisfy the goal condition of $P$, so that $S^G = \{ s \in 2^F \mid G \subseteq s\}$.}

\subsection{\changes{The} Complexity Class \GI{}}

This section introduces the complexity class \GI{}, for which \stripsiso{} is later shown to be complete. \GI{} is built around the Graph Isomorphism problem,
\changes{which consists in determining whether two graphs have the same structure.}
\changes{Formally, a graph $\graph$ is a pair $(V, E)$, where $V$ is a set of elements called \emph{vertices} and $E$ is a set of (unordered) pairs called \emph{edges}. $\graph$ is said to be \emph{undirected} when edges are unordered pairs (denoted $\{v_1, v_2\}$ with $v_1, v_2 \in V$), and \emph{directed} when edges are ordered pairs (denoted $(v_1, v_2)$).} 

\changes{The Graph Isomorphism problem} consists in determining the existence of 
a bijection $u: V \rightarrow V'$ between the vertices of two \changes{undirected} graphs $\graph{}(V, E)$ and $\graph{}'(V', E')$, such that the images of vertices linked by an edge in $\graph{}$ are also linked by an edge in $\graph{}'$ (and vice-versa). Formally, we require that the following condition holds:
\begin{equation}
    \label{eq:gi_morphism}
    \{x, y\} \in E \text{ iff } \{u(x), u(y)\} \in E'
\end{equation}



\begin{definition}
The complexity class \GI{} is the class of problems with a polynomial-time Turing reduction to the Graph Isomorphism problem.
\end{definition}

Complexity class $\GI{}$ contains numerous problems concerned with the existence of an isomorphism between two non-trivial structures encoded explicitly. Such problems are often complete for the class. For instance, the problems of finding an isomorphism between colored graphs, hypergraphs, automata, etc. are $\GI{}$-complete~\cite{zemlyachenko1985graph}.
In particular, we later use the following result:

\begin{proposition}[Zemlyachenko \emph{et al.} 1985\cite{zemlyachenko1985graph}, Ch. 4, Sec. 15]
    The \emph{\changes{Directed} Graph Isomorphism problem} is \GI{}-complete.
\end{proposition}

As with the Graph Isomorphism problem, an isomorphism between \changes{directed} graphs $\graph{}(V, E)$ and $\graph{}'(V', E')$ is a bijection $u: V \rightarrow V'$ such that 
 $(x, y) \in E \text{ iff } (u(x), u(y)) \in E'$.

In this paper, we consider another category of structures, called Finite Model, defined below. Finite models are also such that the related isomorphism existence problem is \GI{}-complete.

\begin{definition}
\label{def:finite_model}
A Finite Model is a tuple $M = \langle V, R_1, \ldots, R_n \rangle$ where $V$ is a finite non-empty set and each $R_i$ is a relation on elements of $V$ with a finite number of arguments.
\end{definition}

Let $M = \langle V, R_1, \ldots, R_n \rangle$ and $M' = \langle V', R'_1, \ldots, R'_n \rangle$ be two finite models. An isomorphism between $M$ and $M'$ is a bijection $u: V \rightarrow V'$ such that, for any $i \in \{1, \ldots, n\}$, for any set of elements $v_1, \ldots, v_m$ with $m$ the arity of $R_i$, $R_i(v_1, \ldots, v_m)$ iff $R'_i\left(u(v_1), \ldots, u(v_m)\right)$.

\begin{proposition}[Zemlyachenko \emph{et al.} 1985\cite{zemlyachenko1985graph}, Ch. 4, Sec. 15]
    The \emph{Finite Model Isomorphism problem} is \GI{}-complete.
\end{proposition}

Class $\GI$ is believed to be an intermediate class between $\Pclass$ and $\NP$: the Graph Isomorphism problem can indeed be solved in quasi-polynomial time~\cite{babai2018group}. Although the problem is thought not to be $\NP$-complete, no polynomial time algorithm is known.

\subsection{Graph encodings into STRIPS}
\label{sec:graph_constructions}

In this section, we present two ways to encode a graph $\graph{} = (V,E)$ into a planning problem $P=\langle F, I, O, G \rangle$. These constructions are needed at various points in the rest of this paper, and only differ in that \changes{the first one concerns directed graphs, while the second one applies to undirected graphs.}

The intuition behind these constructions is that they model an agent that can move on the graph, resting on vertices and moving along edges. An agent being on vertex $v$ would thus be denoted by the state $\{v\}$, where all fluents other than $v$ are false.

In order to make the construction and resulting proofs simpler to read, for any pair $(v_s, v_t) \in F^2$, we will denote $\moveaction{v_s}{v_t}$ the operator that represents a movement from vertex $v_s$ to vertex $v_t$. Keeping in mind that $F=V$, we have, more formally:
\[
    \moveaction{v_s}{v_t} = \action{\{v_s\}}{\{v_t \} \cup \neg(V \setminus \{v_t\})}
\]
In the following construction, the vertices (resp. edges) of $\graph{}$ are in bijection with the fluents (resp. operators) of $P$. In particular, we do not allow multi-edges.


\begin{construction}
\label{cons:oriented_graph}
    Let $\graph{} = (V, E)$ be \changes{a directed} graph. Let us build the planning problem $P_\graph{} = \langle F, I, O, G \rangle$, where:
    \begin{align*}
        F &= V \\
        O &= \left\{ \moveaction{v_s}{v_t} \mid (v_s, v_t) \in E \right\} \\
        G &= I = \varnothing
    \end{align*}
\end{construction}

In the case of \changes{undirected} graphs, the construction is essentially the same, except that moves are possible in both directions.
This gives us the following definition:
\begin{construction}
\label{cons:symmetric_graph}
    Let $\graph{} = (V, E)$ be \changes{an undirected} graph. Let us build the planning problem $P_\graph{} = \langle F, I, O, G \rangle$, where $F$, $I$ and $G$ are defined as in Construction~\ref{cons:oriented_graph}, but where:
    \begin{equation*}
        O = \left\{ \moveaction{v_s}{v_t} \mid \{v_s, v_t\} \in E \right\} 
    \end{equation*}
\end{construction}

Note that in Construction~\ref{cons:symmetric_graph}, we have $\moveaction{v_s}{v_t} \in O$ iff $\moveaction{v_t}{v_s} \in O$.


\section{STRIPS Isomorphism Problem}
\label{sec:strips_iso_complexity}

This section is concerned with the problem of finding an isomorphism between two STRIPS planning problems. After introducing the notion of isomorphism between STRIPS instances that we use throughout this paper, we define problem \stripsiso{} more formally, and settle its complexity.

Intuitively, two STRIPS planning instances are isomorphic when they encode the exact same problem, but do not necessarily agree on the names given to their fluents and operators. For instance, let us consider the pair of 3x3x3 cubes depicted in Figure~\ref{fig:cube_isom}, where the one on the right-hand side is a the exact same cube as the other, but recolored: blue facets have been recolored in light green, yellow facets in purple, etc. The STRIPS instance that models the right cube is the same as the STRIPS instance of the left cube, except that fluents of the form \lstinline|color(X, yellow)| have been replaced by fluents of the form \lstinline|color(X, purple)|, where X designates a facet. An isomorphism between both instances is then a bijective mapping that associates each fluent and operator of a problem to its counterpart in the other.

\begin{figure}
    \centering
    \subfloat{\tikz[remember
    picture]{\node(1ALcubeisom){\includegraphics[width=4cm]{images/3x3.pdf}};}}%
    \hspace*{8ex}%
    \subfloat{\tikz[remember picture]{\node(1ARcubeisom){\includegraphics[width=4cm]{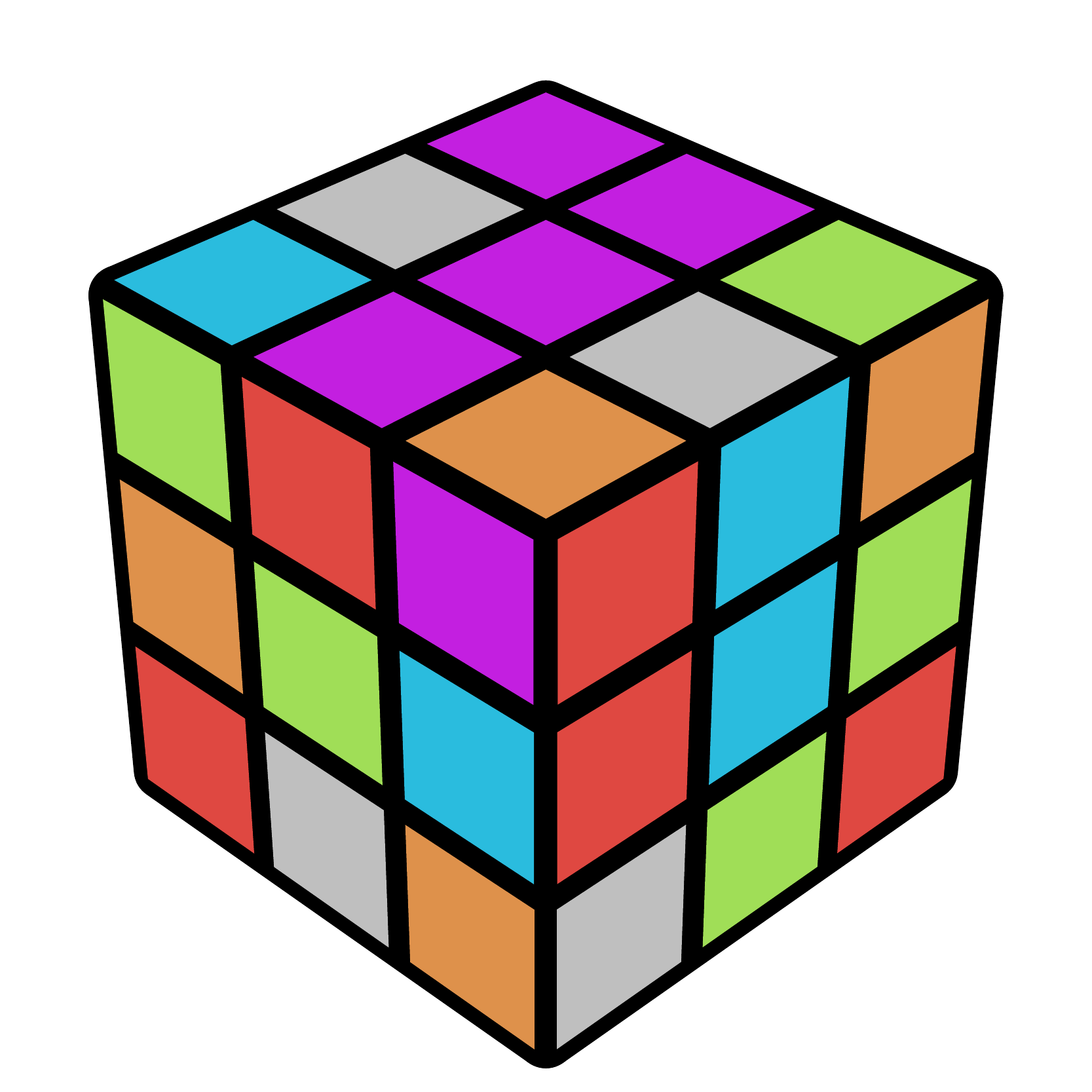}};}}
    \tikz[overlay,remember picture]{\draw[-latex,thick] (1ALcubeisom) -- (1ALcubeisom-|1ARcubeisom.west) node[midway,below,text width=2.5cm]{\hspace{7ex}$\cong$};}
    \caption{Two Rubik's cubes of the same shape, that are isomorphic. The one on the right is a copy of the left one, except that colors have been replaced consistently.}
    \label{fig:cube_isom}
\end{figure}

%
\begin{definition}[Isomorphism between STRIPS instances]
    \label{def:stripsiso}
    Let $P = \langle F, I, O, G \rangle$ and $P' = \langle F', I', O', G' \rangle$ be two STRIPS instances. An \emph{isomorphism} from $P$ to $P'$ is a pair $(\isomsymb{}, \isompsymb{})$ of bijections 
    $\isomsymb{}: F \rightarrow F'$ and $\isompsymb{}: O \rightarrow O'$ that respect the following three conditions:%
    \begin{align}
        &\forall o \in O, \, \isomp{o} = \langle \isom{\pre{o}}, \isom{\eff{o}} \rangle \label{eq:def_strips_iso_m}\\
        &\isom{I} = I' \label{eq:def_strips_iso_i}\\
        &\isom{G} = G' \label{eq:def_strips_iso_g}
    \end{align}
\end{definition}%
Where, by a slight abuse of notation, for any two sets $F_1$ and $F_2$ of fluents of $F$,
%
%
\begin{equation*}
    \isom{ F_1 \cup \neg F_2 }
    = \isom{F_1} \cup \neg \isom{F_2}
\end{equation*}
An immediate property of this definition is that it carries over all plans: any sequence $o_1, \ldots, o_n$ of operators of $O$ is a plan for $P$ if, and only if, the corresponding plan $\isomp{o_1}, \ldots, \isomp{o_n}$ is a plan for $P'$. This homomorphism property is enforced by equation~(\ref{eq:def_strips_iso_m}). Similarly, all solution-plans carry over, as enforced by the additional conditions defined in equations~(\ref{eq:def_strips_iso_i}) and (\ref{eq:def_strips_iso_g}).

We can now introduce the problem \stripsiso{} formally, in order to
analyze its complexity:

\begin{problem}{STRIPS Isomorphism problem \stripsiso{}}
    \label{prob:stripsiso}
    \begin{problembody}
        \emph{\textbf{Input}}: & Two STRIPS instances $P$ and $P'$\\
        \emph{\textbf{Output}}: & An isomorphism $(\isomsymb{}, \isompsymb)$ between $P$ and $P'$, if one exists 
    \end{problembody}
\end{problem}
%
%
\begin{proposition}
    The decision problem corresponding to \stripsiso{} is \GI{}-complete
\end{proposition}

The rest of this section is dedicated to the proof of this result. We first show the \GI{}-hardness of the problem, and then that it belongs to \GI{}.

\begin{lemma}
\stripsiso{} is \GI{}-hard
\end{lemma}

\begin{proof}
The proof consists in a reduction from the \changes{Directed} Graph Isomorphism problem to \stripsiso{}. 
Let $(\graph, \graph')$ be an instance of the \changes{Directed} Graph Isomorphism problem, where $\graph = (V, E)$ and $\graph' = (V', E')$ \changes{are directed graphs}. 
The proof relies on Construction~\ref{cons:oriented_graph}, which gives us in polynomial time the STRIPS planning problems $P_\graph{}$ and $P_{\graph{}'}$.
%

We show that there exists an isomorphism $u: V \rightarrow V'$ between $\graph{}$ and $\graph{}'$ iff there exists an isomorphism $(\isomsymb{}, \isompsymb{})$ between $P_\graph{}$ and $P_{\graph{}'}$.
The main idea consists in, first, identifying mappings $u$ and $\isomsymb{}$, and second, showing that the morphism condition between the edges of graphs $\graph{}$ and $\graph{}'$ is enforced by the morphism condition on the operators of STRIPS instances $P_{\graph{}}$ and $P_{\graph{}'}$, and vice-versa.

$(\Rightarrow)$ Suppose that there exists a graph isomorphism $u: V \rightarrow V'$ between $\graph{}$ and $\graph{}'$, and let us show that there exists an isomorphism between $P_\graph{}$ and $P_{\graph{}'}$. 
We define the transformation $\isompsymb{}$ on elements of $O$ by $\isomp{\action{\pre{o}}{\eff{o}}} = \action{u(\pre{o})}{u(\eff{o})}$.
We will show that $\isompsymb{}: O \rightarrow O'$ is well-defined and that the pair $(u, \isompsymb{})$ forms an isomorphism between $P_\graph{}$ and $P_{\graph{'}}$.
For any $o\in O$, by construction, there exists a unique pair $(v_1, v_2) \in V^2$ such that $o = \moveaction{v_1}{v_2}$. Thus, we have that
%
\begin{align*}
    o \in O
    &\textit{ iff } (v_1, v_2) \in E  \\
    &\textit{ iff } (u(v_1), u(v_2)) \in E' \ \ \ \ \ \changes{\text{(since $u$ is a graph isomorphism)}} \\
    &\textit{ iff } \moveaction{u(v_1)}{u(v_2)} \in O' \ \ \ \ \ \changes{\text{(by definition of \textsf{move} in Construction~\ref{cons:oriented_graph})}} \\
    &\textit{ iff } \isomp{\moveaction{v_1}{v_2}} \in O' \ \ \ \ \ \changes{\text{(by our definition of $\isompsymb$)}} \\
    &\textit{ iff } \isomp{o} \in O'
\end{align*}
%
Thus, we have shown that $P_\graph{}$ and $P_{\graph{'}}$ are isomorphic.

$(\Leftarrow)$ Suppose that $P_\graph{}$ and $P_{\graph{}'}$ are isomorphic, and that there exists an isomorphism $(\isomsymb{}, \isompsymb{})$ between them. We will show that there exists an isomorphism between $\graph{}$ and $\graph{'}$. By hypothesis, we have $\isomsymb{}: F \rightarrow F'$ (or $\isomsymb{}: V \rightarrow V'$) and $\isompsymb{}: O \rightarrow O'$ two bijections.

In the following, we will denote by $g$ and $g'$ the bijections $g: E \rightarrow O$ and $g': E' \rightarrow O'$, that exist by the  construction
(e.g. $g((v_1, v_2)) = \moveaction{v_1}{v_2}$).

Let us show that the function $\isomsymb{}$ is a graph isomorphism between $\graph{}$ and $\graph{'}$. We have that, for any $e = (v_1, v_2) \in E$, $g(e) = \moveaction{v_1}{v_2} \in O$. So $\isompsymb{} \circ g(e) = \isomp{\moveaction{v_1}{v_2}}$, and as such, $\isompsymb{} \circ g(e) = \moveaction{\isom{v_1}}{\isom{v_2}} \in O'$. Then, $g'^{-1} \circ \isompsymb{} \circ g(e) = (\isom{v_1}, \isom{v_2})$, but also $g'^{-1} \circ \isompsymb{} \circ g(e) \in E'$. As a consequence, $(\isom{v_1}, \isom{v_2}) \in E'$.

With similar arguments, as $g$, $g'$ and $\isompsymb{}$ are bijections, the converse can be shown. As a consequence, $\isompsymb{}$ is a graph isomorphism between $\graph{}$ and $\graph{'}$.
\end{proof}

\begin{lemma}
The decision problem corresponding to \stripsiso{} is in \GI{}
\end{lemma}

\begin{proof}
The proof follows a reduction from \stripsiso{} to the Finite Model isomorphism problem, as introduced in Definition~\ref{def:finite_model}.

The demonstration is based on the following construction: for any planning problem $P = \langle F, I, O, G \rangle$, we build the finite model
\begin{align*}
    &M_P = \langle V, \relation_F, \relation_I, \relation_O, \relation_G, \relation_{\presymb}, \relation_{\effpsymb}, \relation_{\effmsymb} \rangle \\
    &V = F \,\changes{\cup}\, O \\
    &\text{For } X \in \{F, I, O, G\}\text{, }\relation_X = X \\
    &\text{For each } \mathcal{S} \in \mathcal{C}\text{, }\relation_{\mathcal{S}} = \left\{ (o, f) \in V^2 \mid o \in O \textit{ and } f \in \mathcal{S}(o) \right\}
\end{align*}
where we recall that $\mathcal{C} = \{\presymb{}, \effpsymb{}, \effmsymb{}\}$. \changes{Note that $F$ and $O$ form a partition of $V$.} We will show that any two STRIPS planning problems $P$ and $P'$ are isomorphic iff $M_P$ and $M_{P'}$ are isomorphic.   

Let us denote $M_P = \langle V, \relation_F, \ldots, \relation_{\effmsymb{}}\rangle$ and $M_{P'} = \langle V', \relation'_{F}, \ldots, \relation'_{\effmsymb{}}\rangle$

$(\Rightarrow)$ Suppose that there exists an isomorphism $(\isomsymb{}, \isompsymb{})$ between $P$ and $P'$. Define the mapping $g : V \rightarrow V'$ such that, for $x\in V$,
\begin{equation}
    g(x) = 
    \begin{cases}
        \isom{x} & \text{ if } x \in F \\
        \isomp{x} & \text{ if } x \in O
    \end{cases}
\end{equation}
$g$ is immediately a bijection, by hypothesis on $(\isomsymb{}, \isompsymb{})$. In addition, for $X \in \{F, I, O, G\}$, $\relation_X(v)$ iff $\relation'_X(g(v))$, by hypothesis on  $(\isomsymb{}, \isompsymb{})$.

Let $o \in O$, $p \in F$. We have that, for any $\mathcal{S} \in \mathcal{C} = \{\presymb{}, \effpsymb{}, \effmsymb{}\}$,
\begin{align}
    \relation_{\mathcal{S}}(o, p) &\textit{ iff } o\in O \textit{ and } p \in \mathcal{S}(o) \label{eql:r_pre_proof_0}\\
    &\textit{ iff } \isomp{o} \in O' \textit{ and } \isom{p} \in \mathcal{S}(\isomp{o}) \label{eql:r_pre_proof_1}\\
    &\textit{ iff } \relation'_{\mathcal{S}}(\isomp{o}, \isom{p}) \nonumber \\
    &\textit{ iff } \relation'_{\mathcal{S}}(g(o), g(p)) \nonumber
\end{align}
The passage from~(\ref{eql:r_pre_proof_0}) to~(\ref{eql:r_pre_proof_1}) is by definition of the isomorphism. The other equivalences follow mostly by definition. This proves that $\finitemodel_P$ and $\finitemodel_{P'}$ are isomorphic.


$(\Leftarrow)$ Suppose that $\finitemodel_P$ and $\finitemodel_{P'}$ are isomorphic, and that $g: V \rightarrow V'$ is an isomorphism between the two models. Let us define $\isomsymb{} = g_{\vert F}$ (resp. $\isompsymb{} = g_{\vert O}$) the restriction of $g$ on the subdomain $F$ (resp. $O$). Clearly, we have that $\isomsymb{}: F \rightarrow F'$, as otherwise there would exist an element $v \in V$ such that $\relation_F(v)$ but \emph{without} $\relation'_F(g(v))$, violating the isomorphism hypothesis on $g$. Similarly, we have that $\isompsymb{}: O \rightarrow O'$.

With similar arguments as above, we have, for any $o \in O$,
\begin{align}
            &o = \action{\pre{o}}{\eff{o}}  \nonumber \\ 
\textit{ iff }
            &\forall p \in F, \forall \mathcal{S} \in \mathcal{C}, p \in \mathcal{S}(o) \Leftrightarrow \relation_{S}(o, p)   \label{eql:gi_isom_conv_1} \\  
\textit{ iff }
            &\forall p \in F, \forall \mathcal{S} \in \mathcal{C}, p \in \mathcal{S}(o) \Leftrightarrow \relation'_{\mathcal{S}}(g(o), g(p)) \label{eql:gi_isom_conv_2} \\
\textit{ iff }
            &\forall p \in F, \forall \mathcal{S} \in \mathcal{C}, p \in \mathcal{S}(o) \Leftrightarrow g(p) \in \mathcal{S}(g(o)) 
            \label{eql:gi_isom_conv_3} \\
\textit{ iff }
            &g(o) = \action{g(\pre{o})}{g(\eff{o})} \label{eql:gi_isom_conv_5}\\
\textit{ iff }
            &\isomp{o} = \action{\isom{\pre{o}}}{\isom{\eff{o}}} \nonumber
\end{align}

The relations between the first line and~(\ref{eql:gi_isom_conv_1}), as well as between~(\ref{eql:gi_isom_conv_2}) and~(\ref{eql:gi_isom_conv_3}) hold by construction of $\finitemodel_P$ and $\finitemodel_{P'}$. The equivalence between (\ref{eql:gi_isom_conv_1}) and~(\ref{eql:gi_isom_conv_2}) comes from the hypothesis that $g$ is an isomorphism. 

This proves that $(\isomsymb{}, \isompsymb{})$ is a homomorphism, and thus an isomorphism by choice of its domain and codomain.
\end{proof}

The results still hold if we do not enforce conditions~(\ref{eq:def_strips_iso_i}) and (\ref{eq:def_strips_iso_g}), that the initial and goal states of $P$ and $P'$ are in bijection. Indeed, the hardness proof relies on a reduction from the Graph Isomorphism problem, with graphs that do not have initial or goal nodes, which renders trivial the initial and goal states of the construction. Conversely, the proof that \stripsiso{} belongs to class~\GI{} can include, or not, the relations $\mathcal{R}_I$ and $\mathcal{R}_G$ that take into account the information concerning initial and goal states, and still remain correct for the version of $\stripsiso{}$ without conditions on the initial and goal states. This means that 
the hardness of~\stripsiso{} comes from matching the inner structure of the state-space, and that additional properties on some states (like being initial states or goal states) do not impact significantly the complexity of the problem. This is consistent with our intuition of class~\GI{}: it is known that finding a color-preserving isomorphism between colored graphs (i.e., an isomorphism that conserves a given property on nodes) is also a problem that is complete for this class~\cite{zemlyachenko1985graph}. 

\changes{An isomorphism between a planning instance and itself is called an \emph{automorphism}, or \emph{symmetry}. Intuitively, a symmetry is a set of pairs of elements of the model (for instance operators or fluents) such that each element is interchangeable with the element it is paired with, in a way that does not modify the behaviour of the model, but only its description. It is well known that, when it comes to solving a CSP\cite{DBLP:journals/constraints/CohenJJPS06}, a SAT instance~\cite{sakallah2021symmetry}, or a planning instance~\cite{fox1999detection}, finding the symmetries of the model as a preprocessing step can prove fruitful, since it can dramatically reduce the size of the search space (when symmetry-breaking techniques are implemented).
In planning, in particular, finding the symmetries of a STRIPS planning instance boils down to finding operators or fluents that have similar roles, and preventing the planner from exploring branches of the search space that are too similar. 

A common example (due to Pochter et al.~\cite{pochter2011exploiting}), that illustrates the need for breaking symmetries in search, occurs in the Gripper domain (as found in the International Planning Competition). In a Gripper instance, a two-armed robots has to move all balls from a room to the room next door, grabbing and moving them individually with its two grippers. Since all balls are identical, they are all interchangeable (and are said to be symmetric). This is usually a problem for planners: since the $n$ balls can be brought to the other room in any order, there are $n!$ different shortest paths to the goal state, which are all explored simultaneously, thus severely harming the performance of the planner.

Symmetries of a planning instance are closely related to the symmetries of the Problem Description Graph (PDG)~\cite{pochter2011exploiting}, which is a colored graph that describes the instance (expressed in Finite-Domain Representation (FDR)). Any symmetry of the PDG is equivalent to a symmetry of the associated planning instance~\cite{shleyfman2021}: as a consequence, to compute the symmetries of the planning task, one can resort to a solver that computes a set of generators for the group of symmetries of a colored graph. In this context, a set of generators for the group of symmetries is a set $\Gamma$ of symmetries that can be applied one after the other, such that any other symmetry can be expressed in terms of a composition of symmetries of $\Gamma$. Computing a set of generators of the symmetries of a planning instance $P$ is known to be $\GI$-complete~\cite{shleyfman2021}, even with drastic syntactical restrictions on the instance (such as limiting the number of preconditions and effects of each operator).

Even though isomorphisms, as defined here, are closely related to symmetries, they are much more general. This is why there is no immediate group structure on the set of isomorphisms between two planning instances $P$ and $P'$, when $P \not= P'$: as a consequence, it does not make sense to seek a set of generators for the isomorphisms between $P$ and $P'$. However, the PDGs of the planning instances at hand can be still used to find an isomorphism between $P$ and $P'$, when passed to a graph isomorphism solver. In the following sections, we will introduce other more general notions of homomorphisms, that are less restrictive while still carrying over some properties from one instance to the other.
}


\section{The STRIPS Subinstance isomorphism problem}
\label{sec:strips_subiso_complexity}

Let us now introduce problems $\stripssubisogeneral$ and $\stripssubiso$, which are concerned with finding (different kinds of) isomorphisms between a planning instance $P$ and some subinstance of another STRIPS instance $P'$. In this section, we settle the complexity of both problems, and show that they are \NP{}-complete. We use this result in order to propose, in the next section, an algorithm for $\stripssubiso$ and $\stripssubisogeneral$. This algorithm is based on a reduction to SAT, assisted by a preprocessing phase that relies on constraint propagation.

\changes{In essence, a subinstance of a planning instance $P = \langle F, O, I, G \rangle$ is a planning instance $P_s$ which shares the same initial state and goal as $P$, but whose fluents and operators $F_s$ and $O_s$ are subsets of the fluents and operators $F$ and $O$ of $P$. Note that any operator that occurs in $O_s$ also appears \emph{as is} in $O$, and is solely defined on the fluents $F_s$. 
}


We begin by introducing the notion of \emph{homogeneous subinstance isomorphism}. It is concerned with finding an isomorphism between $P$ and a subinstance of $P'$, which does not necessarily conserve the initial state and goal. That is to say, it maps operators of $P$ to operators of $P'$, and fluents alike, in a way that ensures that all plans (and not just solution-plans) can be transferred from $P'$ to $P$. 
In the case of such a homomorphism between two versions of the Rubik's cube, it would consist in mapping each operator of $P$ to an equivalent operator in $P'$ even though $P'$ may have extra operators (such as 180-degree rotations \changes{for instance, which are equivalent to a sequence of two operators that rotate the same face}).

\begin{definition}[Homogeneous subinstance isomorphism]
    \label{def:stripssubiso_gen} Consider two STRIPS instances 
    $P = \langle F, I, O, G \rangle$ and $P' = \langle F', I', O', G' \rangle$. A \emph{homogeneous subinstance isomorphism} from $P$ to $P'$ is a pair $(\isomsymb{}, \isompsymb{})$ consisting of the \emph{injective} mapping $\isomsymb{}: F  \rightarrow F'$ and the mapping $\isompsymb{}: O \rightarrow O'$ that respect condition (\ref{eq:def_strips_iso_m}) of Definition~\ref{def:stripsiso}.
\end{definition}
\begin{problem}{STRIPS Homogeneous Subinstance Isomorphism \stripssubisogeneral{}}
    \label{prob:stripssubisogeneral}
    \begin{problembody}
        \emph{\textbf{Input}}: & Two STRIPS instances $P$ and $P'$\\
        \emph{\textbf{Output}}: & A \emph{homogeneous subinstance isomorphism} $(\isomsymb{}, \isompsymb)$ between $P$ and $P'$, if one exists 
    \end{problembody}
\end{problem}
A homogeneous subinstance isomorphism between $P$ and $P'$ is useful, for instance, in the case 
when the smaller problem $P$ has already been solved: the image of a solution plan for $P$
could be added as an operator to $P'$ as a high-level action.
In the case of a Rubik's cube, suppose that $P$ and $P'$ have the same set of operators, but the goal of $P$ is simply to have one face all red whereas $P'$ has the traditional goal that all faces are a uniform colour. Then there is a homogeneous subinstance isomorphism from $P$ to $P'$ and a solution-plan for $P$ can be converted into an \changes{operator} that could be added to $P'$.
\changes{Such additional operators (often called \emph{macro-operators}~\cite{korf1985macro}) have been widely studied, as they have the potential to significantly speed up search~\cite{korf1985macro}. The problem of learning them has been previously addressed~\cite{korf1983learning,botea2005macro,minton1985selectively}, and indeed, a common approach is to reuse plans from a problem that has already been solved~\cite{fikes1972learning}. Our work provides a formal basis to detect instances whose solution-plans can be translated into a macro-operator for another instance. 
}
\changes{Indeed, let $P''$ be a modified version of $P'$, so that a plan for $P$ has been added to $P''$ as a new operator.} Then there is a subinstance isomorphism between $P'$ and $P''$ but in this case, the two problems have the same initial states and goals.
This leads \changes{in turn} to the following more precise notion of subinstance isomorphism which takes into account the information provided by the initial state and goal. This allows us to carry over \emph{solution-plans} from one problem to the other:
a solution for the smaller problem directly provides a solution for the larger problem. 


\begin{definition}[Subinstance isomorphism]
    \label{def:stripssubiso}
    A \emph{subinstance isomorphism} from $P$ to $P'$ is a homogeneous subinstance isomorphism that respects conditions (\ref{eq:def_strips_iso_i}) and (\ref{eq:def_strips_iso_g}) of Definition~\ref{def:stripsiso}.
\end{definition}

\begin{proposition}
    \label{prop:stripsubiso_sol}
    If there is a subinstance isomorphism from $P$ to $P'$ and $P$ has a solution-plan,
    then so does $P'$ (the image of the plan under the subinstance isomorphism).
\end{proposition}

One application of Proposition~\ref{prop:stripsubiso_sol} is in plan compilation. Suppose that there is a subinstance homomorphism from $P$ to $P'$ and that all solutions to $P'$ \changes{(up to a certain length $L$)} are stored in a compiled form that allows conditioning in polynomial time~\cite{DBLP:journals/jair/DarwicheM02}. Then we can find the compiled form of all solutions \changes{of length at most $L$} to (the isomorphic image of) $P$ in polynomial time. 
Note, however, that if the aim is to detect unsolvable instances, the notion of embedding introduced in Section~\ref{sec:strips_embedding_complexity} provides a stronger test than the contrapositive of Proposition~\ref{prop:stripsubiso_sol}.

\begin{problem}{STRIPS Subsintance Isomorphism \stripssubiso{}}
    \label{prob:stripssubiso}
    \begin{problembody}
        \emph{\textbf{Input}}: & Two STRIPS instances $P$ and $P'$\\
        \emph{\textbf{Output}}: & A subinstance isomorphism $(\isomsymb{}, \isompsymb)$ between $P$ and $P'$, if one exists 
    \end{problembody}
\end{problem}

The main difference between \stripsiso{} and \stripssubiso{} is that, in \stripssubiso{}, we relax the condition on the bijectivity of $\isomsymb{}$ and $\isompsymb{}$, to account for the possible difference in size between $P$ and $P'$. Injectivity of $\isomsymb{}$ is still required in order to prevent fluents from being merged together by the mapping 
(but injectivity of $\isompsymb{}$ and surjectivity of both mappings are relaxed). All other conditions remain the same. 
\changes{Note that, even when the injectivity of the mapping $\isompsymb{}$ over the operators is not enforced, $\isompsymb{}$ is still often injective. This is due to the injectivity of $\isomsymb{}$: two operators of $O$ that have the same image by $\isompsymb{}$ must also have the same preconditions, the same positive effects, and the same negative effects. Thus, except in the case where two operators are identical (except maybe for their names), the injectivity of $\isompsymb{}$ is a consequence of the injectivity of $\isomsymb{}$, even when not explicitely enforced.}

The main result of this section is presented below. The proof is based on a reduction from the Subgraph Matching problem~\cite{DBLP:journals/jacm/Ullmann76}, which is known to be \NP{}-complete~\cite{DBLP:conf/stoc/Cook71}. 
As such, we introduce that problem before stating our result. Essentially, it consists in finding a mapping $g$, that defines an isomorphism between \changes{an undirected graph} $\graph{}$ and the subgraph $(g(V), E' \, \cap \, g(V) \times g(V))$ of $\graph{'}$ (i.e. the induced subgraph
of $\graph{'}$ on $g(V)$). 


\begin{problem}{Subgraph Matching problem}
    \label{prob:subgraph_matching}
    \begin{problembody}
        \emph{\textbf{Input}}: & Two \changes{undirected} graphs $\graph{}(V, E)$ and $\graph{'}(V', E')$\\
        \emph{\textbf{Output}}: & An injective mapping $g: V \rightarrow V'$ such that, for any $v_1, v_2 \in V$, $\{v_1, v_2\} \in E$ iff $\{g(v_1), g(v_2)\} \in E'$.
    \end{problembody}
\end{problem}





\begin{proposition}
    \label{prop:stripsubiso_npc}
    The decision problem corresponding to \stripssubiso{} is \NP{}-complete
\end{proposition}

\begin{proof}
In order to prove that $\stripssubiso$ belongs to \NP{}, it suffices to resort to the certificate-based definition of the class \NP{}, and observe that the mappings $\isomsymb$ and $\isompsymb$ constitute a polynomial size certificate that can be checked in polynomial time.

The proof that $\stripssubiso{}$ is \NP{}-hard consists in a reduction from the Subgraph Matching problem, which is straightforward with the construction that we proposed earlier.

Let $(\graph{}, \graph{'})$ be \changes{a pair of undirected graphs}, an instance of the Subgraph Matching problem, and let us follow Construction~\ref{cons:symmetric_graph} to build planning problems $P_{\graph{}}$ and $P_{\graph{'}}$. We show that there exists a subgraph matching $g$ between $\graph{}$ and $\graph{'}$ iff there exists a subinstance isomorphism between $P_{\graph{}}$ and $P_{\graph{'}}$.

$(\Rightarrow)$ Suppose that there exists a subgraph matching $g: V \rightarrow V'$ between $\graph{}$ and $\graph{'}$. Then by construction, as $F = V$ and $F' = V'$, $g$ is also an injective mapping between $F$ and $F'$. In addition, let us define the mapping $\isompsymb{}: O \rightarrow O'$ such that $\isompsymb{} : \moveaction{v_1}{v_2} \mapsto \moveaction{g(v_1)}{g(v_2)}$. $\isompsymb{}$ is well-defined, as $\{v_1, v_2\} \in E$ iff $\{g(v_1), g(v_2)\} \in E'$, so $\moveaction{v_1}{v_2} \in O$ iff $\moveaction{g(v_1)}{g(v_2)} \in O'$. In addition, as $g$ is injective, so is $\isompsymb{}$. As a consequence, $(g, \isompsymb{})$ is a subinstance isomorphism between $P_{\graph{}}$ and $P_{\graph{'}}$.

$(\Leftarrow)$ Suppose that there exists a subinstance isomorphism $(\isomsymb{}, \isompsymb{})$ between $P_{\graph{}}$ and $P_{\graph{'}}$. As above, $\isomsymb{}: V \rightarrow V'$ is an injective mapping. In addition, we have that%
\begin{align*}
    &(v_1, v_2) \in E \\
    \text{ iff } &\moveaction{v_1}{v_2} \in O \ \ \ \ \ \changes{\text{(by definition of Construction~\ref{cons:symmetric_graph})}} \\
    \text{ iff } &\isomp{\moveaction{v_1}{v_2}} \in O' \ \ \ \ \ \changes{\text{(since $\isompsymb$ is a mapping from $O$ to $O'$)}} \\
    \text{ iff } &\moveaction{\isom{v_1}}{\isom{v_2}} \in O' \ \ \ \ \ \changes{\text{(by definition of a subinstance isomorphism)}} \\
    \text{ iff } &(\isom{v_1}, \isom{v_2}) \in E' \ \ \ \ \ \changes{\text{(by definition of Construction~\ref{cons:symmetric_graph})}}
\end{align*}
As a consequence, $\isomsymb{}$ is a subgraph matching between $\graph{}$ and $\graph{'}$.
\end{proof}

In addition, it is clear that $\stripssubisogeneral{}$ is in \NP{}.
As the above proof of NP-hardness of $\stripssubiso{}$ is independent of the initial and goal states,
it also applies to the problem $\stripssubisogeneral{}$.

\begin{corollary}
    The decision problem corresponding to \stripssubisogeneral{} is \NP{}-complete
\end{corollary}

\section{An algorithm for \stripssubiso{}}
\label{sec:ssi_algorithmic}

In this section, we present an algorithm for \changes{solving the \stripssubiso{}} problem, for which the pseudo-code is presented in Algorithm~\ref{alg:subiso_finder}. This algorithm is based on a compilation of the problem into a propositional formula, which is then passed to a SAT solver. It is completed by a preprocessing step, based on constraint propagation, that allows us to prune impossible mappings early on.

Given two STRIPS instances $P$ and $P'$, the algorithm outputs, when possible, a subinstance isomorphism $(\isomsymb{}, \isompsymb{})$. Algorithm~\ref{alg:subiso_finder} consists in two main phases.
The first phase, that spans lines~\ref{algline:cp_start} to~\ref{algline:cp_end}, consists in pruning as many associations between fluents (resp. operators) of problem $P$ and fluents (resp. operators) of problem $P'$ that are impossible, because of some syntactical inconsistencies (described below) that are then propagated. The second phase, that starts at line~\ref{algline:encode_SAT}, consists in a search phase, by means of an encoding of the problem into a CNF formula, that is then passed to a SAT solver.

\begin{algorithm}[tb]
    \caption{to find a subinstance isomorphism}
    \label{alg:subiso_finder}
    \textbf{Input}: Two STRIPS instances $P$ and $P'$\\
    \textbf{Output}: A subinstance isomorphism between $P$ and $P'$ if one exists
    \begin{algorithmic}[1] 
    \State \changes{$\domainsymb{} := $ \text{Initialize\_domains}$(P, P')$} \label{algline:subiso_finder_init}
    \Statex /* Prune impossible associations */
    \State $Q := F \cup O$ \label{algline:cp_start}
    \While{$Q \not= \varnothing$}
    \State $v := Q.\text{Pop()}$
    \State $\changes{\text{Revise\_domain}(\domainsymb{}, v)}$
    \If {\changes{\text{$\domainsymb{}$ has been updated}}}
    \If {$\domain{v} = \varnothing$}  \textbf{return} UNSAT
    \Else \text{ } $Q.\text{Add}( \{ v' \mid v' \text{ related to } v \} $) \EndIf \label{algline:cp_end}
    \EndIf
    \EndWhile
    \Statex /* Search phase through a SAT solver */
    \State $\varphi := \text{Encode\_to\_SAT}(P, P', \mathcal{D})$ \label{algline:encode_SAT}
    \State \textbf{return} $\text{Interpret}(\text{Solver.Find\_model}(\varphi))$ 
    \end{algorithmic}
\end{algorithm}

\subsection{Pruning invalid associations}
\label{sec:ssi_pruning}

By \emph{association} between fluents, we mean a pair $(f, f') \in F \times F'$ such that $f'$ is a candidate for the value of $\isom{f}$. Similarly, we call an \emph{association} between operators a pair $(o, o') \in O \times O'$ such that $o'$ is a candidate for the value of $\isomp{o}$. Detecting early on associations that can not be part of a valid subinstance isomorphism reduces the size of the search space.

In order to prune as many inconsistent associations as possible, we use a technique similar to constraint propagation, as commonly found in the constraint programming literature.
The general idea is to maintain, for each fluent $f \in F$ of $P$, a \emph{domain} $\domain{f} \subseteq F'$ of fluents of $P'$, that consists of the plausible candidates for the value of $\isom{f}$. Similarly, each operator $o \in O$ is assigned a domain $\domain{o} \subseteq O'$, containing the plausible candidates for $\isomp{o}$. In the following, we will call fluents and operators \emph{variables}. The aim of the procedure presented below is to trim the domains of the different variables, thus alleviating the load left to the SAT solver. 

The first step is to initialize the domains. For each fluent $f \in F$, we set its initial domain $\domain{f}$ to either \changes{$I' \setminus G'$}, \changes{$G' \setminus I'$}, $I' \cap G'$ or $F' \setminus (I' \cup G')$, depending if $f$ is in $I$ \changes{only}, $G$ \changes{only}, in both or in neither, respectively. Note that if the algorithm is given an $\stripssubisogeneral{}$ instance, the distinction between domains made above does not hold anymore. In this case, for all fluents $f \in F$, we set $\domain{f} = F'$. 

The initial assignment of the domains of operators $o \in O$, however, is based on \emph{operator profiles}.
For each operator $o \in O \cup O'$, we define the vector $\opprofile{o} \in \mathbb{N}^4$, called the \emph{profile} of $o$. This vector numerically abstracts some characteristics of the operator, so that an operator $o \in O$ cannot be associated to operator $o' \in O'$ if $\opprofile{o} \not= \opprofile{o'}$.
In practice, $\opprofile{o}$ consists of the number of fluents in the precondition of $o$, the number of positive and negative fluents in the effect of $o$, and its number of \emph{strict-delete} fluents. A fluent $f$ is said to be  \emph{strict-delete} if $f \in \pre{o} \wedge f \in \effm{o}$. Then, we initialize the domain of each $o \in O$ so that 
\[
    \domain{o} = \left\{ o' \in O' \mid \opprofile{o'} = \opprofile{o} \right\}
\]


The second step is to propagate the additional constraints posed by these newly-found restrictions of the domains. The technique we propose is based on the concept of arc consistency, which is ubiquitous in the field of constraint programming. The idea consists in eliminating, from the domains of fluents (resp. operators), the candidate fluents (resp. operators) that have no support in the domain of some operator (resp. fluent). 

More specifically, let us consider a fluent $f \in F$. When an operator $o \in O$ is such that $f$ appears, negated or not, in its precondition or effect, then we say that $o$ \emph{depends} on $f$. Let us denote $d(f)$ the set of operators that depend on $f$. When $f' \in F'$, we define $d(f')$ in a similar fashion. Now suppose that $\isom{f} = f'$. As a consequence of equation (\ref{eq:def_strips_iso_m}) of Definition~\ref{def:stripsiso}, each operator of $d(f)$ must have its image by $\isompsymb{}$ in $d(f')$. Otherwise, $f$ would appear in $\pre{o}$ or $\eff{o}$, but $\isom{f}$ would not appear in $\isom{\pre{o}}$ nor $\isom{\eff{o}}$. Thus, if for some operator $o \in d(f)$ no candidate operator for its image is in $d(f')$ (\emph{i.e.}, $\domain{o} \cap d(f') = \varnothing$), then it means that $f'$ can not be chosen as the image of $f$.

In the following, we refine the argument of last paragraph by identifying $\pre{o}$ with $\pre{o'}$ and $\eff{o}$ with $\eff{o'}$. 
We thus have the following constraint for $\domain{f}$, where $\mathcal{C} = \{\presymb{}, \effpsymb{}, \effmsymb{} \}$:
%
%
\begin{equation}
    \label{eq:ac_fluents}
    \domain{f} \subseteq \Set{ f'\ | \begin{array}{l}
        \forall o \in O, \forall \absrel{} \in \mathcal{C} \text{ s.t. } f \in \absrel{}(o), \\
        \exists o' \in \domain{o} \text{ s.t. } f' \in \absrel{}(o')
  \end{array}}
\end{equation}
A similar case can be made for operators. Let $o \in O$ be any operator, and consider a candidate operator $o' \in O'$. In order for the morphism property to hold, in the case where $\isomp{o} = o'$, for every fluent $f$ of $\pre{o}$, for instance, there must exist in $\pre{o'}$ a fluent that belongs to $\domain{f}$. More generally and more formally, we have the following:
\begin{equation}
    \label{eq:ac_operators}
    \domain{o} \subseteq \left\{ o' \mid \forall \absrel{} \in \mathcal{C}, \forall f \in \absrel{}(o), \exists f' \in \domain{f} \cap \absrel{}(o') \right\}
\end{equation}
Algorithmically, we enforce these constraints using an adaptation of AC3~\cite{mackworth1977consistency,rossi2006handbook}. The algorithm revolves around the revision of the variables' domains.
Revising a variable $v$ boils down to checking that all elements of its domain still comply with the necessary condition evoked earlier, which is either equation~(\ref{eq:ac_fluents}) if $v$ is a fluent, or equation~(\ref{eq:ac_operators}) if $v$ is an operator.
The main loop, depicted in Algorithm~\ref{alg:subiso_finder}, then consists in revising all fluents and operators iteratively, by maintaining a queue $Q$ of variables to revise (line 1). The algorithm begins by revising once each variable.
If, during the revision of a variable $v$, the domain of $v$ is altered by the procedure, then all variables that are related to $v$ are added to the set of variables to revise later on (lines 5 to 9). We say that $v'$ is related to $v$ if $v$ is a fluent and $v' \in d(v)$, or conversely.

If the domain of a variable is empty, then no isomorphism exists, and the procedure ends prematurely (line 7). Otherwise, the loop ends when there is no variable left to revise. 

This procedure is often not sufficient to conclude, but greatly alleviates the pressure on the search phase, which we present in the following section.

\subsection{Propositional encoding for \stripssubiso{} and \stripssubisogeneral{}}
\label{sec:ssi_encoding}

In this section, we build the propositional formula $\varphi$ evoked earlier, from whose models an isomorphism can be extracted. $\varphi$ is built on the set of variables $\dom{\varphi}$, where
\[    
\dom{\varphi} = 
\Set{ \fassoc{i}{j} | i \in F, j \in F' } \ \cup \ 
    \Set{ \oassoc{r}{s} | r \in O, s \in O' }
\]
The propositional variable $\fassoc{i}{j}$ represents the association of fluent $i \in F$ to fluent $j \in F'$. Likewise, $\oassoc{r}{s}$ represents the association of $r \in O$ to $s \in O'$.

In the rest of this section, we show how to build formula $\varphi$, which encodes the \stripssubiso{} problem input to  Algorithm~\ref{alg:subiso_finder}. $\varphi$ consists of the conjunction of the following formulas, each of which enforces a different property. 

The formula presented in~(\ref{eq:fluents_image}) enforces that each fluent has an image which is unique. Similarly, by swapping $\fassoc{i}{j}$ variables for $\oassoc{i}{j}$ and adapting the domains of $i$ and $j$, we enforce that each operator has an image by $\isompsymb$.

\begin{align}
    \label{eq:fluents_image}
    \bigwedge_{i \in F}
    \left(
        \bigvee_{j \in \domain{i}} \fassoc{i}{j} \hspace{1ex} \wedge 
        \bigwedge_{\substack{j, k \in \domain{i} \\ j \not= k}} (\neg \fassoc{i}{j} \vee \neg \fassoc{i}{k})
    \right)
\end{align}

%
We now need to ensure that $\isomsymb{}$ \changes{is injective, which is done through~(\ref{eq:injectivity})}
\begin{equation}
    \label{eq:injectivity}
    \bigwedge_{i \in F'} \bigwedge_{\substack{j, k \in F \\ j \not= k}} \neg \fassoc{j}{i} \vee \neg \fassoc{k}{i}
\end{equation}
\changes{Note that, in Definition~\ref{def:stripssubiso_gen}, we did not require the mapping on operators $\isompsymb{}$ to be injective. However, even when not specifically enforced, $\isompsymb{}$ is still injective in any (homogeneous) subinstance isomorphism, except potentially when two operators that have the exact same precondition and effects exist in $P$. Since this is never the case in our set of benchmarks, we also enforce the injectivity of $\isompsymb{}$, in order to make the search for a mapping more efficient.}

The morphism property is enforced by formulas~(\ref{eq:morphism_inclusion}) and~(\ref{eq:morphism_reverse_inclusion}), for each $\absrel \in \mathcal{C}$.
More precisely, (\ref{eq:morphism_inclusion}) ensures that, for any $\absrel \in \mathcal{C}$ and for any operator $o \in O$, we have $\isom{\absrel{}(o)} \subseteq \absrel{}(\isomp{o})$. Conversely,~(\ref{eq:morphism_reverse_inclusion}) ensures that $\absrel{}(\isomp{o}) \subseteq \isom{\absrel{}(o)}$.
\begin{equation}
    \label{eq:morphism_inclusion}
    \bigwedge_{\substack{r \in O \\ s \in O'}} \left(
        \oassoc{r}{s} \longrightarrow \bigwedge_{i \in \absrel{}(r)} \bigvee_{j \in \absrel{}(s)} \fassoc{i}{j}
    \right)
\end{equation}
\begin{equation}
    \label{eq:morphism_reverse_inclusion}
    \bigwedge_{\substack{r \in O \\ s \in O'}} \left(
        \oassoc{r}{s} \longrightarrow \bigwedge_{j \in \absrel{}(s)} \bigvee_{i \in \absrel{}(r)} \fassoc{i}{j}
    \right)
\end{equation}
%
%
%
%
%

The formulas presented in (\ref{eq:fluents_image}) \changes{and} (\ref{eq:injectivity}),
are immediately in CNF, and the size of their conjunction is in 
$\mathcal{O}(\vert F \vert \cdot \vert F' \vert^2 + \vert O \vert \cdot \vert O' \vert^2)$ assuming $\vert F \vert \leq \vert F' \vert$ and $\vert O \vert \leq \vert O' \vert$.
In addition, the formulas presented in~(\ref{eq:morphism_inclusion}) and~(\ref{eq:morphism_reverse_inclusion}) can be readily converted into CNF by duplicating the implication in each clause, and then have a size $\mathcal{O}(\vert O \vert \cdot \vert O' \vert \cdot \vert F \vert \cdot \vert F' \vert)$.

The preprocessing step presented in Section~\ref{sec:ssi_pruning} allows us to find $\varphi$. Indeed, if it is known that fluent $i \in F$ (resp. $r \in O$) cannot be mapped to fluent $j \in F'$ (resp. $s \in O'$), then $\fassoc{i}{j}$ (resp. $\oassoc{r}{s}$) is necessarily false in any model of $\varphi$. As a consequence, as all formulas are in CNF, every positive occurrence of $\fassoc{i}{j}$ is removed in the clauses of $\varphi$, while clauses where $\fassoc{i}{j}$ appears negatively are simplified.

In order to adapt the algorithm for $\stripssubisogeneral{}$, it suffices to \changes{initialize the domain of all fluents to $F'$}. The others formulas and the rest of the algorithm remains the same.

An experimental evaluation of the algorithm can be found in Section~\ref{sec:experimental_evaluation}, along with an evaluation of the adaptation of the same algorithm for a different notion of homomorphism, that we present in the next section.

\section{STRIPS Embedding Problem}
\label{sec:strips_embedding_complexity}

In this section, we introduce the notion of \emph{embedding} between two planning instances $P$ and $P'$, which allows us to deduce that $P$ has no solution-plan if $P'$ has no solution-plan.
We can define the notion of embedding of $P'$ in $P$ informally by saying that
for each fluent $f' \in F'$ of $P'$ there is a corresponding fluent $\isom{f'} \in F$ in $P$ and that when
all fluents that are not in the image of $\isomsymb$ are ignored, $P'$ is isomorphic to 
a \changes{simplified version} of $P$. In this context, simplifying a problem means (possibly) weakening its goals and the preconditions of operators, 
and (possibly) strengthening its initial state.
This implies that if $P'$ has no solution, then $P$ also has no solution. Embedding an unsolvable problem into another problem $P$ can help understand why $P$ can not be solved. An example with Rubik's cubes is presented in Figure~\ref{fig:cube_embed}.

\begin{figure}
    \centering
    \subfloat{\tikz[remember
    picture]{\node(1ALcubeunsat){\includegraphics[width=4cm]{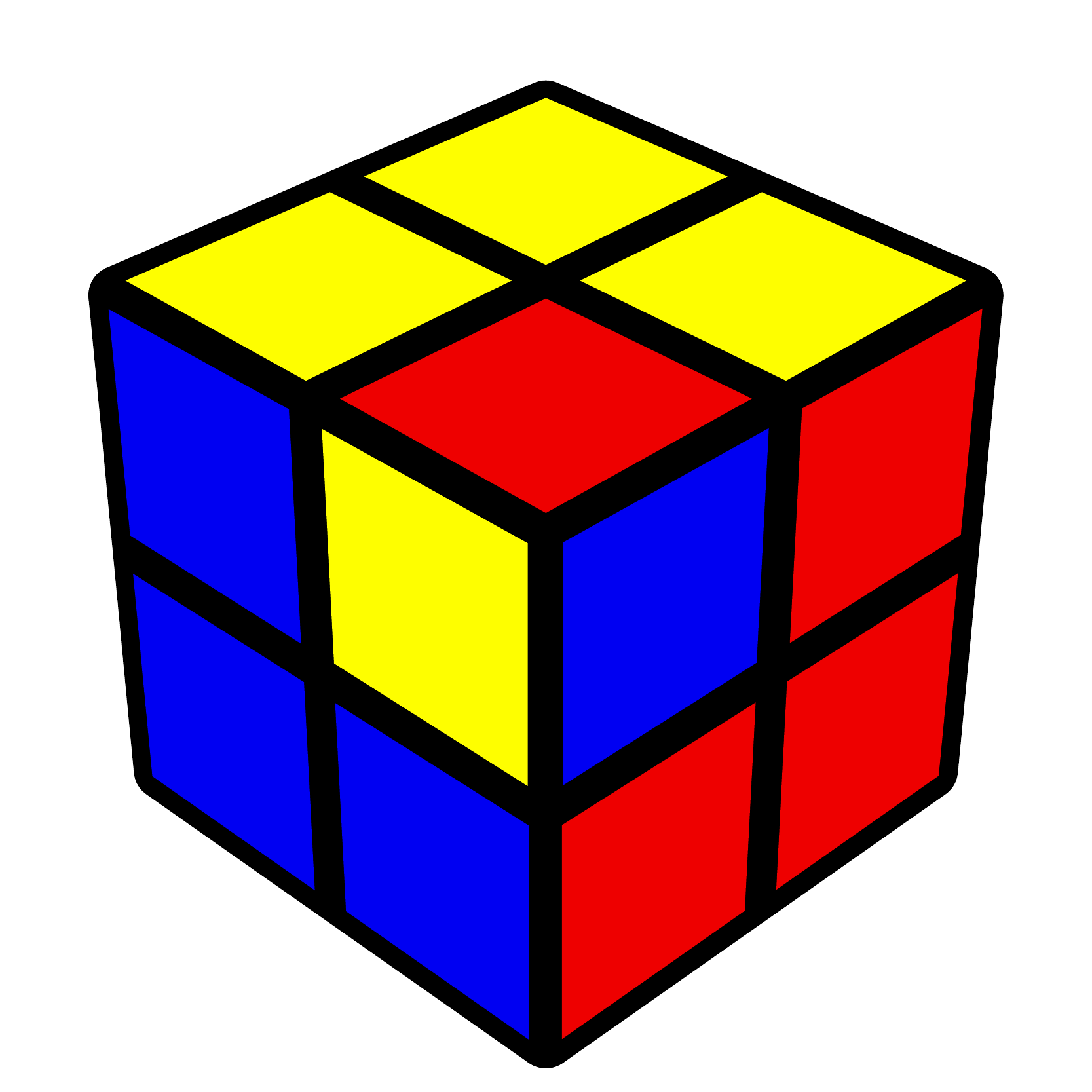}};}}%
    \hspace*{8ex}%
    \subfloat{\tikz[remember picture]{\node(1ARcubeunsat){\includegraphics[width=4cm]{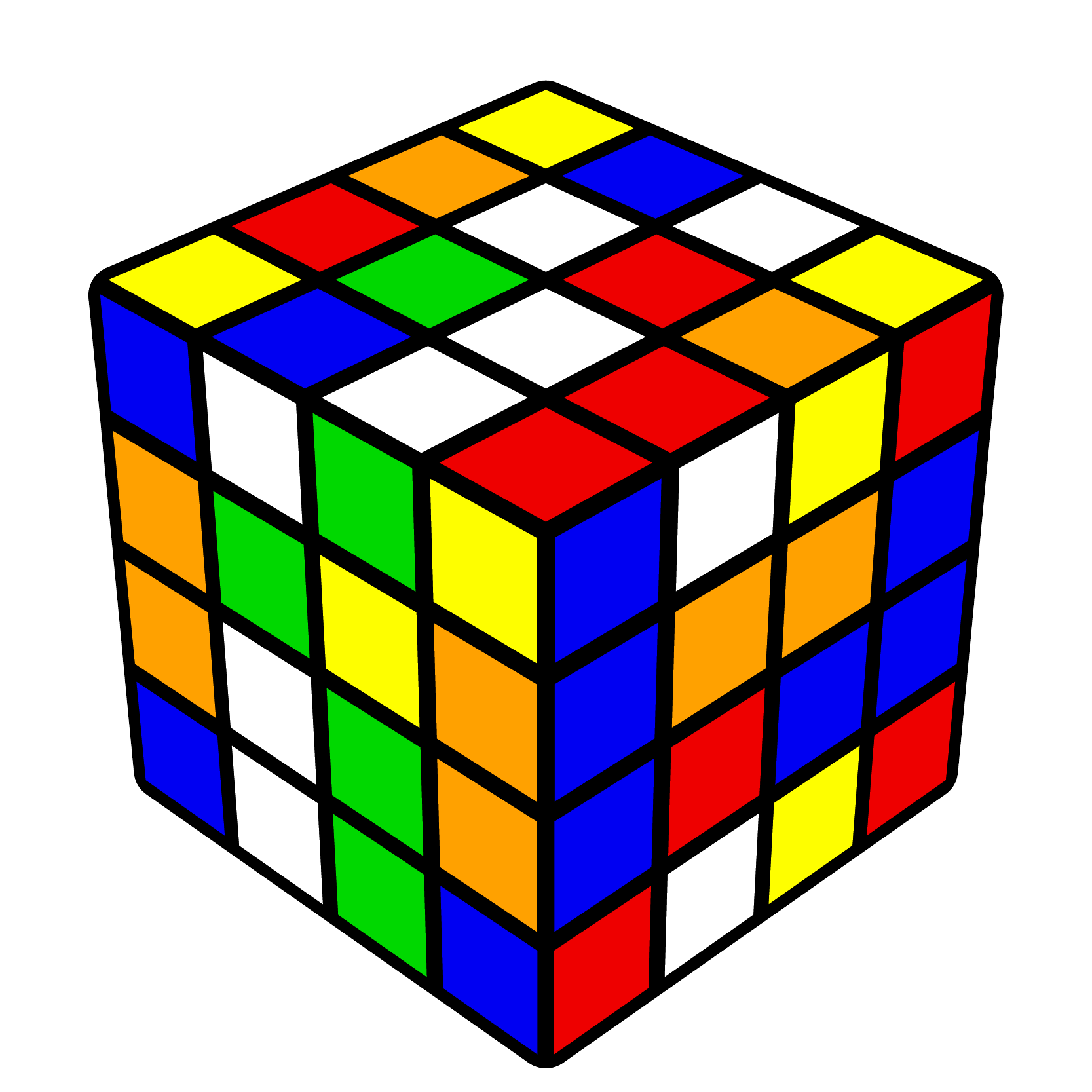}};}}
    \tikz[overlay,remember picture]{\draw[-latex,thick] (1ALcubeunsat) -- (1ALcubeunsat-|1ARcubeunsat.west) node[midway,below,text width=2.5cm]{\hspace{7ex}};}
    \caption{Through an invariance argument, it can be proven that the 2x2x2 cube on the left can not be solved. This cube can be embedded into the 4x4x4 cube on the right: the corners of the 2x2x2 cube can be mapped onto the corners of the 4x4x4 cube. When one only considers these parts of the 4x4x4 cube, we end up with a puzzle where each move that has an effect on the corner pieces can be mimicked on the smaller cube, and where the goal is to correctly position the corners. Other moves are ignored: for instance, moves that rotate one of the center slices leave the corners in place and are of no interest in this case. As no solution exists for the 2x2x2 cube, no sequence of moves exists that would correctly position the corner pieces on the 4x4x4, as we would end up with a contradiction otherwise. Thus, the 4x4x4 cube presented here can not be solved.}
    \label{fig:cube_embed}
\end{figure}

\begin{definition}
\label{def:embedding}
Let $P=\langle F, O, I, G \rangle$ and $P'=\langle F', O', I', G' \rangle$ be two planning problems.
\changes{An \emph{embedding} of $P'$ in $P$ is a pair of}
functions
$\isomsymb{} : F' \rightarrow F$ and $\isompsymb{} : O \rightarrow O'$ such that $\isomsymb{}$ is injective and
for each operator $o \in O$ that verifies
\begin{equation}
    (\effp{o} \cup \effm{o}) \cap \isom{F'} \neq \emptyset \label{eq:active_condition}
\end{equation}
we have:
\begin{align}
    &\isom{\effp{\isomp{o}}} = \effp{o} \cap \isom{F'} \label{eq:effp_morphism} \\
    &\isom{\effm{\isomp{o}}} = \effm{o} \cap \isom{F'} \label{eq:effm_morphism} \\
    &\isom{\pre{\isomp{o}}} \subseteq \pre{o} \cap \isom{F'} \label{eq:pre_morphism} 
\end{align}

In addition, we require that
\begin{align}
    &\isom{G'} \subseteq G \cap \isom{F'} \label{eq:goal_morphism}  \\
    &\isom{I'} \supseteq I \cap \isom{F'} \label{eq:init_morphism} 
\end{align}
\end{definition}

\changes{When there exists an embedding of $P'$ in $P$, we simply say that $P'$ \emph{embeds} in $P$.}

\changes{
\subsection{State-space of an embedded problem}    

In the following, we study the state-space that underlies an embedded instance. We relate embeddings to a class of abstractions, which can be used to compute heuristic values during a state-space search. Abstractions thus have strong ties with unsolvability detection, since various abstraction-based heuristics are efficient at detecting unsolvable states~\cite{hoffmann2014distance}, sometimes including the initial state. In that regard, they share this purpose with embeddings. 

An abstraction of a transition system $\lts{} = \langle S, L, T, s_I, S^G \rangle$ is a function $\alpha: S \rightarrow \alpha(S)$ that maps states of $\lts$ into abstract states while preserving the paths that already exist in the LTS, with the aim of obtaining a more compact transition system $\lts^\alpha$, called the \emph{abstract transition system}. This transition system is such that $\lts^\alpha = \langle S^\alpha, L, T^\alpha, s_I^\alpha, S^{G\alpha}\rangle$, where $S^\alpha = \alpha(S)$, $s_I^\alpha = \alpha(s_I)$, $S^{G\alpha} = \alpha(S^G)$, and $T^\alpha = \left\{ \langle \alpha(s_1), \ell, \alpha(s_2)\rangle \mid \langle s_1, \ell, s_2 \rangle \in T \right\}$. Abstracting the state-space of a planning problem $P$ (i.e. the LTS $\lts^P$) is a common technique for reducing the size of the search space, thus making it easier to explore by shrinking (but preserving the existing) paths. If no path from the initial state to a goal state exists in some abstracted state-space $\lts^{\alpha}$, where $\alpha$ is an abstraction of $\lts^P$, then the planning problem $P$ has no solution. 
A particular class of abstractions is the one that stems from \emph{projections}.
Given a planning instance $P = \langle F, I, O, G\rangle$, its state-space $\lts = \langle 2^F, O, T, I, S^G \rangle$ and a set of fluents $E \subseteq F$, the projection of $\lts$ onto $E$ is the abstraction $\alpha_E: 2^F \rightarrow 2^E$ such that, for any $s \in 2^F$, $\alpha_E(s) = s \cap E$.

Projections and embeddings have in common that they both select a subset of fluents of $P$, in order to obtain a more concise representation of the state-space (directly and indirectly, respectively). However, they differ in that the projection of $P$ over a subset $E$ of fluents is unique (and can be computed given $E$), while various embeddings that range over the same set of fluents $E := \isom{F'}$ can exist (and thus, the embedded problem $P'$ must be provided). More precisely, projections have a canonical way of dealing with operators, while embeddings offer more leeway with regard to how they occur in the more concise representation. 

In the following, we suppose that there exists an embedding $(\isomsymb{}, \isompsymb{})$ of $P'$ in $P$, and we compare the state-space $\lts^{P'}$ of a problem $P'$ that embeds in $P$, to the abstracted state-space $\lts^{\alpha}$ that results from the projection of $P$ over $\isom{F'}$, that we simply denote $\alpha$. More precisely, we show that the former is more succinct and as informative as the latter, in the sense that a path between two states $s_1$ and $s_2$ exists in $\lts^\alpha$ iff a path between $s_1$ and $s_2$ also exists in $\lts^{P'}$. More specifically, the only transitions that one may find in the abstract state space $\lts^{\alpha}$, and not in the state space of the embedded problem $\lts^{P'}$, are reflexive
(and hence superfluous). 



\begin{lemma}
    \label{lem:abstraction_transition_preservation}
    Suppose that $(\isomsymb{}, \isompsymb{})$ is an embedding of $P'$ in $P$, and let $\lts^{P'}$ be the state-space of $P'$. Let $\alpha$ be the projection of $P$ over the set of fluents $\isom{F'}$ over which the embedding ranges, and let $\lts^{\alpha}$ be the abstract state space. Then
    \begin{itemize}
        \item There exists a bijection $b$ between the states of $\lts^{\alpha}$ and the states of $\lts^{P'}$;
        \item For any two different states $s^\alpha_1 \not= s^\alpha_2$ of $\lts^\alpha$, if there exists a transition from $s^\alpha_1$ to $s^\alpha_2$, then there exists a transition from $b(s^\alpha_1)$ to $b(s^\alpha_2)$ in $\lts^{P'}$
    \end{itemize}
     
\end{lemma}

For the sake of readability, we provide the proofs of Lemma~\ref{lem:abstraction_transition_preservation} and Proposition~\ref{prop:embedding_paths} in Appendix~\ref{appendix:embeddings_proofs}.
Note that reflexive transitions found in $\lts^\alpha$ are not guaranteed to have an equivalent in $\lts^{P'}$. This is due to the fact that the operators that have no tangible effects in $P'$ (i.e. the operators that do not satisfy (\ref{eq:active_condition})) are not carried over from the concrete state space $\lts$ to $\lts^{P'}$, while they are carried over to $\lts^\alpha$.
}

\changes{
    \begin{proposition}
        \label{prop:embedding_paths}
        Suppose that $(\isomsymb{}, \isompsymb{})$ is an embedding of $P'$ in $P$, and let $\lts^{P'}$ be the state-space of $P'$. Let $\alpha$ be the projection of $P$ over the set of fluents $\isom{F'}$ over which the embedding ranges, and let $\lts^{\alpha}$ be the abstract state space. Then
        if there exists a path from the initial state to a goal state in $\lts^\alpha$, there also exists a path from the initial state to a goal state in $\lts^{P'}$.
    \end{proposition}
}

\changes{As a conclusion, when compared to the projection that an embedding is associated to, the embedding has a state-space that contains fewer uninformative transitions (since they are reflexive), and that has an initial state that is at least as close to some goal state than its equivalent in the state-space of the projection. This is why, given an embedded problem and a projection, it is more likely to be easier to prove that the embedded problem is unsolvable than the projection.

We move on to the main result of this section.
}

\begin{proposition}
If $P'$ embeds in $P$ and $P'$ has no solution-plan, then neither does $P$.
\end{proposition}


\begin{proof}
Let $\isomsymb{}, \isompsymb{}$ be the functions defining the embedding.
We define $\isom{P'}$ in the natural way, as a copy of $P'$ in which each fluent $f \in F'$ is
replaced throughout by $\isom{f}$. Clearly, since $\isomsymb{}$ is injective, $P'$ and $\isom{P'}$ are isomorphic, and as $P'$ does not have a solution-plan, then neither does $\isom{P'}$.
Now suppose by contraposition that $P$ has a solution-plan $\plan{}$. Let $\tilde{O}$ be the set of operators $o \in O$ such
that $(\effp{o} \cup \effm{o}) \cap \isom{F'} \neq \emptyset$. Then deleting from $\plan{}$ all operators
not in $\tilde{O}$ leaves a sequence of operators $\tilde{\plan{}}$ such that $\isomp{\tilde{\plan{}}}$ is a solution-plan for $\isom{P'}$.
\end{proof}

\changes{\subsection{Relations to subinstance isomorphisms}}

As mentioned in Section~\ref{sec:strips_subiso_complexity}, 
the notion of embedding is a generalisation of the notion of subinstance isomorphism
in the sense that it allows us to detect more unsolvable 
instances than the contrapositive of Proposition~\ref{prop:stripsubiso_sol}.
In fact, as the following proposition shows, we obtain an embedding by considering the \emph{inverse} of the mapping between fluents in the subinstance isomorphism.

\begin{proposition} \label{prop:SSIimpliesSE}
Let $P$ and $P'$ be STRIPS planning instances with non-empty sets of fluents.
If $(\isomsymb,\isompsymb)$ is a subinstance isomorphism from $P$ to $P'$, then 
there is an embedding $(\tilde{\isomsymb},\isompsymb)$ of $P'$ in $P$.
\end{proposition}

\begin{proof}
It suffices to define $\tilde{\isomsymb}$ and show that the conditions of an embedding are satisfied.
Since $(\isomsymb,\isompsymb)$ is a subinstance isomorphism, $\isomsymb$ is injective, hence invertible
on its image $\isomsymb(F)$. Define $\tilde{\isomsymb}(q) = \isomsymb^{-1}(q)$ for $q \in \isom{F}$
and  $\tilde{\isomsymb}(q) = p_0$ otherwise, where $p_0 \in F$ is an arbitrary fluent.
Clearly $\tilde{\isomsymb}(F') = F$, so to show that  $(\tilde{\isomsymb},\isompsymb)$ is an embedding
$P'$ in $P$, it suffices to show:
\begin{enumerate}
\item $\forall o \in O$, $\tilde{\isomsymb}(\eff{\isomp{o}}) = \eff{o}$ and $\tilde{\isomsymb}(\pre{\isomp{o}}) \subseteq \pre{o}$,
\item $\tilde{\isomsymb}(G') \subseteq G$, and
\item $\tilde{\isomsymb}(I') \supseteq I$
\end{enumerate}
But, since $(\isomsymb,\isompsymb)$ is a subinstance isomorphism from $P$ to $P'$, we have
\begin{enumerate}
\item $\forall o \in O$, $\eff{\isomp{o}} = \isom{\eff{o}}$ and $\pre{\isomp{o}} = \isom{\pre{o}}$,
\item $G' = \isom{G}$, and
\item $I' = \isom{I}$
\end{enumerate}
Applying $\tilde{\isomsymb}$ to both sides of each equation, we obtain:
\begin{enumerate}
\item $\forall o \in O$, $\tilde{\isomsymb}(\eff{\isomp{o}}) = \tilde{\isomsymb}(\isom{\eff{o}}) = \eff{o}$ and 
$\tilde{\isomsymb}(\pre{\isomp{o}}) = \tilde{\isomsymb}(\isom{\pre{o}}) = \pre{o}$,
\item $\tilde{\isomsymb}(G') = \tilde{\isomsymb}(\isom{G}) = G$, and 
\item $\tilde{\isomsymb}(I') = \tilde{\isomsymb}(\isom{I}) = I$
\end{enumerate}
which completes the proof.
\end{proof}

Let us now give examples to highlight some differences between subinstance isomorphisms and embeddings. As a first example, there is an embedding between the two instances shown in Figure~\ref{fig:cube_embed} but no subinstance isomorphism since the number of preconditions and effects of each operator is strictly greater in the 4x4x4 cube compared to the 2x2x2 cube.
As a second example, consider the generic path-finding problem in a graph $\graph{}$ encoded as a STRIPS instance
$P_\graph{}$, described in Section~\ref{sec:graph_constructions}. 
If the initial state is $a$ and the goal state $b$, then a solution-plan is a path from vertex $a$ to vertex $b$. For simplicity, in the discussion that follows, we assume the same initial state and goal in each instance.
If $\mathcal{G}_1$ is a proper subgraph of $\mathcal{G}_2$, in the sense that $\mathcal{G}_2$ has extra edges, then there is a subinstance isomorphism from 
$P_{\mathcal{G}_1}$ to $P_{\mathcal{G}_2}$, but no embedding from 
$P_{\mathcal{G}_1}$ to $P_{\mathcal{G}_2}$
(since the operators corresponding to the extra edges in $\mathcal{G}_2$ have no image in $P_{\mathcal{G}_1}$). Proposition~\ref{prop:SSIimpliesSE} tells us that
there is an embedding in the opposite direction, from $P_{\mathcal{G}_2}$ to $P_{\mathcal{G}_1}$.
If, on the other hand, the set of vertices $V_1$ of $\mathcal{G}_1$ is a proper subset
of the set of vertices $V_2$ of $\mathcal{G}_2$ and that $\mathcal{G}_2$ restricted
to the vertices $V_1$ is a proper subgraph of $\mathcal{G}_1$, then $P_{\mathcal{G}_1}$
embeds in $P_{\mathcal{G}_2}$, but there is no subinstance isomorphism,
in either direction, between $P_{\mathcal{G}_1}$ and $P_{\mathcal{G}_2}$
(since there is either some operator in $P_{\mathcal{G}_1}$ that has no image in $P_{\mathcal{G}_2}$ or there is some fluent in $P_{\mathcal{G}_2}$ with no image in $P_{\mathcal{G}_1}$). 

\changes{\subsection{Complexity results}}

In the rest of this section, we prove the \NP{}-completeness of the problem concerned with finding whether an embedding of $P'$ in $P$ exists.

\begin{problem}{STRIPS Embedding \stripsembedding{}}
    \begin{problembody}
        \emph{\textbf{Input}}: & Two STRIPS instances $P$ and $P'$\\
        \emph{\textbf{Output}}: & An embedding $(\isomsymb{}, \isompsymb)$ of $P'$ in $P$, if one exists 
    \end{problembody}
\end{problem}

The proof is based on a reduction from the decision problem $k$-\textsf{Independent Set}, which we introduce below.

\begin{definition}[Independent Set]
Let $\graph{}(V, E)$ be a graph. An independent set is a set $V' \subseteq V$ such that, for any $v_i, v_j \in V'$, we have $\{v_i, v_j\} \not\in E$.
\end{definition}

\begin{problem}{$k$-\textsf{Independent Set}}
    \begin{problembody}
        \emph{\textbf{Input}}: & A graph $\graph{}(V, E)$\\
                               & An integer $k$ \\
        \emph{\textbf{Output}}: & \emph{Yes} \emph{iff} there exists an independent set of size $k$ in $\graph$ 
    \end{problembody}
\end{problem}

\begin{proposition}
    $k$-\textsf{Independent Set} is \NP{}-complete\cite{Garey1979computers}
\end{proposition}


\begin{proposition}
    The decision problem corresponding to \stripsembedding{} is \NP{}-complete
\end{proposition}

\begin{proof}
The problem is immediately in $\NP{}$. In the following, in order to prove the \NP{}-hardness of the problem, we build a reduction from $k$-\textsf{Independent Set}.

Let $\graph{}(V, E)$ be an instance of $k$-\textsf{Independent Set}, and let us denote $V = \{v_1, \ldots, v_{\vert V  \vert}\}$. We can assume, without loss of generality, that $\graph{}$ does not contain any reflexive edge (otherwise, we can remove the associated node, as they can not be part of any solution). Let us define the planning problem $P = \langle V, \varnothing, O, \varnothing \rangle$ such that 
\[ 
    O = \Set{\action{\varnothing}{\{v_i, v_j\}} | \{v_i, v_j\} \in E}
\]

In addition, we build $P' = \langle F', \varnothing, O', \varnothing \rangle$ such that $F' = \{ 1, \ldots, k \}$ and 
\[
    O' = \Set{\action{\varnothing}{\varnothing}} \cup \Set{\action{\varnothing}{\{i\}} | i \in \{1, \ldots, k\}}
\]

Let us show that there exists an embedding of $P'$ in $P$ iff $\graph{}$ is a positive $k$-\textsf{Independent Set} instance.

$(\Rightarrow)$ Suppose that there exists an embedding $\isomsymb{}: F' \rightarrow V$, $\isompsymb{}: O \rightarrow O'$, and let us show that $\isom{F'}$ is an independent set of $\graph{}$ of size $k$.

Firstly, as $\isomsymb$ is injective, we immediately have that $\isom{F'}$ is of size $k$. Secondly, let $v_i, v_j \in \isom{F'}$, with $v_i \not= v_j$. Let us show that $\{v_i, v_j\} \not\in E$.


Suppose by contradiction that $\{v_i, v_j\} \in E$. Then by construction of $O$, $o = \action{\varnothing}{\{v_i, v_j\}} \in O$. We immediately have that $o$ verifies condition $(\ref{eq:active_condition})$, as $\effp{o} = \{v_i, v_j\} \subseteq \isom{F'}$. As a consequence, $o$ also verifies condition $(\ref{eq:effp_morphism})$, i.e., 
\[
    \isom{\effp{\isomp{o}}} = \effp{o} \cap \isom{F'}
\]

On the right-hand-side, we have that $\effp{o} \cap \isom{F'} = \{ v_i, v_j \}$, which is a set of size exactly 2. However, on the left-hand-side, we have, by definition of the domain of $\isompsymb{}$, that $\isomp{o} = \action{\varnothing}{\varnothing}$ or $\isomp{o} = \action{\varnothing}{\{v_l\}}$ for some $v_l \in V$. In both cases, we have $\vert \effp{\isomp{o}} \vert \leq 1$, and thus $\vert \isom{\effp{\isomp{o}}} \vert \leq 1$, which is a contradiction. 

As a consequence, we have that $\{v_i, v_j\} \not \in E$, and $\isom{F'}$ is a $k$-independent set.

$(\Leftarrow)$ Suppose that $\graph{}$ is a positive instance of $k$-\textsf{Independent Set}. Then there exist $k$ different vertices of $V$, that form an independent set. Let us denote them $v_1, \ldots, v_k \in V$ without loss of generality.

Let us build $\isomsymb{}: F' \rightarrow V$ such that $\isom{i} = v_i$ for all $i \in \{1, \ldots, k\}$, so that $\isom{F'}$ is an independent set. $\isomsymb{}$ is immediately injective. In addition, we define $\isompsymb{}: O \rightarrow O'$ such that
\[
    \isomp{\action{\varnothing}{\{v_i, v_j\}}} = \left\{
                \begin{array}{ll}
                    \action{\varnothing}{\{l\}} \textit{ if } l = i \textit{ or } l = j \textit{ for some } l \in F'\\
                    \action{\varnothing}{\varnothing} \textit{ otherwise }
                \end{array}
                \right.
\]

Let us show that this forms a correct embedding. As $\pre{o} = \effm{o} = \varnothing$, and $\pre{o'} = \effm{o'} = \varnothing$ for any $o' \in O'$, conditions (\ref{eq:effm_morphism}) and (\ref{eq:pre_morphism}) are always satisfied. All that is left to show is that condition $(\ref{eq:effp_morphism})$ is satisfied when required, i.e., when $(\ref{eq:active_condition})$ is satisfied. 

Let $o = \action{\varnothing}{\{v_i, v_j\}} \in O$, and let us then suppose that $\effp{o} \cap \isom{F'} \not= \varnothing$.
%
We necessarily have $\effp{o} \cap \isom{F'} = \{ v_i \}$ for some $v_i$. Otherwise, we would have $\effp{o} \cap \isom{F'} = \{ v_i, v_j \}$, where $v_i \not= v_j$, but where $v_i, v_j \in \isom{F'}$. By construction of the elements of $O$, we would have $\{v_i, v_j \} \in E$, which is impossible because they are vertices of the independent set, by hypothesis.

As a consequence, $o = \action{\varnothing}{\{v_i, v_j\}}$, with $i \in \{1, \ldots, k\} = F'$ and $j \in \{k+1, \ldots, \vert V \vert\}$. We have the following:
\begin{align*}
    \isom{\effp{\isomp{o}}} &= \isom{\effp{\action{\varnothing}{\{i\}}}} \\
    &= \isom{\{i\}} \\
    &= \{ v_i \} \\
    &= \effp{o} \cap \isom{F'}
\end{align*}
Condition~$(\ref{eq:effp_morphism})$ is thus satisfied. As the initial state and goal are empty for both problems, conditions $(\ref{eq:goal_morphism})$ and  $(\ref{eq:init_morphism})$ are immediately satisfied. $\isomsymb{}$ and $\isompsymb{}$ therefore form a correct embedding of $P'$ into $P$.
\end{proof}

\section{Adapted algorithm for \stripsembedding{}}
\label{sec:se_algorithmic}

In this section, we show how to adapt Algorithm~\ref{alg:subiso_finder}, originally proposed for subinstance isomorphisms, to the case of \stripsembedding{}. More precisely, we propose adapted constraints for the preprocessing step, as well as an adequate propositional encoding for \stripsembedding{}. These steps being the only required modifications, the outline of the algorithm remains unchanged.


\subsection{Constraint Propagation}
\label{sec:se_pruning}

In this section, we show how constraint propagation allows us to prune a number of inconsistent associations, as previously.
For a start, we consider a variable for each operator $o \in O$, and a variable for each fluent $f' \in F'$, so that the respective domains of each of those variables contain the candidates for their image through $\isomsymb$ or $\isompsymb$.

In addition, to make the propagation of constraints more efficient, we introduce variables of the form $\fuseful{f}$, where $f \in F$ and $\domain{\fuseful{f}} \subseteq \{\top, \bot\}$. They translate the \emph{usefulness} of their associated fluents. A fluent $f$ is said to be useful in an embedding if it belongs to $\isom{F'}$. More generally, we will say that a fluent is \emph{useful} if it belongs to the image of \emph{every} possible embedding.
Intuitively, useful fluents are fluents that are likely to play a part in the enforcement of the properties imposed in Definition~\ref{def:embedding}.
This is why keeping in memory the domains of fluents of the form $\fuseful{f}$ allows us to know if $f$ is useful ($\domain{\fuseful{f}} = \{\top\}$), potentially useful ($\domain{\fuseful{f}} = \{\top, \bot\}$) or not useful ($\domain{\fuseful{f}} = \{\bot\}$).
Informally, when an operator of $O$ has a useful fluent in its effect (appearing positively or negatively) and thus satisfies (\ref{eq:active_condition}), we say that the operator is \emph{active}. It means that the equations (\ref{eq:effp_morphism}) to (\ref{eq:pre_morphism}) have to be satisfied for this operator. This notion, however, is mostly useful in the propositional encoding, and we thus elaborate further on this point in the appropriate section.

\subsubsection{Initial domains}


First, as with subinstance isomorphisms, the initialization of the domains of fluents is straightforward: each fluent $f' \in F'$ has an initial domain set to either \changes{$I \setminus G$}, \changes{$G \setminus I$}, $I \cap G$ or $F \setminus (I \cup G)$, depending 
on whether $f'$ is in $I'$ \changes{only}, $G'$ \changes{only}, both or neither, respectively.

We also use operator profiles to initialize the domains of operators. As previously, an operator profile is a vector of $\mathbb{N}^4$ that bears the information of the size of the precondition, the (positive and negative) effect of the operator, as well as the number of fluents that are strict-delete. 
However, in our current case, an operator $o \in O$ that has a different profile from $o' \in O'$ can still be mapped to $o'$, depending on the set of fluents that constitute the image of $\isomsymb{}$.

We will only consider the domains of active operators, in the sense that operators that are not active do not have to fulfill any condition.
So, let us suppose that $o$ is active, and that we wish to map $o$ to $o'$ (i.e., $\isomp{o} = o'$). Then it has to satisfy the condition:
\[
\underbrace{\isom{\pre{\isomp{o}}}}_{\vert \cdot \vert = \vert \pre{o'} \vert} \subseteq \underbrace{\pre{o} \cap \isom{F'}}_{\vert \cdot \vert \leq \vert \pre{o} \vert}
\]

The equality on the left-hand side uses the fact that $\isomsymb{}$ is injective. So as a consequence, we have that necessarily, $\vert \pre{o'} \vert \leq \vert \pre{o} \vert$. A similar case can be made for effects, which allows us to define an operator profile $\opprofile{o}$ for each operator $o$ of either problem.
As a consequence, an operator $o$ can only be mapped to $o'$ if $\opprofile{o'} \leq \opprofile{o}$, where $\leq$ compares each element of each tuple to its counterpart in the other tuple.


\subsubsection{Propagation}

        
        
        
    

In the following, we denote $\mathcal{C}_{\effsymb} = \{ \effpsymb{}, \effmsymb{} \}$. The constraint below indicates that fluent $f'$ can only be mapped to a fluent $f$ 
if all operators $o$ where $f$ appears in $\effsymb{}$ 
can be matched with some $o' \in O'$. We also check that $\top \in \domain{\fuseful{f}}$, because if it is not the case, then mapping $f'$ to $f$ leads to a contradiction.

\begin{equation}
    \domain{f'} \subseteq \Set{f | 
    \begin{array}{l}
        \top \in \domain{\fuseful{f}} \hspace{2ex} \wedge \\
        \forall o \in O, \; \forall \absrel{} \in \mathcal{C}_{\effsymb{}} \textit{ s.t. } f \in \absrel{}(o), \; \exists o' \in \domain{o} \textit{ s.t. } f' \in \absrel{}(o')
        \end{array}}
    \label{eq:embedding_fluent_cp}
\end{equation}


The constraint below only applies for operators that are active. We however maintain the domain of every operator as if it were active, as an inactive operator can be mapped to any other operator \-- in fact, our algorithm does not bother giving an image to inactive operators, as any operator would suffice.
Note that, by enforcing a constraint for active operators only, we do not interfere incorrectly with the other constraints. Indeed, (\ref{eq:embedding_fluent_cp}) and (\ref{eq:embedding_fuseful_cp}) are only concerned with operators that are active, and as a consequence, we are not at risk of pruning an otherwise possibly valid association.
The first line below indicates that mapping $o$ to $o'$ can only be done if all fluents in $\eff{o}$ that are necessarily useful can be matched by some fluent of $\eff{o'} \subseteq F'$. 
The second indicates that all fluents in $\eff{o'}$ (resp. $\pre{o'}$) 
must map to a fluent in $\eff{o}$ (resp. $\pre{o}$).
%
%
\begin{equation}
\begin{array}{ll}
    \domain{o} \subseteq &\Set{o' | \forall \absrel{} \in \mathcal{C}_{\effsymb}, \forall f \in \absrel{}(o) \textit{ s.t. } \domain{\fuseful{f}} = \{\top \}, \exists f' \in \absrel{}(o') \textit{ s.t. } f \in \domain{f'}} \cap\\  
     &\Set{o' | 
     \forall \absrel{} \in \mathcal{C}, \forall f' \in \absrel{}(o'), \exists f \in \domain{f'} \textit{ s.t. } \top \in \domain{\fuseful{f}} \ \land \ f \in \absrel{}(o)}
\end{array}
      \label{eq:embedding_operator_cp}
\end{equation}



%

Despite the constraint only being applicable for active operators, it still allows us to simplify a number of clauses. Clauses found in (\ref{eq:operators_images}), (\ref{eq:morph_effect_inclusion}) and (\ref{eq:morph_effect_inclusion_reverse})
(in the propositional encoding given in  Section~\ref{sec:encSE}), which are the only formulas involving variables that model associations between operators, all have $\oactive{o}$ for hypothesis, which is a variable that represents that $o$ is active. In other words, the clauses that are simplified as a consequence of (\ref{eq:embedding_operator_cp}) are actually clauses that would have been satisfied nonetheless if $o$ weren't active. All that is left to do is ensuring that some fluents can be detected as useful, as otherwise, the above constraint would be powerless. This is why the rest of this section, as well as the next, is dedicated to methods for finding useful fluents, both through constraint propagation techniques and through a more ad-hoc criterion. 

Some useful fluents can be detected through the following constraint,  which allows us to revise partially the usefulness of a fluent. Its role is mainly to prune $\top$ of the domain of the fluent, and thus marking a fluent as not useful. Conversely, in some rare cases, it can also detect that a fluent is useful, but detecting fluent usefulness is mostly left to a criterion that we present in the next section.

The second line, in the first set, translates the fact that if $f$ were to make $o$ active, then $o$ has to have a possible image in which $f$ finds a match.

\begin{align}
\label{eq:embedding_fuseful_cp}
    \domain{\fuseful{f}} \subseteq 
    &\Set{ \top | \begin{array}{l}
         \exists f' \in F' \textit{ s.t. } f \in \domain{f'} \; \wedge  \\
         \forall o \in O, \forall \absrel{} \in \mathcal{C}_{\effsymb} \textit{ s.t. } f \in \absrel{}(o),  \exists o' \in \domain{o}, \exists f' \in \absrel{}(o'), f \in \domain{f'} 
        \end{array} } \cup \\ 
    &\Set{ \bot | \not \exists f' \in F' \textit{ s.t. } \domain{f'} = \{f\}} \nonumber
\end{align}





\subsubsection{Fluent usefulness}
\label{sec:fluent_usefulness}

As shown in (\ref{eq:embedding_operator_cp}), detecting that a fluent is useful is crucial for operators to be pruned efficiently. However, (\ref{eq:embedding_fuseful_cp}) can only detect fluent usefulness in the very specific case where a fluent's domain is a singleton. It is reasonable to assume that this will not occur often. We thus propose an additional method to detect useful fluents more efficiently, which is based on the following simple lemma that generalizes the above criterion:

\begin{lemma} \label{lem:3}
Let $U \subseteq F$ and $f'_1, \ldots, f'_n \in F'$ be such that $\domain{f'_1} = \ldots = \domain{f'_n} = U$. 
\begin{itemize}
    \item If $\vert U \vert = n$, then every $f \in U$ is useful.
    \item If $\vert U \vert < n$, then no embedding exists.
\end{itemize}

\end{lemma}

\begin{proof}
In the first case, if $n$ fluents of $U \subseteq F$ have to be the images of $n$ fluents of $F'$ through an injective mapping $\isomsymb$, then $\isomsymb{}$ is surjective on $U$, if such a mapping exists.

In the second case, the restriction $\isomsymb': \{f'_1, \ldots, f'_n\} \rightarrow U$ of $\isomsymb$ should be injective, but can not be, as $n > \vert U \vert$. Thus, such a mapping does not exist.
\end{proof}

A stronger version of the lemma could be obtained by replacing the condition on the equality of the domains by the condition $\domain{f'_i} \subseteq U $ for all $i \in \{1, \ldots, m\}$. However, finding sets $U$ of fluents for which the number of domains that are subsets of $U$ is greater than $\vert U \vert$ is a computationally costly problem. This justifies our choice to use our simpler version of the lemma, along with the fact that our experiments tend to indicate that it is sufficient in practice. This is not surprising: as our STRIPS representations come from planning instances encoded in PDDL, fluents and operators follow structures similar to one another. In particular, in typed PDDL instances, STRIPS fluents that result from PDDL objects of the same type often have identical domains.

Algorithmically, this criterion is used during the execution of AC3. More precisely, we add a token in the queue of variables to revise, that indicates that the above procedure (based on Lemma~\ref{lem:3}) should be called. This token is inserted between the fluent variables, that are revised first, and the operator variables, that are revised next. This ensure\changes{s} that as many useful fluents as possible are detected before any operator is revised, hence maximising the effectiveness of the revisions of operators.

\subsection{Propositional encoding for \stripsembedding{}} \label{sec:encSE}

We now build the formula $\varphi$ such that models of $\varphi$ correspond to embeddings as defined previously (Definition~\ref{def:embedding}). 
The variables of $\varphi$ are:
\begin{align}
    \dom{\varphi} = &\Set{ \fassoc{i}{j} | i \in F, j \in F' } \cup \label{eq:dom_fassoc} \\
    &\Set{ \oassoc{r}{s} | r \in O, s \in O' } \cup \label{eq:dom_oassoc} \\
    &\Set{ \fuseful{i} | i \in F} \cup \label{eq:dom_fuseful} \\
    &\Set{ \oactive{r} | r \in O} \label{eq:dom_oactive}
\end{align}

The propositional variable $\fassoc{i}{j}$ represents that fluent $j \in F'$ is mapped to $i \in F$. Similarly, $\oassoc{r}{s}$ represents that operator $r \in O$ is mapped to $s \in O'$. For the sake of clarity, despite the homomorphisms $\isomsymb{}$ and $\isompsymb{}$ being asymmetric, we use a consistent notation for fluents and operators. Indeed, even though we have $\isomsymb{}: F' \rightarrow F$ and $\isompsymb{}: O \rightarrow O'$, subscripts of our propositional variables represent elements (fluents or operators) of $P$, while superscripts always represent elements of $P'$. This allows us to remain consistent with the notation used in Section~\ref{sec:ssi_algorithmic}.

The notion of \emph{active} operator, introduced earlier
(Section~\ref{sec:se_pruning}), also allows us to make our encoding more succinct. Propositional variable $\oactive{r}$ represents that $r \in O$ is active, which means that there is a useful fluent in its effect. Only active operators have to verify the morphism properties of Definition~\ref{def:embedding}. For similar reasons, we use variables of the form $\fuseful{i}$ for each fluent $i \in F$, which represent the fact that $i$ is \emph{useful}, as introduced in Section~\ref{sec:se_pruning}.

We now show how to build the propositional encoding in itself. The formula $\varphi$ consists in the conjunction of the sets of formulas presented below.

The following formula enforces that $\isomsymb{}$ is correctly defined, and has an image for each element of its domain. We also add a similar formula where $\fassoc{i}{j}$ is replaced by $\oassoc{i}{j}$, in order to ensure an image for each element of the domain of $\isompsymb{}$, adapting the domain and codomain as needed. An additional adaptation that we perform is that we check if the operator is active: if it is not, its image can be chosen to be an arbitrary operator, and indeed there is no need to find the image itself.

\begin{align}
    \label{eq:fluents_image_embed}
    \bigwedge_{j \in F'}
    \left(
        \bigvee_{i \in \domain{j}} \fassoc{i}{j} \hspace{1ex} \wedge 
        \bigwedge_{\substack{i_1, i_2 \in \domain{j} \\ i_1 \not= i_2}} (\neg \fassoc{i_1}{j} \vee \neg \fassoc{i_2}{j})
    \right)
\end{align}

\begin{align}
    \label{eq:operators_images}
    \bigwedge_{r \in O}
    \left(
        \left(
        \oactive{r} \rightarrow 
        \bigvee_{s \in \domain{r}} \oassoc{r}{s} \hspace{0.5ex}
        \right)
        \wedge 
        \bigwedge_{\substack{s, t \in \domain{r} \\ s \not= t}} (\neg \oassoc{r}{s} \vee \neg \oassoc{r}{t})
    \right)
\end{align}

In addition, we enforce the injectivity of $\isomsymb{}$ with the following formula:

\begin{equation}
    \label{eq:fluents_injectivity}
    \bigwedge_{i \in F} \bigwedge_{\substack{j, k \in F' \\ j \not= k}} \neg \fassoc{i}{j} \vee \neg \fassoc{i}{k}
\end{equation}

Let $\absrel{} \in \{ \effpsymb{}, \effmsymb{} \}$ and recall that $\isomsymb{}: F' \rightarrow F$. The following formula enforces that $\isom{\absrel(\isomp{o})} \subseteq \absrel{}(o) \cap \isom{F'}$ for operators $o$ that are active.
It means that, if operator $r \in O$ is active and mapped to $s \in O'$ (\emph{i.e.}, if $\isomp{r} = s$), then for each fluent $j$ in $\absrel{}(s) = \absrel{}(\isomp{r})$, there must exist a fluent $i \in \absrel{}(r) \subseteq F$ such that $\isom{j} = i$ (and thus $i$ is useful).

\begin{equation}
    \bigwedge_{r \in O} \left( \oactive{r} \rightarrow \bigwedge_{s \in O'} \left( \oassoc{r}{s} \rightarrow \bigwedge_{j \in \absrel{}(s)} \left( \bigvee_{i \in \absrel{}(r)} \fassoc{i}{j} \right) \right) \right) \label{eq:morph_effect_inclusion}
\end{equation}

We also apply the above formula with $\absrel{} = \presymb{}$,
which immediately enforces (\ref{eq:pre_morphism}).

Similarly, the following formula enforces that $\isom{\absrel(\isomp{o})} \supseteq \absrel{}(o) \cap \isom{F'}$. It starts by considering an operator $r \in O$ that is active (\emph{i.e.}, that must satisfy conditions (\ref{eq:effp_morphism}) and (\ref{eq:effm_morphism})), and then all fluents $i \in F$ that are useful and in $\absrel{}(r)$ (i.e. 
fluents $i$ that are necessarily in $\absrel{}(r) \cap \isom{F'}$). 
We then find the operator $s \in O'$ that is such that $\isomp{r} = s$ (when $\oassoc{r}{s}$ is not verified, the implication is true). In $\absrel{}(s) = \absrel{}(\isomp{r})$, we must then find some fluent $j$ that is such that $\isom{j} = i$. Hence the formula.

\begin{equation}
    \bigwedge_{r \in O} \left( \oactive{r} \rightarrow \bigwedge_{i \in \absrel{}(r)} \left( \fuseful{i} \rightarrow \bigwedge_{s \in O'} \left( \oassoc{r}{s} \rightarrow \bigvee_{j \in \absrel{}(s)} \fassoc{i}{j} \right) \right) \right) \label{eq:morph_effect_inclusion_reverse}
\end{equation}


Finally, the following formula enforces (\ref{eq:init_morphism}). It considers all fluents $i$ that are useful and in $I$, and requires that they are associated to some fluent in $I'$.
\begin{equation}
    \bigwedge_{i \in I} \left( \fuseful{i} \rightarrow \bigvee_{j \in I'} \fassoc{i}{j} \right)
\end{equation}

Propositional variables of the form $\fuseful{i}$ and $\oactive{o}$ do not immediately represent objects of $P$ or $P'$, but rather properties about the homomorphism that we are building. We express them in the following formulas.

A fluent $i \in F$ is useful if it belongs to $\isom{F'}$. As a consequence, we have, for every $i \in F$:
\begin{equation*}
    \fuseful{i} \leftrightarrow \bigvee_{j \in F'} \fassoc{i}{j}
\end{equation*}
which is, in CNF, 
\begin{equation}
    \left( \neg \fuseful{i} \vee \bigvee_{j \in F'} \fassoc{i}{j} \right) \hspace{2ex} \wedge \hspace{2ex}
    \bigwedge_{j \in F'} \left( \fuseful{i} \vee \neg \fassoc{i}{j} \right)
\end{equation}
%
%
%
In addition, an operator $o \in O$ is active if it satisfies condition (\ref{eq:active_condition}). In the following formula, we write $f \in \textsf{var}(\eff{o})$ to denote the set of fluents $f \in F$ such that either $f \in \eff{o}$ or $\neg f \in \eff{o}$. We thus have:
%
%
\begin{equation*}
    \bigwedge_{o \in O} \ \bigwedge_{i \in \textsf{var}(\eff{o})} \fuseful{i} \rightarrow \oactive{o}
\end{equation*}

Note that this does not perfectly translate our definition of active operator: we do \emph{not} have that $o \in O$ is active \emph{iff} it satisfies (\ref{eq:active_condition}), but only that, if (\ref{eq:active_condition}) is satisfied, then necessarily, it has to be marked active. It suffices for our encoding, as no problem arises when an operator $o$ is marked as active during solving, even though it's not: an inactive operator can be mapped to any other operator. In particular, the set of models of the formula remains unchanged, compared to the case where we enforce strictly the activeness of operators.


\section{Experimental evaluation}
\label{sec:experimental_evaluation}

We implemented Algorithm~\ref{alg:subiso_finder} in Python 3.10, and used it to solve $\stripssubiso{}$, $\stripssubisogeneral{}$ and $\stripsembedding{}$ instances.
We used the parser and translator of the Fast Downward planning system~\cite{helmert2006fast}.
The SAT solver we used was Maple LCM~\cite{luo2017effective}, winner of the main track at SAT 2017. Experiments were run on a machine running Rocky Linux 8.5, powered by an Intel Xeon E5-2667 v3 processor, using at most 8GB of RAM, 4 threads, and 300 seconds per test.
The code, sets of benchmarks, as well as our full results are available online
\footnote{\url{https://github.com/arnaudlequen/PDDLIsomorphismFinder}}.
%

Our set of benchmarks is formed out of ten sets found in previous International Planning Competitions, namely Barman, Blocks, Ferry, Gripper, Hanoi, Rovers, Satellite, Sokoban, TSP and Visitall. For each of these domains, we created what we call \emph{STRIPS matching instances}, which are pairs of instances of the same domain. We did this for each possible pair of planning instances of each considered domain. Consequently, each domain of our set of benchmarks has a number of STRIPS matching instances that is quadratic in the number of planning instances in the associated IPC domain. We did not generate any new instance.

A STRIPS matching instance is a valid input for problems \stripssubiso{}, \stripssubisogeneral{} and \stripsembedding{} simultaneously. We thus evaluated our algorithm adapted for all three problems on the same set of benchmarks. We only report domains for which at least one instance can be solved. In particular, as no instance of Rovers was solved for any of the problems we considered, the domain does not appear in the results.


The goal of the experiments is twofold. First, the aim is to demonstrate that, despite the theoretical hardness of the problems, it is possible to find (homogeneous) subinstance isomorphisms and embeddings in reasonable time for problems of non-trivial size \-- or at least prove the absence of such mappings. Second, the goal is to show the efficiency of the pruning techniques presented in Section~\ref{sec:ssi_pruning} and Section~\ref{sec:se_pruning}, that is to say, to prove that the additional cost of the preprocessing is outbalanced by the speed-up it provides to the search phase. 

\subsection{Absolute coverage}

\begin{table}
    \centering
    \caption{Number of instances of \stripssubisogeneral{}, \stripssubiso{} and \stripsembedding{} on which the implementation of our algorithm terminates within 300 seconds. For each problem, the first pair of columns shows the number of STRIPS matching instances solved with and without the constraint propagation-based preprocessing, respectively. The last column shows the average percentage of clauses that were eliminated from the propositional encoding, thanks to the pruning step. When the pruning step allowed us to conclude immediately (thus skipping the encoding into SAT altogether), we consider that all clauses have been simplified. As no \stripssubisogeneral{} or \stripssubiso{} instance of Barman could be decided, the average number of simplified clauses is not applicable.
    }
    \begin{tabular}{l||rr|r||rr|r||rr|r}
    	\toprule
    	\multirow{2}{*}{Domain} & \multicolumn{3}{c||}{\stripssubisogeneral{}} & \multicolumn{3}{c||}{\stripssubiso{}} & \multicolumn{3}{c}{\stripsembedding{}} \\
    	 & CP & NoCP & Av. Simp. & CP & NoCP & Av. Simp. & CP & NoCP & Av. Simp. \\
    	\midrule 
    	\textbf{barman} & 0 & 0 & N/A\% & 0 & 0 & N/A\% & 19 & 0 & 100.0\% \\ 
    	\textbf{blocks} & 125 & 61 & 81.2\% & 118 & 57 & 81.3\% & 261 & 136 & 80.9\% \\ 
    	\textbf{ferry} & 100 & 0 & 65.9\% & 134 & 0 & 66.1\% & 199 & 161 & 90.0\% \\ 
    	\textbf{gripper} & 196 & 134 & 70.1\% & 62 & 60 & 69.5\% & 210 & 73 & 89.0\% \\ 
    	\textbf{hanoi} & 48 & 43 & 41.5\% & 52 & 47 & 41.5\% & 153 & 109 & 84.6\% \\ 
    	\textbf{satellite} & 26 & 16 & 87.1\% & 30 & 16 & 93.6\% & 167 & 41 & 97.2\% \\ 
    	\textbf{sokoban} & 189 & 0 & 100.0\% & 189 & 0 & 100.0\% & 73 & 0 & 100.0\% \\ 
    	\textbf{tsp} & 332 & 324 & 34.3\% & 263 & 255 & 38.5\% & 444 & 252 & 95.1\% \\ 
    	\bottomrule
    \end{tabular}
    \label{tab:eval_cp}
\end{table}
The coverage of our implementation on our set of benchmarks is shown in Table~\ref{tab:eval_cp} for $\stripssubiso{}$, $\stripssubisogeneral{}$, and $\stripsembedding{}$. The table shows the absolute and relative numbers of instances of each problem on which our implementation terminates within the time and memory cutoffs. 

The first point we notice is that problems $\stripssubisogeneral{}$ and $\stripssubiso{}$ are often closely comparable in terms of hardness, except for some particular domains. These include domains TSP and Gripper, for which 25\% and 216\% more instances are solved when imposing no condition on the initial state and goal. 
For both domains, we believe this to be due to the additional constraints in \stripssubiso{}, that turn all instances of these domains into negative ones. Not only are these negative instances harder for the SAT solver to decide, but also the pruning step is of little to no help. This is a consequence of the abundance of symmetries in these domains, on which we elaborate in the next section.

Finding embeddings, however, is significantly different from finding subinstance isomorphisms. In our entire set of benchmarks, our algorithm could not identify a single positive \stripsembedding{} instance, except for STRIPS matching instances that were built upon two identical STRIPS instances. This leads us to believe that embeddings between real-world planning instances are rare, which does not necessarily render the notion of embedding impractical. Indeed, a small embeddable unsolvable instance could possibly be  synthesized out of a larger unsolvable instance, and serve as a certificate of unsolvability, for example. 
We can also note that the IPC planning instances have been crafted to be
solvable, which precludes the application of fast detection of
unsolvability via a small embeddable unsolvable instance.


In Table~\ref{tab:absolute_size}, we present a few results on the absolute sizes of the problems that we solved during our experiments, within the time and memory limits.
For a STRIPS planning problem $P = \langle F, I, O, G \rangle$, we denote $\vert P \vert = \vert F \vert + \vert O \vert$. As a STRIPS matching instance has two main dimensions, represented by the respective sizes of the planning instances that constitute it, we present two different ways of measuring the size of a matching     instance.
In the first set of columns of Table~\ref{tab:absolute_size}, the sum of both planning instances is considered, and we report the size of the STRIPS matching instance that maximizes that sum. With this metric, the larger instance is often disproportionately bigger than the smaller instance. This imbalance can be explained by the fact that the encoding into a propositional formula is of time and size $\mathcal{O}(\vert O \vert \cdot \vert O' \vert \cdot \vert F \vert \cdot \vert F' \vert)$, for all three problems.
This kind of situation often represents the limit of what could be dealt with with our program, when trying to extract the solutions of a small problem out of a larger database for the domain (in the case of \stripssubiso{})
or detecting unsolvability of a large instance by embedding a small instance which is known to be unsolvable (in the case of \stripsembedding{}). 

In the second set of columns, to measure the size of an \stripssubiso{} or \stripssubisogeneral{} instance, we consider the lexicographic order on pairs $(\vert P \vert, \vert P' \vert)$, and we report the biggest problem with respect to that metric.
Likewise, we consider the lexicographic order on pairs $(\vert P' \vert, \vert P \vert)$ for \stripsembedding{}.
This gives an order of magnitude of the size of the biggest problem that can be identified as a subinstance of another one, under ideal conditions on this other instance. 


\begin{figure}[h!]
    \centering
    \subfloat[\centering \stripssubiso{} instances with positive outcome]
    {\includegraphics[width=0.42\textwidth]{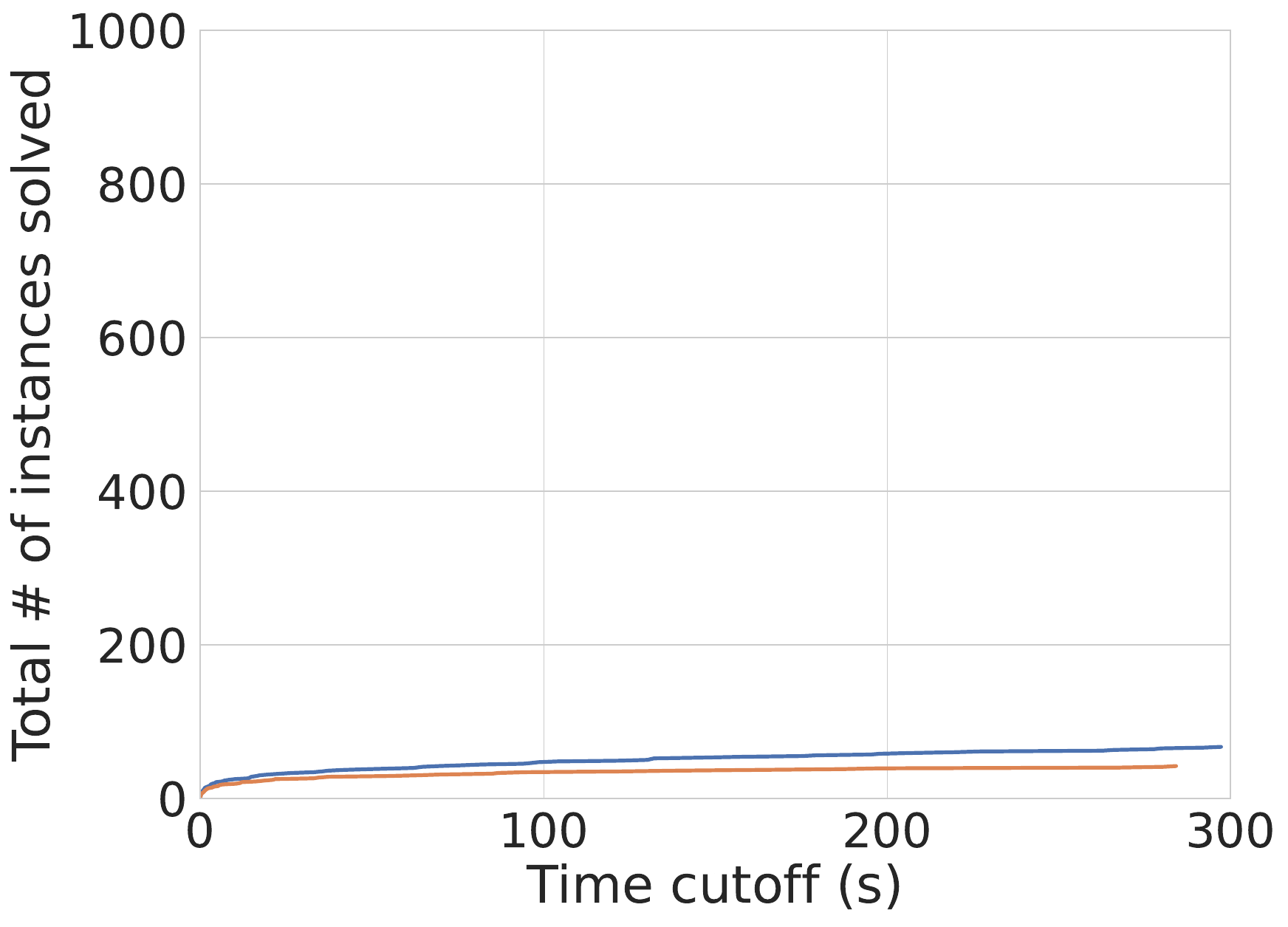}\label{fig:solved_vs_cutoff:ssi_positive}}%
    \qquad
    \subfloat[\centering \stripssubiso{} instances with negative outcome]
    {\includegraphics[width=0.42\textwidth]{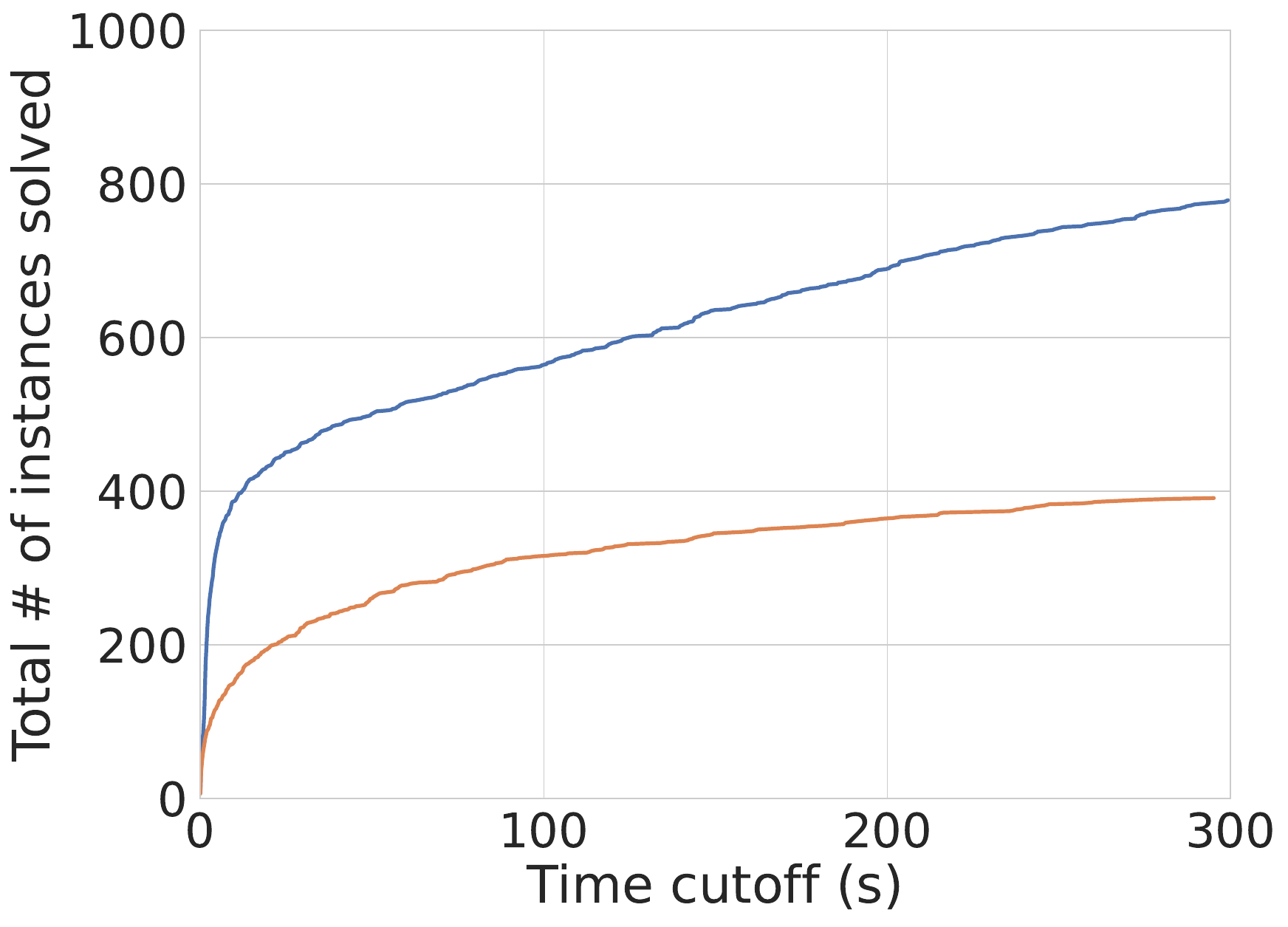}\label{fig:solved_vs_cutoff:ssi_negative}}%
    
    \vspace{3ex}
    
    \subfloat[\centering \stripssubisogeneral{} instances with positive outcome]
    {\includegraphics[width=0.42\textwidth]{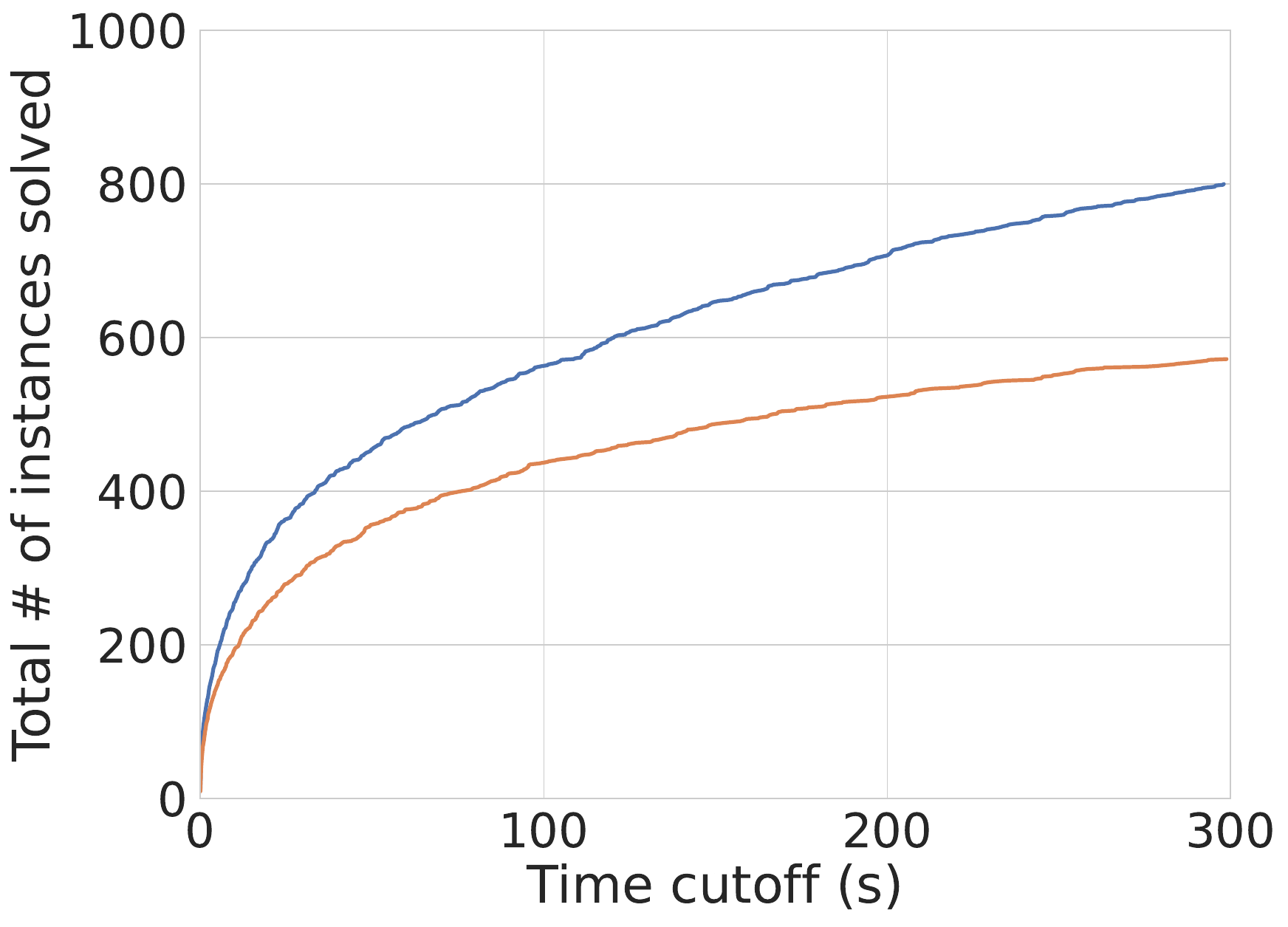}\label{fig:solved_vs_cutoff:ssig_positive}}%
    \qquad
    \subfloat[\centering \stripssubisogeneral{} instances with negative outcome]
    {\includegraphics[width=0.42\textwidth]{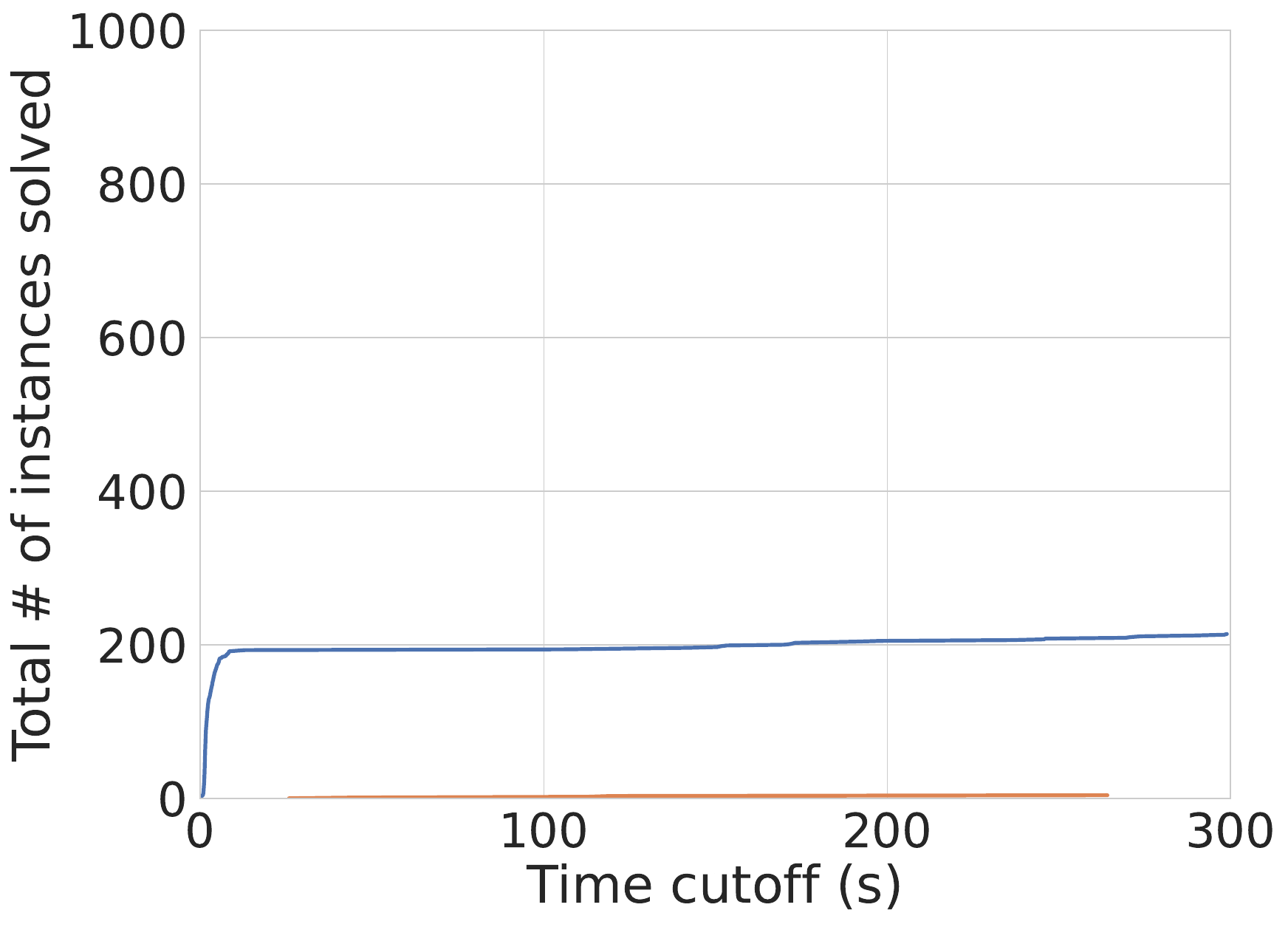}\label{    fig:solved_vs_cutoff:ssig_negative}}%
    
    \vspace{3ex}
    
    \subfloat[\centering \stripsembedding{} instances with positive outcome]
    {\includegraphics[width=0.42\textwidth]{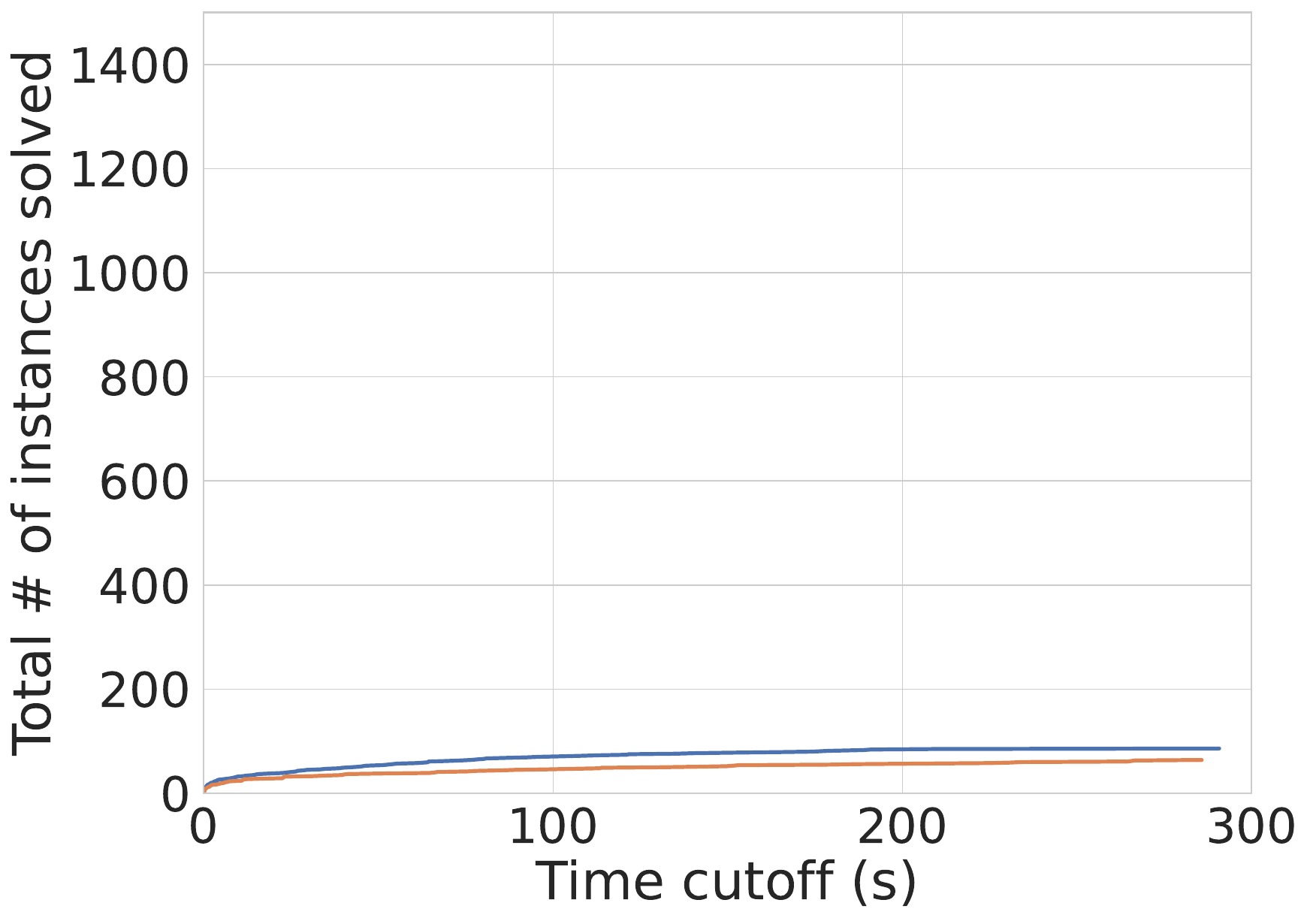}\label{fig:solved_vs_cutoff:se_positive}}%
    \qquad
    \subfloat[\centering \stripsembedding{} instances with negative outcome]
    {\includegraphics[width=0.42\textwidth]{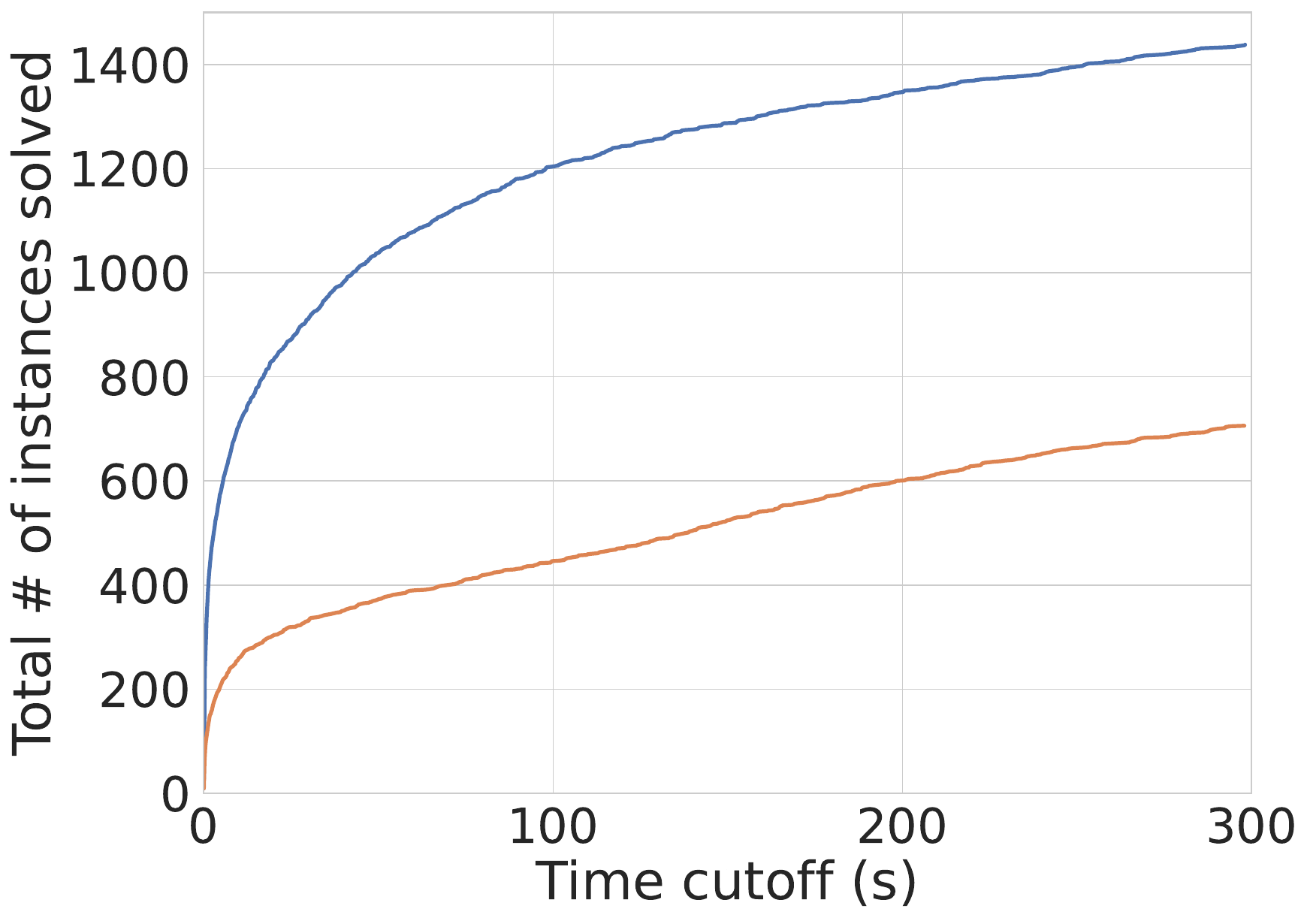}\label{    fig:solved_vs_cutoff:se_negative}}%
    
    \caption{Number of \stripssubiso{} (top), \stripssubisogeneral{} (middle) and \stripsembedding{} (bottom) instances that can be solved by our implementation, as a function of the time cutoff. Blue/orange (upper/lower) curves correspond respectively to with/without pruning (constraint propagation preprocessing). \changes{The left column reports positive instances (for which a homomorphism could be found), while the right column reports negative instances (for which no homomorphism exists).}}    
    \label{fig:solved_vs_cutoff}
\end{figure}

    

\begin{table}
    \centering
    \caption{Sizes of the biggest instances that can be solved by our implementation within the time and memory limits, for $\stripssubisogeneral{}$, $\stripssubiso{}$ and \stripsembedding{}, as well as statistics on the set of benchmarks. For each problem, in the first set of columns, we consider the sum of the sizes of the planning instances that constitute the STRIPS matching instance. In the second set, we consider the size of the smallest planning instance among the pair that constitutes the instance. The last column of the second table gives statistics for the whole domain, by reporting the average number of operators per fluent.}
    \begin{tabular}{l||rrr|rr||rrr|rr}
    	\toprule
    	\multirow{3}{*}{Domain} & \multicolumn{5}{c||}{\stripssubisogeneral{}} & \multicolumn{5}{c}{\stripssubiso{}} \\
    	 & \multicolumn{3}{c|}{Maximum sum} & \multicolumn{2}{c||}{Max $\vert P \vert$} & \multicolumn{3}{c|}{Maximum sum} & \multicolumn{2}{c}{Max $\vert P \vert$} \\	
    	       & $\vert P \vert$ & $\vert P' \vert$ & Sum
    	       & $\vert P \vert$ & $\vert P' \vert$
    	       & $\vert P \vert$ & $\vert P' \vert$ & Sum
    	       & $\vert P \vert$ & $\vert P' \vert$ \\
    	\midrule
		\textbf{blocks} & 73 & 3265 & 3338 & 505 & 586 & 73 & 3265 & 3338 & 505 & 505 \\
    	\textbf{ferry} & 434 & 730 & 1164 & 542 & 562 & 442 & 730 & 1172 & 562 & 562 \\
    	\textbf{gripper} & 464 & 604 & 1068 & 492 & 492 & 576 & 576 & 1152 & 576 & 576 \\
    	\textbf{hanoi} & 23 & 2853 & 2876 & 358 & 358 & 23 & 4887 & 4910 & 358 & 358 \\
    	\textbf{satellite} & 184 & 1845 & 2029 & 644 & 644 & 103 & 2187 & 2290 & 644 & 696 \\
    	\textbf{sokoban} & 2617 & 2703 & 5320 & 2617 & 2703 & 2617 & 2703 & 5320 & 2617 & 2703 \\
    	\textbf{tsp} & 108 & 990 & 1098 & 378 & 418 & 108 & 990 & 1098 & 418 & 418 \\
    	\bottomrule
    \end{tabular}
    
    \vspace{2ex}
    
    \begin{tabular}{l||rrr|rr||r}
    	\toprule
    	\multirow{3}{*}{Domain} & \multicolumn{5}{c||}{\stripsembedding{}} & \\
    	 & \multicolumn{3}{c|}{Maximum sum} & \multicolumn{2}{c||}{Max $\vert P' \vert$} & \multicolumn{1}{c}{Average}\\	
    	       & $\vert P \vert$ & $\vert P' \vert$ & Sum
    	       & $\vert P \vert$ & $\vert P' \vert$ & Op/Fluent\\
    	\midrule
    	\textbf{barman} & 3184 & 3184  & 6368 & 3184 & 3184 & 5.87 \\
    	\textbf{blocks} & 5290 & 73  & 5363 & 766 & 766 & 1.25 \\
    	\textbf{ferry} & 730 & 730  & 1460 & 730 & 730 & 1.09 \\
    	\textbf{gripper} & 604 & 604  & 1208 & 604 & 604 & 1.25 \\
    	\textbf{hanoi} & 13682 & 60  & 13742 & 625 & 625 & 8.04 \\
    	\textbf{satellite} & 646540 & 696  & 647236 & 15712 & 10895 & 84.77 \\
    	\textbf{sokoban} & 2703 & 1881  & 4584 & 2703 & 1881 & 0.72 \\
    	\textbf{tsp} & 990 & 378  & 1368 & 504 & 504 & 6.78 \\
    	\bottomrule
    \end{tabular}
    \label{tab:absolute_size}
\end{table}

\subsection{Impact of the preprocessing}

Our first observation is that the preprocessing step almost never holds back the algorithm: almost all instances of our test sets that can be solved without preprocessing are also solved when the preprocessing step is performed.
\changes{In all of the sets of benchmarks for which we report the results here, no significant slowdown of our algorithm was incurred by the preprocessing.}

Furthermore, in many sets of benchmarks and for all three problems, the preprocessing greatly improves the overall performance of our implementation, so much so that some previously infeasible domains are now within the range of our algorithm. Such extreme cases include Sokoban, for which our algorithm is powerless without the pruning step: all 189 (resp. 73) instances solved by our implementation are outside the range of the preprocessing-less version of the algorithm, for \stripssubisogeneral{} and \stripssubiso{} (resp. \stripsembedding{}). In most cases, however, we observe a significant increase in the coverage of the algorithm, which remains nonetheless within the same order of magnitude. For example, for \changes{the Satellite domain} in the case of \stripssubiso{}, 30 instances are solved when constraint propagation is enabled, whereas only 16 can be settled without it.

More specifically, in almost every case, the preprocessing step leads to a reduction of the size of the propositional encoding. This is shown by the columns labeled ``Av. Simp.'' in Table~\ref{tab:eval_cp}, which represent the average proportion of clauses that are simplified as a consequence of the pruning step. Note that these particular columns only consider the instances for which the algorithm did not terminate before the end of the preprocessing, and thus had to resort to a propositional encoding. We remark that the highest percentages of simplified clauses are found in domains that contain little to no symmetries. For example, in Rovers, fluents represent entities that often have different types, and that are affected in different ways by operators. For instance, operators of the form \lstinline{navigate(rover, x, y)} have a unique profile, and are not numerous. As a consequence, their respective domains remain small, which is something our algorithm makes the most of, especially in the case of subinstance isomorphisms.

On the contrary, for domains that contain lots of symmetries, the pruning step does not remove a significant number of associations, be it for \stripssubisogeneral{}, \stripssubiso{} or \stripsembedding{}. This is the case in Hanoi, where all operators have the same profile: except for the information provided by the initial and goal states, all disks are interchangeable, which does not allow our preprocessing to draw any conclusive result. The only information that can guide the search is encoded in the initial state, which we believe partly explains the slightly greater coverage of $\stripssubiso{}$ over $\stripssubisogeneral{}$.


In some instances of our set of benchmarks, pruning alone suffices to find negative matching instances \changes{(i.e. instances that have no solution)}, be it for subinstance isomorphisms or embeddings. This happens when the majority of associations between fluents or between operators are ruled out, and the domain of some variable becomes empty. In these cases, our algorithm can return UNSAT prematurely, skipping the search phase altogether. As a direct consequence, our algorithm is most effective in detecting negative instances of the problems we consider.
This is why the pruning step allows us to significantly increase our coverage on STRIPS matching instances that are negative, as shown in Figure~\ref{fig:solved_vs_cutoff}, while our performance on positive instances is more modest, although significant.

%
In Table~\ref{tab:time_decomposition}, we can also see that, in the majority of the domains we consider, the additional time required by the constraint propagation phase is negligible compared to the rest of the algorithm. In fact, be it in domains where it prunes many associations or in domains where its efficiency is limited, constraint propagation rarely takes more than a handful of seconds. As a consequence, some instances that would otherwise require a substantial amount of time are now solved almost immediately.
This is made clear in Figure~\ref{fig:solved_vs_cutoff:ssi_negative}, for instance:  solving the easiest 400 negative \stripssubiso{} instances requires 5 minutes when pruning is not enabled, while it takes a handful of seconds when pruning is enabled.

The most notable exception to this, however, is domain Satellite, for which the preprocessing step seems to be the most costly in time. However, it is also crucial, as the  domain is generally hard for the preprocessing-less algorithm, which underperforms on this domain compared to other sets of benchmarks. In fact, for \stripsembedding{}, in about 65\% of this set of benchmarks, the preprocessing is not only necessary, but also sufficient to find that the instance is unsolvable: in most cases, it allows the algorithm to cut short, and detect a negative instance right away. When it comes to positive instances, only 2 additional positive instances are found with constraint propagation activated. Thus, most instances do not even require the SAT solving phase, hence the lower average time spent in compilation or solving. A similar phenomenon occurs for Barman, in the case of \stripsembedding{}.

\begin{table}[ht]
    \centering
    \caption{Average time, in seconds, spent in each of the main three steps of the algorithm: pruning (CP), compilation to SAT, and solving, respectively. The last column summarizes the average total running time of the algorithm for each domain of each problem. We only report instances that were successfully solved (either positively or negatively): results for \stripssubisogeneral{}, \stripssubiso{} and \stripsembedding{} are thus non-comparable, and Barman does not have statistics for \stripssubisogeneral{} nor \stripssubiso{}, as no such instance was solved.}
    {
    \begin{tabular}{l||rrr|r||rrr|r}
    	\toprule
    	\multirow{2}{*}{Domain} & \multicolumn{4}{c||}{\stripssubisogeneral{}} & \multicolumn{4}{c}{\stripssubiso{}} \\
    	 & CP & Comp. & Solving & Total & CP & Comp. & Solving & Total \\ 
    	\midrule 
    	\textbf{barman} & - & - & - & - & - & - & - & - \\
    	\textbf{blocks} & 0.3 & 58.3 & 32.4 & 90.7 & 0.3 & 52.4 & 40.8 & 93.3 \\
    	\textbf{ferry} & 0.2 & 76.3 & 114.3 & 190.7 & 0.2 & 78.3 & 106.8 & 185.2 \\
    	\textbf{gripper} & 0.2 & 26.2 & 27.6 & 53.8 & 0.1 & 16.4 & 28.4 & 44.9 \\
    	\textbf{hanoi} & 0.2 & 28.9 & 38.9 & 67.8 & 0.2 & 40.5 & 29.7 & 70.3 \\
    	\textbf{satellite} & 0.4 & 54.9 & 14.8 & 69.8 & 47.8 & 29.2 & 5.0 & 81.8 \\
    	\textbf{sokoban} & 2.5 & 0.0 & 0.0 & 2.5 & 2.5 & 0.0 & 0.0 & 2.5 \\
    	\textbf{tsp} & 0.1 & 26.8 & 21.4 & 48.2 & 0.1 & 14.4 & 19.3 & 33.8 \\
    	\bottomrule
    \end{tabular}

    \vspace{2ex}

    \begin{tabular}{l||rrr|r}
    	\toprule
    	\multirow{2}{*}{Domain} & \multicolumn{4}{c}{\stripsembedding{}} \\
    	 & CP & Comp. & Solving & Total \\ 
    	\midrule 
    	\textbf{barman} & 6.6 & 0.0 & 0.0 & 6.6 \\ 
    	\textbf{blocks} & 2.0 & 73.3 & 22.6 & 96.1 \\ 
    	\textbf{ferry} & 1.3 & 41.4 & 8.1 & 49.6 \\ 
    	\textbf{gripper} & 0.6 & 15.5 & 1.4 & 16.9 \\ 
    	\textbf{hanoi} & 0.4 & 51.1 & 15.5 & 66.8 \\ 
    	\textbf{satellite} & 34.1 & 25.3 & 2.8 & 62.0 \\ 
    	\textbf{sokoban} & 1.0 & 0.0 & 0.0 & 1.0 \\ 
    	\textbf{tsp} & 0.1 & 26.3 & 8.1 & 34.4 \\ 
    	\bottomrule
    \end{tabular}
    }
    \label{tab:time_decomposition}
\end{table}

\section{Related work}

Hor\v{c}{\'{\i}}k and Fi\v{s}er\cite{horcik2021endomorphisms} proposed a notion of endomorphism for classical planning tasks in finite domain representation (FDR). Mathematically, an endomorphism is a homomorphism where the domain and the codomain are identical. As STRIPS is a special case of FDR where variables have binary domains, their definition of an \emph{FDR endomorphism} is directly comparable to (a special case of) our notions of homomorphisms. In essence, compared to a subinstance isomorphism, 
an FDR endomorphism only requires that $\pre{\isomp{o}} \subseteq \isom{\pre{o}}$, while we require equality (even though in practice, the CSP encoding Hor\v{c}{\'{\i}}k and Fi\v{s}er propose also requires it). In addition, our mapping between fluents is injective, whereas FDR endomorphisms do not impose an 
\emph{injective} mapping between facts (i.e. literals, corresponding to variable-value assignments) but they do impose that each fact
maps to a fact representing an assignment to the same variable. In the case of boolean variables, such as \changes{the ones} we study in this paper,
this means that an endomorphism may map a fluent $p$ to its negation $\neg p$. However, if we assume all goals and preconditions are positive
(as we have done in this paper) and given that goals and preconditions must map respectively to goals and preconditions, 
this prevents fluents occurring in goals or preconditions mapping to negative fluents. Of course, generalising our notion of 
subinstance isomorphism to allow positive fluents to map to negative fluents would be an interesting extension if there may be
negative goals or preconditions.
The main difference between our notions resides in their intent: while we wish to keep the planning model intact, Hor\v{c}{\'{\i}}k and Fi\v{s}er try to reduce its size as much as possible, by folding it over itself and deleting redundant operators. 

Shleyfman \emph{et al.}~\cite{shleyfman2015heuristics} defined the notion of \emph{structural symmetry} for STRIPS planning instances. A structural symmetry is, in essence, a STRIPS Isomorphism as we defined, from an instance $\Pi$ to itself, except that Equation~(\ref{eq:def_strips_iso_i}) guaranteeing the stability of the initial state is not enforced. They show that a wide range of heuristics are invariant by structural symmetry, albeit their study excludes abstraction-based heuristics. Sievers \emph{et al.}~\cite{sievers2017structural} extended the notion of structural symmetry to PDDL instances, and defined it as a permutation of constants, variables and predicates that conserves the semantics of the action schemas, as well as the initial and goal states. They prove that a structural symmetry of a PDDL instance induces a structural symmetry on the STRIPS instance obtained through grounding. Later, R\"oger \emph{et al.}~\cite{roeger2018symmetry} showed that structural symmetries on PDDL instances can be leveraged to optimize its grounding, by pruning irrelevant operators through a relaxed reachability analysis.


A wide variety of notions of homomorphisms are defined not on the planning model itself, as we do in this paper, but on structures derived from it. Such structures notably include Labeled Transitions Systems (LTS), which are, in our context, graphs where edges are labeled with the names of operators, a node is designated as the initial state, and a set of nodes are designated as goal states. LTS's are naturally used to represent the state-space underlying any STRIPS instance. Symmetries of this state-space have also been studied, although through more compact representations, such as the Problem Description Graph (PDG)~\cite{pochter2011exploiting,shleyfman2015heuristics}.

Abstractions aim at creating equivalence classes between the states of the LTS $\mathcal{T}$, in order to build an abstract LTS $\mathcal{T}^\alpha$ on which some desirable properties and features of the original LTS $\mathcal{T}$ are carried over. Bäckström and Jonsson~\cite{backstrom2012abstracting} proposed a framework for analysing LTS-based abstractions, and understanding how some mathematical properties of an abstraction translate in terms of the structure of the set of paths it carries over to the abstract LTS. Hor\v{c}\'{i}k and Fi\v{s}er studied the connection between the notions of FDR endomorphism previously mentioned, and the notion of LTS endomorphism, which they introduced in the same article~\cite{horcik2021endomorphisms}. The aim is similar to their previous method: computing a homomorphism from (a factored version of) the state-space to itself allows them to remove redundant operators. 

In a preprocessing step, Pattern Databases (PDBs)~\cite{edelkamp2001planning} also reduce the size of the LTS underlying 
an FDR model by mapping it to 
a more concise one, obtained through \emph{syntactic projection}.
Syntactic projections take as input a set of fluents $S \subseteq F$, and forget all other fluents of the planning instance $P$, resulting in an  instance $P_{\mid S}$.
In our terminology, there is an embedding from the syntactic projection $P_{\mid S}$ to the original problem $P$: the mappings $\isomsymb,\isompsymb$ that define this embedding are both identity. Embedding can be said to be a more general notion, in that when $P'$ embeds in $P$, the problems $P'$ and $P_{\mid \isom{F'}}$ (where $F'$ is the set of fluents of $P'$) are not necessarily identical, notably because only inclusion rather than equality is imposed between preconditions, goals and the inital state (Equations~(\ref{eq:pre_morphism}), (\ref{eq:goal_morphism}), and (\ref{eq:init_morphism})).

More generally, computing the perfect heuristic on an abstract LTS $\mathcal{T}^\alpha$ is a common technique. The projection of the problem $\Pi$ onto a subset of its fluents induces an LTS that is an abstract search space for $\Pi$, on which the perfect heuristic can be computed, and extrapolated to the states of $\Pi$~\cite{edelkamp2001planning}. In addition to PDBs, methods based on abstractions of LTS include the Merge \& Shrink framework~\cite{sievers2021merge}, or Cartesian abstractions~\cite{seipp2018counterexample}.
However, in general, these works focus on computing the most relevant abstraction of a given structure, and the abstract LTS that results. Ultimately, the goal is to perform some computations on an abstract LTS, as they would have been impossible to achieve on the original LTS. In contrast, our work focuses on finding homomorphisms between pre-existing STRIPS models, without even considering the underlying state-space explicitly. In particular, the mappings that we create and work with are not necessarily surjective.

The embedding of an unsolvable planning instance $\Pi$ into another instance $\Pi'$ is, in itself, a proof that $\Pi'$ is unsolvable. Even more so, when $\Pi$ is small enough to be humanly understandable, it provides an explanation of why $\Pi'$ is, in fact, unsolvable. Even though research in planning was historically focused on finding plans, there has been, in the recent years, a surge in interest in unsolvability detection. A wealth of techniques have been tailored to better handle unsolvable instances. This includes heuristics such as Merge \& Shrink~\cite{hoffmann2014distance}, PDBs~\cite{torralba2016sympa}, or even logic-based heuristics~\cite{staahlberg2021learning}. Other such techniques include, for instance, dead-end detection~\cite{cserna2018avoiding} and traps~\cite{lipovetzky2016traps}. In addition to adapting already existing methods, some other work propose new, more specific approaches, sometimes inspired from other domains where unsatisfiability detection is more crucial, such as constraint programming~\cite{backstrom2013fast} or propositional logic satisfiability~\cite{steinmetz2017state}, but also sometimes more ad-hoc~\cite{christen2022detecting}.

In the case where an instance is solvable, a plan acts as a certificate of solvability. But in the case of unsolvable instance, there exists no immediate counterpart. Akin to proofs of unsolvability found in the SAT community, a few methods have been recently proposed to certify plan non-existence: Eriksson \emph{et al}~\cite{eriksson2017unsolvability} have shown that a form of inductive certificates are suitable for planning instances, and later extended their work with a more flexible proof system~\cite{eriksson2018proof}. Both methods can be adapted on existing planning systems. Other works have focused on explaining why a planner failed to find a plan, by providing humanly intelligible reasons and excuses~\cite{sreedharan2019can,gobelbecker2010coming}. \changes{In this article, we showed that an embedding from $P'$ to $P$, given that $P'$ is a smaller planning instance that was proven to be unsolvable, can serve as a certificate of unsolvability for $P$. More than that, if $P'$ is small enough that the reasons why it is unsolvable are clear, an embedding can help forge an intuition on why $P$ itself is unsolvable.}








\section{Conclusion}

In this article, we introduced the problem \stripsiso{} of finding an isomorphism between two planning problems, and showed that it is $\GI{}$-complete. Afterwards, we introduced the notions of subinstance isomorphism and embedding, as well as the associated problems \stripssubiso{} and \stripssubisogeneral{}, on the one hand, and \stripsembedding{} on the other. In addition to proving the \NP{}-completeness of these problems, we proposed a generic algorithm for them, based on constraint propagation techniques and a reduction to SAT.
We chose to use a reduction to SAT to take advantage of the efficiency of state-of-the-art SAT solvers, but, of course, a direct coding into CSP could take advantage of the automatic constraint propagation of CSP solvers.

The experimental evaluation of this algorithm shows that traditional constraint propagation in a preprocessing step can greatly improve the efficiency of SAT solvers.
Even though some planning domains benefited greatly from an almost costless and effective pruning step, some others were left almost unchanged after the (fortunately short) preprocessing. A common characteristic of these latter problems is that they all have a significant amount of symmetries, which might be harnessed to some extent. Recent theoretical results seem to indicate that, in practice, this could be done with a reasonable computational effort. Indeed, finding a set of generators for the group of symmetries of a planning instance is a \GI{}-complete problem~\cite{shleyfman2021}, as well as the first step in the direction of breaking symmetries in the search for a homomorphism between planning instances. 
More generally, as the study of symmetries in constraint programming is a well-established field~\cite{DBLP:journals/constraints/CohenJJPS06}, we believe this avenue of future research to be promising.

On a more general note, it remains an interesting open question to identify which characteristics of problems in NP make them amenable to this hybrid CP-SAT approach.

Subinstance isomorphism and embedding define two distinct partial orders between
STRIPS instances. This could help to perceive hidden structure
in the space of all planning instances. In this vein, an interesting avenue for future
research is the definition of formal explanations of (un)solvability via
minimal solvable isomorphic subinstances or minimal unsolvable embedded subinstances.

\bibliographystyle{plain}
\bibliography{wileyNJD-AMA}%

\changes{\appendix

\section{Proofs of the relations between embeddings and projections}
\label{appendix:embeddings_proofs}
\vspace*{12pt}}

\changes{
\noindent\textbf{Lemma~\ref{lem:abstraction_transition_preservation}.}
\emph{
    Suppose that $(\isomsymb{}, \isompsymb{})$ is an embedding of $P'$ in $P$, and let $\lts^{P'}$ be the state-space of $P'$. Let $\alpha$ be the projection of $P$ over the set of fluents $\isom{F'}$ over which the embedding ranges, and let $\lts^{\alpha}$ be the abstract state space. Then
    \begin{itemize}
        \item There exists a bijection $b$ between the states of $\lts^{\alpha}$ and the states of $\lts^{P'}$;
        \item For any two different states $s^\alpha_1 \not= s^\alpha_2$ of $\lts^\alpha$, if there exists a transition from $s^\alpha_1$ to $s^\alpha_2$, then there exists a transition from $b(s^\alpha_1)$ to $b(s^\alpha_2)$ in $\lts^{P'}$
    \end{itemize}
}

\begin{proof}
With $P' = \langle F', I', O', G'\rangle$, let us denote $\lts^{P'} = \langle 2^{F'}, O', T', I', {S^{G'}} \rangle$ the state-space of $P'$, and $\lts^\alpha = \langle S^\alpha, O, T^\alpha, \alpha(I), \alpha(S^G) \rangle$ the LTS associated to the projection $\alpha$ over $\isom{F'}$. We show that any transition of $T^\alpha$ between two \emph{different} states $s_1, s_2 \in S^\alpha \subseteq 2^F$ has an equivalent transition in $P'$.

For a start, notice that $\alpha \circ \isomsymb{}$ is a bijection from the set $2^{F'}$ of states of $\lts^{P'}$, to the set $S^\alpha = 2^{\isom{F'}}$ of states of $\lts^\alpha$. Indeed, since for any $s \in 2^F$, $\alpha(s) = s \cap \isom{F'}$, $\alpha$ is surjective when restricted to the domain $2^{\isom{F'}}$, and thus a bijection since $\alpha: 2^{F} \rightarrow 2^{\isom{F'}}$ (by definition of a projection, and using a cardinality argument). Likewise, $\isomsymb{}$ defines a bijection from $2^{F'}$ to $2^{\isom{F'}}$ since $\isomsymb{}$ is injective by hypothesis (and thus bijective from $F'$ to $\isom{F'}$). Let us denote $b = \left(\alpha \circ \isomsymb{}\right)^{-1}$ the bijection from $S^\alpha$ to $2^{F'}$.

Let us now show that every \emph{non-reflexive} transition $\langle s^\alpha_1, o, s^\alpha_2 \rangle \in T^\alpha$ of $\lts^\alpha$ has an equivalent transition $\langle b(s^\alpha_1), \isomp{o}, b(s^\alpha_2) \rangle \in \lts^{P'}$. Let $\langle s^\alpha_1, o, s^\alpha_2 \rangle \in T^\alpha$ be a transition of the abstract state space. By definition of the transitions $T^\alpha$ of an abstraction, there exists a transition $\langle s_1, o, s_2 \rangle \in T$ in the concrete state space, such that $\alpha(s_1) = s^\alpha_1$ and $\alpha(s_2) = s^\alpha_2$. 

First, suppose that $o$ does not satisfy condition (\ref{eq:active_condition}), i.e. that $(\effp{o} \cup \effm{o}) \cap \isom{F'} = \emptyset$. Then since $s_2 = \apply{s_1}{o}$, and $\alpha(s) =s \cap \isom{F'}$, we have that
\begin{align*}
  \alpha(s_2) &= \alpha(\apply{s_1}{o}) \\
              &= \left((s_1 \setminus \effm{o}) \cup \effp{o} \right) \cap \isom{F'} \\
              &= \left(\left(s_1 \cap \isom{F'}\right) \setminus \left(\effm{o} \cap \isom{F'}\right) \cup \left(\effp{o} \cap \isom{F'} \right) \right) \\
              &= s_1 \cap \isom{F'} &\text{(since (\ref{eq:active_condition}) is not satisfied)} \\
              &= \alpha(s_1)
\end{align*}  
As a consequence, the transition $\langle s^\alpha_1, o, s^\alpha_2 \rangle$ is reflexive.

Now suppose that condition (\ref{eq:active_condition}) is satisfied by $o$, i.e. that $(\effp{o} \cup \effm{o}) \cap \isom{F'} \not= \emptyset$. Then let us show that $\langle s'_1, o', s'_2 \rangle$ is a transition in $\lts^{P'}$, where $s'_1 = \isomsymb{}^{-1}(s_1)$, $s'_2 = \isomsymb{}^{-1}(s_2)$, and $o' = \isomp{o}$. For this, we have to show that $\pre{o'} \subseteq s'_1$ and that $s'_2 = \apply{s'_1}{o'}$.

We start by showing that $o'$ is applicable in $s'_1$. Since $\pre{o} \subseteq s_1$ and $\isom{\pre{\isomp{o}}} \subseteq \pre{o} \cap \isom{F'}$ (since $(\isomsymb{}, \isompsymb{})$ is an embedding), we have
\begin{align*}
    \isom{\pre{\isomp{o}}} \subseteq s_1 \cap \isom{F'}
\end{align*}
So, by the application of $\isomsymb{}^{-1}$ on both sides,
\begin{align*}
    \pre{\isomp{o}} &\subseteq \isomsymb{}^{-1}(s_1 \cap \isom{F'}) \\
                    &\subseteq s'_1
\end{align*}

We now show that $s'_2 = \apply{s'_1}{o'}$, and we follow a similar schema as above to prove this.
\begin{align*}
    s_2 &= \apply{s_1}{o} \\
        &= (s_1 \setminus \effm{o}) \cup \effp{o} \\
    \isomsymb{}^{-1}(s_2) 
        &= \isomsymb{}^{-1}\left((s_1 \setminus \effm{o}) \cup \effp{o}\right) \\
        &= \isomsymb{}^{-1}\left(\left((s_1 \setminus \effm{o}) \cup \effp{o}\right) \cap \isom{F'}\right) \\
        &= \isomsymb{}^{-1}\left(\left(\left(\left(s_1 \cap \isom{F'}\right) \setminus \left(\effm{o} \cap \isom{F'}\right)\right) \cup \left(\effp{o} \cap \isom{F'} \right)\right)\right) \\
        &= \isomsymb{}^{-1}\left(\left(\left(s_1 \cap \isom{F'}\right) \setminus \isom{\effm{\isomp{o}}}\right) \cup \isom{\effp{\isomp{o}}}\right) &\text{(using (\ref{eq:effp_morphism}) and (\ref{eq:effm_morphism}))} \\
        &= \isomsymb{}^{-1}\left(\left(\left(s_1 \cap \isom{F'}\right)\right) \setminus \effm{\isomp{o}}\right) \cup \effp{\isomp{o}} &\text{(since } \isomsymb{}^{-1} \circ \isomsymb{} = Id \text{)} \\
        &= \left(s'_1 \setminus \effm{\isomp{o}}\right) \cup \effp{\isomp{o}} \\
    s'_2
        &= \apply{s'_1}{o'}    
\end{align*}
As a consequence, we have $\langle s'_1, \isomp{o}, s'_2\rangle \in T'$, which is such that $\langle \alpha(\isom{s'_1}), o, \alpha(\isom{s'_2})\rangle = \langle s^\alpha_1, o, s^\alpha_2\rangle \in T^\alpha$.

In conclusion, we have showed that every non-reflexive transition in $\lts^\alpha$ has an equivalent transition in $\lts^{P'}$, with respect to the bijection $b = \left( \alpha \circ \isomsymb{}\right)^{-1}$ from the states of $\lts^{\alpha}$ to the states of $\lts^{P'}$.
\end{proof}
}

\changes{
\noindent\textbf{Proposition~\ref{prop:embedding_paths}}
        Suppose that $(\isomsymb{}, \isompsymb{})$ is an embedding of $P'$ in $P$, and let $\lts^{P'}$ be the state-space of $P'$. Let $\alpha$ be the projection of $P$ over the set of fluents $\isom{F'}$ over which the embedding ranges, and let $\lts^{\alpha}$ be the abstract state space. Then
        if there exists a path from the initial state to a goal state in $\lts^\alpha$, there also exists a path from the initial state to a goal state in $\lts^{P'}$.

    \begin{proof}
        Suppose that there exists a path from the initial state to a goal state in $\lts^\alpha$. Let us use the notations of Lemma~\ref{lem:abstraction_transition_preservation}. Lemma~\ref{lem:abstraction_transition_preservation} and its proof give a bijection $b$ between states of $\lts^\alpha$ and $\lts^{P'}$ such that if a path between two states $s^\alpha_1, s^\alpha_2$ of $\lts^\alpha$ exists, then there exists a path from $b(s^\alpha_1)$ to $b(s^\alpha_2)$ in $\lts^{P'}$. This is straightforward to show by induction.

        Note that the initial state $s^\alpha_I$ of $\lts^\alpha$ is not necessarily mapped by $b$ to the initial state $I'$ of $\lts^{P'}$ (this is due to condition (\ref{eq:init_morphism}), which does not enforce equality).
        
        %
        %

        Two points are to be shown. First, we show that if $s^\alpha_G \in \alpha(S^{G})$, then $b(s^\alpha_G) \in S^{G'}$ (i.e. that $b$ maps goals states of $\lts^\alpha$ to goal states of $\lts^{P'}$). Then, we show that there exists a state $s^\alpha \in S^{\alpha}$ such that $b(s^\alpha) = I'$ and there exists a path in $\lts^\alpha$ from $s^\alpha$ to some goal state $s^\alpha_G$.

        Let $s^\alpha_G \in \alpha(S^G)$. By definition of $\alpha$ and $S^G$, we have that $G \cap \isom{F'} \subseteq s^\alpha_G$. In addition, since $\isomsymb{}$ is an embedding, we have (by (\ref{eq:goal_morphism})) that $\isom{G'} \subseteq G \cap \isom{F'}$. Combining the two equations, we obtain $\isom{G'} \subseteq s^\alpha_G$. Applying $\alpha$ on both sides (which is possible since $\alpha: 2^F \rightarrow 2^{\isom{F'}}$ and $2^{\isom{F'}} \subseteq 2^F$), we have $\alpha \circ \isom{G'} \subseteq s^\alpha_G$ (since the projection $\alpha$ is idempotent). By applying $b = \left( \alpha \circ \isomsymb{} \right)^{-1}$ on both sides, we have that $G' \subseteq b(s^\alpha_G)$, which shows that $s^\alpha_G$ is mapped to a goal state in $\lts^{P'}$.

        Let us now show that there exists a state $s^\alpha \in S^{\alpha}$ such that $b(s^\alpha) = I'$ and there exists a path in $\lts^\alpha$ from $s^\alpha$ to some goal state $s^\alpha_G$. First remark that, if there exists a path from $\alpha(I)$ to some $s^\alpha_G \in \alpha(S^G)$ in $\lts^\alpha$, then there exists a path from any $s^\alpha \in S^\alpha$ such that $\alpha(I) \subseteq s^\alpha$ to some $s'^\alpha_G \in \alpha(S^G)$. Let us consider $s^\alpha_{I'} = \alpha(\isom{I'}) \in S^\alpha$. We have:
        \begin{align*}
            s^\alpha_{I'} &= \alpha(\isom{I'}) \\
                          &= \isom{I'} &\text{(Definition of }\alpha\text{)} \\
                          &\supseteq I \cap \isom{F'} &\text{(By (\ref{eq:init_morphism}))} \\
                          &= \alpha(I)
        \end{align*}
        So there exists a path from $s^\alpha_{I'}$ to a goal state $s'^\alpha_G$ in $\lts^\alpha$. Using the results that stems from the generalization of Lemma~\ref{lem:abstraction_transition_preservation}, as presented in the beginning of this proof, we have that there exists a path from $I' = b(s^\alpha_{I'})$ to some goal state $b(s'^\alpha_G)$ in $\lts^{P'}$, which concludes the proof. 
        



    \end{proof}
}

\end{document}

%% file: macros.tex
\NewEnviron{problembody}[1][]{%
  \par\noindent%
  \renewcommand{\arraystretch}{1.2}%
  \begin{tabularx}{\textwidth}{lX}%
    \BODY\\
  \end{tabularx}%
  \medskip\par%
}
\makeatother

\newtheorem{corollary}{Corollary}
\newtheorem{definition}{Definition}
\newtheorem{proposition}{Proposition}
\newtheorem{lemma}{Lemma}
\newtheorem{construction}{Construction}
\newtheorem{problem}{Problem}


\definecolor{sweetblue}{HTML}{75CFEB}
\definecolor{sweetdarkblue}{HTML}{5F9AD4}

\makeatletter
\newcommand*{\skipnumber}[2][1]{%
  {\renewcommand*{\alglinenumber}[1]{}\State #2}%
  \addtocounter{ALG@line}{-#1}}
\makeatother

%
%

\newcommand{\GI}{\text{GI}}
\newcommand{\NP}{\text{NP}}
\newcommand{\Pclass}{\text{P}}

\newcommand{\dom}[1]{\textit{Var}(#1)}

\newcommand{\graph}{\mathcal{G}}

\newcommand{\pctext}[2]{\text{\parbox{#1}{\centering #2}}}

\newcommand{\parts}[1]{\mathcal{P}(#1)}

\newcommand{\stripsiso}{\textsf{SI}}
\newcommand{\stripssubiso}{\textsf{SSI}}
\newcommand{\stripssubisogeneral}{\textsf{SSI-H}}
\newcommand{\isomsymb}{\upsilon}
\newcommand{\isompsymb}{\nu}
\newcommand{\isom}[1]{\upsilon\left(#1\right)}
\newcommand{\isomp}[1]{\nu\left(#1\right)}

\newcommand{\stripsembedding}{\textsf{SE}}

\newcommand{\pre}[1]{\textsf{pre}(#1)}
\newcommand{\prep}[1]{\textsf{pre}^+(#1)}
\newcommand{\prem}[1]{\textsf{pre}^-(#1)}
\newcommand{\eff}[1]{\textsf{eff}(#1)}
\newcommand{\effp}[1]{\textsf{eff}^+(#1)}
\newcommand{\effm}[1]{\textsf{eff}^-(#1)}
\newcommand{\presymb}{\textsf{pre}}
\newcommand{\prepsymb}{\textsf{pre}^+}
\newcommand{\premsymb}{\textsf{pre}^-}
\newcommand{\effsymb}{\textsf{eff}}
\newcommand{\effpsymb}{\textsf{eff}^+}
\newcommand{\effmsymb}{\textsf{eff}^-}
\newcommand{\plan}{\pi}
\newcommand{\compos}{\odot}
\newcommand{\bigcompos}[2]{\bigodot_{#1}^{#2}}
\newcommand{\action}[2]{\langle #1, #2\rangle}
\newcommand{\absrel}{\mathcal{S}}  

\newcommand{\lts}{\ensuremath{\Theta}}
\newcommand{\apply}[2]{\ensuremath{#1[#2]}}

\newcommand{\fdist}[2]{\ensuremath{\delta(#1, #2)}}

\newcommand{\moveaction}[2]{\ensuremath{\textsf{move}(#1, #2)}}

\newcommand{\circuitgraph}[1]{\mathcal{C}_{#1}}

\newcommand{\finitemodel}{\ensuremath{M}}
\newcommand{\relation}{\mathcal{R}}

\newcommand{\domainsymb}{\mathcal{D}}
\newcommand{\domain}[1]{\mathcal{D}(#1)}
\newcommand{\opprofile}[1]{\textsf{profile}(#1)}

\newcommand{\fassoc}[2]{\ensuremath{f_{#1}^{#2}}}
\newcommand{\oassoc}[2]{\ensuremath{o_{#1}^{#2}}}
\newcommand{\fuseful}[1]{\ensuremath{u_{#1}}}
\newcommand{\oactive}[1]{\ensuremath{a_{#1}}}
\newcommand{\markedassoc}[2]{\ensuremath{m_{#1}^{#2}}}

\newcommand{\arnaud}[1]{{\color{sweetdarkblue}\textbf{Arnaud:} #1}}
\newcommand{\fred}[1]{{\color{purple}\textbf{Fred:} #1}}

%% file: main.bbl
\begin{thebibliography}{10}

\bibitem{DBLP:journals/ai/AmilhastreFM02}
J{\'{e}}r{\^{o}}me Amilhastre, H{\'{e}}l{\`{e}}ne Fargier, and Pierre Marquis.
\newblock Consistency restoration and explanations in dynamic csps --
  application to configuration.
\newblock {\em Artificial Intelligence}, 135(1-2):199--234, 2002.

\bibitem{babai2018group}
L{\'a}szl{\'o} Babai.
\newblock Group, graphs, algorithms: the graph isomorphism problem.
\newblock In Boyan Sirakov, Paulo {Ney de Souza}, and Marcelo Viana, editors,
  {\em Proceedings of the International Congress of Mathematicians}, pages
  3319--3336, 2018.

\bibitem{backstrom2012abstracting}
Christer B{\"{a}}ckstr{\"{o}}m and Peter Jonsson.
\newblock Abstracting abstraction in search with applications to planning.
\newblock In Gerhard Brewka, Thomas Eiter, and Sheila~A. McIlraith, editors,
  {\em {KR}}, 2012.

\bibitem{backstrom2013fast}
Christer B{\"{a}}ckstr{\"{o}}m, Peter Jonsson, and Simon St{\aa}hlberg.
\newblock Fast detection of unsolvable planning instances using local
  consistency.
\newblock In Malte Helmert and Gabriele R{\"{o}}ger, editors, {\em SOCS}, 2013.

\bibitem{botea2005macro}
Adi Botea, Markus Enzenberger, Martin M{\"{u}}ller, and Jonathan Schaeffer.
\newblock Macro-ff: Improving {AI} planning with automatically learned
  macro-operators.
\newblock {\em Journal of Artificial Intelligence Research}, 24:581--621, 2005.

\bibitem{buchner2022comparison}
Clemens B{\"u}chner, Patrick Ferber, Jendrik Seipp, and Malte Helmert.
\newblock A comparison of abstraction heuristics for rubik's cube.
\newblock In {\em ICAPS Workshop on Heuristics and Search for
  Domain-independent Planning}, 2022.

\bibitem{bylander1994computational}
Tom Bylander.
\newblock The computational complexity of propositional {STRIPS} planning.
\newblock {\em Artificial Intelligence}, 69(1-2):165--204, 1994.

\bibitem{christen2022detecting}
Remo Christen, Salom{\'{e}} Eriksson, Florian Pommerening, and Malte Helmert.
\newblock Detecting unsolvability based on separating functions.
\newblock In Akshat Kumar, Sylvie Thi{\'{e}}baux, Pradeep Varakantham, and
  William Yeoh, editors, {\em ICAPS}, pages 44--52, 2022.

\bibitem{DBLP:journals/constraints/CohenJJPS06}
David~A. Cohen, Peter Jeavons, Christopher Jefferson, Karen~E. Petrie, and
  Barbara~M. Smith.
\newblock Symmetry definitions for constraint satisfaction problems.
\newblock {\em Constraints}, 11(2-3):115--137, 2006.

\bibitem{DBLP:conf/stoc/Cook71}
Stephen~A. Cook.
\newblock The complexity of theorem-proving procedures.
\newblock In Michael~A. Harrison, Ranan~B. Banerji, and Jeffrey~D. Ullman,
  editors, {\em 3rd Annual {ACM} Symposium on Theory of Computing}, pages
  151--158, 1971.

\bibitem{CP2022}
Martin~C. Cooper, Arnaud Lequen, and Fr{\'{e}}d{\'{e}}ric Maris.
\newblock Isomorphisms between {STRIPS} problems and sub-problems.
\newblock In Christine Solnon, editor, {\em {CP} 2022}, volume 235 of {\em
  LIPIcs}, pages 13:1--13:16. Schloss Dagstuhl - Leibniz-Zentrum f{\"{u}}r
  Informatik, 2022.

\bibitem{cserna2018avoiding}
Bence Cserna, William~J. Doyle, Jordan~S. Ramsdell, and Wheeler Ruml.
\newblock Avoiding dead ends in real-time heuristic search.
\newblock In Sheila~A. McIlraith and Kilian~Q. Weinberger, editors, {\em AAAI},
  pages 1306--1313, 2018.

\bibitem{DBLP:journals/jair/DarwicheM02}
Adnan Darwiche and Pierre Marquis.
\newblock A knowledge compilation map.
\newblock {\em Journal of Artificial Intelligence Research}, 17:229--264, 2002.

\bibitem{edelkamp2001planning}
Stefan Edelkamp.
\newblock Planning with pattern databases.
\newblock In {\em ECP}, 01 2001.

\bibitem{eriksson2017unsolvability}
Salom{\'{e}} Eriksson, Gabriele R{\"{o}}ger, and Malte Helmert.
\newblock Unsolvability certificates for classical planning.
\newblock In Laura Barbulescu, Jeremy Frank, Mausam, and Stephen~F. Smith,
  editors, {\em ICAPS}, pages 88--97, 2017.

\bibitem{eriksson2018proof}
Salom{\'{e}} Eriksson, Gabriele R{\"{o}}ger, and Malte Helmert.
\newblock A proof system for unsolvable planning tasks.
\newblock In Mathijs de~Weerdt, Sven Koenig, Gabriele R{\"{o}}ger, and Matthijs
  T.~J. Spaan, editors, {\em {ICAPS}}, pages 65--73, 2018.

\bibitem{fikes1972learning}
Richard Fikes, Peter~E. Hart, and Nils~J. Nilsson.
\newblock Learning and executing generalized robot plans.
\newblock {\em Artificial Intelligence}, 3(1-3):251--288, 1972.

\bibitem{FikesN71}
Richard Fikes and Nils~J Nilsson.
\newblock {STRIPS:} {A} new approach to the application of theorem proving to
  problem solving.
\newblock {\em Artificial Intelligence}, 2(3/4):189--208, 1971.

\bibitem{fox1999detection}
Maria Fox and Derek Long.
\newblock The detection and exploitation of symmetry in planning problems.
\newblock In Thomas Dean, editor, {\em IJCAI}, pages 956--961, 1999.

\bibitem{Garey1979computers}
M.~R. Garey and David~S. Johnson.
\newblock {\em Computers and Intractability: {A} Guide to the Theory of
  NP-Completeness}.
\newblock W. H. Freeman, 1979.

\bibitem{geffner2013concise}
Hector Geffner and Blai Bonet.
\newblock {\em A Concise Introduction to Models and Methods for Automated
  Planning}.
\newblock Morgan \& Claypool Publishers, 2013.

\bibitem{gobelbecker2010coming}
Moritz G{\"{o}}belbecker, Thomas Keller, Patrick Eyerich, Michael Brenner, and
  Bernhard Nebel.
\newblock Coming up with good excuses: What to do when no plan can be found.
\newblock In Ronen~I Brafman, Hector Geffner, J{\"{o}}rg Hoffmann, and Henry~A
  Kautz, editors, {\em ICAPS}, pages 81--88, 2010.

\bibitem{helmert2003complexity}
Malte Helmert.
\newblock Complexity results for standard benchmark domains in planning.
\newblock {\em Artificial Intelligence}, 143(2):219--262, 2003.

\bibitem{helmert2006fast}
Malte Helmert.
\newblock The fast downward planning system.
\newblock {\em Journal of Artificial Intelligence Research}, 26:191--246, 2006.

\bibitem{hoffmann2014distance}
J{\"o}rg Hoffmann, Peter Kissmann, and Alvaro Torralba.
\newblock ``distance''? who cares? tailoring merge-and-shrink heuristics to
  detect unsolvability.
\newblock In Torsten Schaub, Gerhard Friedrich, and Barry O'Sullivan, editors,
  {\em ECAI 2014}, volume 263 of {\em Frontiers in Artificial Intelligence and
  Applications}, pages 441--446. {IOS} Press, 2014.

\bibitem{horcik2021endomorphisms}
Rostislav Hor{\v{c}}{\'\i}k and Daniel Fi{\v{s}}er.
\newblock Endomorphisms of classical planning tasks.
\newblock In {\em {AAAI}}, pages 11835--11843, 2021.

\bibitem{DBLP:conf/ecai/KautzS92}
Henry~A Kautz and Bart Selman.
\newblock Planning as satisfiability.
\newblock In Bernd Neumann, editor, {\em {ECAI}}, pages 359--363, 1992.

\bibitem{korf1985macro}
Richard~E. Korf.
\newblock Macro-operators: {A} weak method for learning.
\newblock {\em Artificial Intelligence}, 26(1):35--77, 1985.

\bibitem{korf1983learning}
Richard~Earl Korf.
\newblock {\em Learning to solve problems by searching for macro-operators}.
\newblock PhD thesis, Carnegie Mellon University, USA, 1983.
\newblock AAI8425820.

\bibitem{lin2021change}
Songtuan Lin and Pascal Bercher.
\newblock Change the world - how hard can that be? {O}n the computational
  complexity of fixing planning models.
\newblock In Zhi-Hua Zhou, editor, {\em {IJCAI}}, pages 4152--4159, 8 2021.

\bibitem{lipovetzky2016traps}
Nir Lipovetzky, Christian~J. Muise, and Hector Geffner.
\newblock Traps, invariants, and dead-ends.
\newblock In Amanda~Jane Coles, Andrew Coles, Stefan Edelkamp, Daniele
  Magazzeni, and Scott Sanner, editors, {\em ICAPS}, pages 211--215, 2016.

\bibitem{luo2017effective}
Mao Luo, Chu{-}Min Li, Fan Xiao, Felip Many{\`{a}}, and Zhipeng L{\"{u}}.
\newblock An effective learnt clause minimization approach for {CDCL} {SAT}
  solvers.
\newblock In Carles Sierra, editor, {\em {IJCAI} 2017}, pages 703--711, 2017.

\bibitem{mackworth1977consistency}
Alan~K Mackworth.
\newblock Consistency in networks of relations.
\newblock {\em Artificial intelligence}, 8(1):99--118, 1977.

\bibitem{minton1985selectively}
Steven Minton.
\newblock Selectively generalizing plans for problem-solving.
\newblock In Aravind~K. Joshi, editor, {\em IJCAI}, pages 596--599, 1985.

\bibitem{muppasani2023solving}
Bharath Muppasani, Vishal Pallagani, Biplav Srivastava, and Forest Agostinelli.
\newblock On solving the rubik's cube with domain-independent planners using
  standard representations.
\newblock arXiv, 2023.

\bibitem{pochter2011exploiting}
Nir Pochter, Aviv Zohar, and Jeffrey~S. Rosenschein.
\newblock Exploiting problem symmetries in state-based planners.
\newblock In Wolfram Burgard and Dan Roth, editors, {\em {AAAI}}, 2011.

\bibitem{roeger2018symmetry}
Gabriele R{\"o}ger, Silvan Sievers, and Michael Katz.
\newblock Symmetry-based task reduction for relaxed reachability analysis.
\newblock In {\em ICAPS}, pages 208--217, 2018.

\bibitem{rokicki2014}
Tomas Rokicki and Morley Davidson.
\newblock God's number is 26 in the quarter-turn metric.
\newblock
  \url{https://web.archive.org/web/20220906131819/https://www.cube20.org/qtm/},
  2022.
\newblock Accessed: 2022-09-6.

\bibitem{rossi2006handbook}
Francesca Rossi, Peter Van~Beek, and Toby Walsh.
\newblock {\em Handbook of Constraint Programming}.
\newblock Elsevier, 2006.

\bibitem{sakallah2021symmetry}
Karem~A. Sakallah.
\newblock Symmetry and satisfiability.
\newblock In Armin Biere, Marijn Heule, Hans van Maaren, and Toby Walsh,
  editors, {\em Handbook of Satisfiability}, volume 336 of {\em Frontiers in
  Artificial Intelligence and Applications}, pages 509--570. {IOS} Press,
  second edition edition, 2021.

\bibitem{seipp2018counterexample}
Jendrik Seipp and Malte Helmert.
\newblock Counterexample-guided cartesian abstraction refinement for classical
  planning.
\newblock {\em Journal of Artificial Intelligence Research}, 62:535--577, 2018.

\bibitem{shleyfman2021}
Alexander Shleyfman and Peter Jonsson.
\newblock Computational complexity of computing symmetries in finite-domain
  planning.
\newblock {\em Journal of Artificial Intelligence Research}, 70:1183–1221,
  may 2021.

\bibitem{shleyfman2015heuristics}
Alexander Shleyfman, Michael Katz, Malte Helmert, Silvan Sievers, and Martin
  Wehrle.
\newblock Heuristics and symmetries in classical planning.
\newblock In Blai Bonet and Sven Koenig, editors, {\em {AAAI}}, pages
  3371--3377, 2015.

\bibitem{sievers2021merge}
Silvan Sievers and Malte Helmert.
\newblock Merge-and-shrink: A compositional theory of transformations of
  factored transition systems.
\newblock {\em Journal of Artificial Intelligence Research}, 71:781--883, 2021.

\bibitem{sievers2017structural}
Silvan Sievers, Gabriele R{\"o}ger, Martin Wehrle, and Michael Katz.
\newblock Structural symmetries of the lifted representation of classical
  planning tasks.
\newblock {\em ICAPS Workshop on Heuristics and Search for Domain-independent
  Planning}, 1:67--74, 2017.

\bibitem{sreedharan2019can}
Sarath Sreedharan, Siddharth Srivastava, David~E. Smith, and Subbarao
  Kambhampati.
\newblock Why can't you do that {HAL}? explaining unsolvability of planning
  tasks.
\newblock In Sarit Kraus, editor, {\em IJCAI}, pages 1422--1430, 2019.

\bibitem{staahlberg2021learning}
Simon St{\aa}hlberg, Guillem Franc{\`{e}}s, and Jendrik Seipp.
\newblock Learning generalized unsolvability heuristics for classical planning.
\newblock In Zhi{-}Hua Zhou, editor, {\em {IJCAI}}, pages 4175--4181, 2021.

\bibitem{steinmetz2017state}
Marcel Steinmetz and J{\"o}rg Hoffmann.
\newblock State space search nogood learning: Online refinement of
  critical-path dead-end detectors in planning.
\newblock {\em Artificial Intelligence}, 245:1--37, 2017.

\bibitem{torralba2016sympa}
Alvaro Torralba.
\newblock Sympa: Symbolic perimeter abstractions for proving unsolvability.
\newblock {\em UIPC 2016 planner abstracts}, pages 8--11, 2016.

\bibitem{DBLP:journals/jacm/Ullmann76}
Julian~R. Ullmann.
\newblock An algorithm for subgraph isomorphism.
\newblock {\em Journal of the {ACM}}, 23(1):31--42, 1976.

\bibitem{zemlyachenko1985graph}
Viktor~N. Zemlyachenko, Nickolay~M. Korneenko, and Regina~I. Tyshkevich.
\newblock Graph isomorphism problem.
\newblock {\em Journal of Soviet Mathematics}, 29(4):1426--1481, 1985.

\end{thebibliography}
